\documentclass[preprint]{colt2020} 
\usepackage{amssymb}
\usepackage{mathtools}
\usepackage{dsfont}

\newtheorem{claim}{Claim}

\usepackage{bbm}
\usepackage{amsfonts}
\newcommand{\poly}{\mathrm{poly}}
\newcommand{\polylog}{\mathrm{polylog}}
\newcommand{\rem}{\mathsf{rem}}

\newcommand{\supp}{\mathsf{supp}}
\newcommand{\R}{\mathsf{ReLU}}
\newcommand{\step}{\mathsf{Step}}
\newcommand{\id}{\mathcal{I}}
\newcommand{\cP}{\mathcal{P}}
\newcommand{\sR}{\mathsf{SReLU}}
\newcommand{\sDelta}{\mathsf{S}\Delta}

\newcommand{\dprime}{\prime\prime}
\newcommand{\sch}{\mathcal{S}(\mathbb{R})}

\newcommand{\RR}{\mathbb{R}}
\newcommand{\unif}{\mathsf{Unif}}
\newcommand{\spacedot}{\,\cdot\,}
\newcommand{\ind}{\mathds{1}}

\newcommand{\qed}{$\blacksquare$}
\newcommand{\qedhere}{\hfill\qed}
\newcommand{\myproof}[1]{\noindent \textbf{#1}}

\title{A Corrective View of Neural Networks: \\ Representation, Memorization and Learning}
\coltauthor{%
 \Name{Guy Bresler} \Email{guy@mit.edu}\\
 \addr Department of EECS, MIT\\
 Cambridge, MA, USA.
 \AND
 \Name{Dheeraj Nagaraj} \Email{dheeraj@mit.edu}\\
 \addr Department of EECS, MIT\\
 Cambridge, MA, USA.
}
\date{July 2019}

\begin{document}
\sloppy
\maketitle

\begin{abstract}
We develop a \emph{corrective mechanism} for neural network approximation: the total available non-linear units are divided into multiple groups and the first group approximates the function under consideration, the second approximates the error in approximation produced by the first group and corrects it, the third group approximates the error produced by the first and second groups together and so on. This technique yields several new representation and learning results for neural networks:
\begin{minipage}{13.8cm}
\begin{enumerate}\item \sloppy Two-layer neural networks in the random features regime (RF) can memorize arbitrary labels for $n$ arbitrary points in $\mathbb{R}^d$ with $\tilde{O}(\tfrac{n}{\theta^4})$ $\R$s, where $\theta$ is the minimum distance between two different points. This bound can be shown to be optimal in $n$ up to logarithmic factors.
\item \sloppy Two-layer neural networks with $\R$ and smoothed $\R$ units can represent functions with an error of at most $\epsilon$ with $O(C(a,d)\epsilon^{-1/(a+1)})$ units for $a \in \mathbb{N}\cup\{0\}$ when the function has $\Theta(ad)$ bounded derivatives. In certain cases $d$ can be replaced with effective dimension $q \ll d$. Our results indicate that neural networks with only a single nonlinear layer are surprisingly powerful with regards to representation, and show that in contrast to what is suggested in recent work, depth is not needed in order to represent highly smooth functions.

\item \sloppy Gradient Descent on the recombination weights of a two-layer random features network with $\R$s and smoothed $\R$s can learn low degree polynomials up to squared error $\epsilon$ with  $\mathrm{subpoly}(1/\epsilon)$ units. Even though deep networks can approximate these polynomials with $\polylog(1/\epsilon)$ units, existing \emph{learning} bounds for this problem require $\poly(1/\epsilon)$ units.  To the best of our knowledge, our results give the first sub-polynomial learning guarantees for this problem. 
\end{enumerate}
\end{minipage}
\end{abstract}

\vspace{-3mm}

\section{Introduction}

Neural networks have been shown to be very powerful in various classification and regression tasks \cite{goodfellow2016deep}. A lot of the properties of multi-layer networks remain unexplained rigorously, despite their success in practice. In this paper we focus on three core questions regarding the capabilities of neural networks: representation, memorization, and learning low degree polynomials.

 \paragraph{Representation.} Neural networks are universal approximators for continuous functions over compact sets and hence, when trained appropriately can solve a variety of machine learning problems 
 \cite{cybenko1989approximation,hornik1989multilayer,funahashi1989approximate,lu2017expressive,hanin2017approximating}. 
 A long line of work, starting with \cite{barron1993universal}, provides bounds on the number of activation functions required for two-layer neural networks to achieve a given error when the function being approximated satisfies certain smoothness conditions \cite{klusowski2018approximation,ma2019barron,liang2016deep,safran2017depth,yarotsky2017error,li2019better}. 
The papers \cite{barron1993universal} and \cite{klusowski2018approximation} use a law of large numbers based argument using random neural networks (see Section~\ref{sec:main_idea}) to achieve a squared error of $1/N$ using $N$ neurons, whereas other works including \cite{liang2016deep,safran2017depth,yarotsky2017error,li2019better} carry out a Taylor series approximation for the target function by implementing additions and multiplications using deep networks. 
 These assume more smoothness (higher number of bounded derivatives) of $f$ and give faster than $1/N$ rates for the squared error.  
  
 Deep neural networks are practically observed to be better approximators than shallow two-layer networks. Depth separation results construct functions that are easily and efficiently approximated by deep networks but cannot be approximated by shallower networks unless their width is very large (see \cite{safran2017depth,daniely2017depth,delalleau2011shallow,telgarsky2016benefits} and references therein). While the results in 
\cite{liang2016deep,safran2017depth,yarotsky2017error,li2019better} consider deep architectures to achieve faster representation results for a class of smooth functions, it remained unclear whether or not the class of functions they consider can be similarly represented by shallow networks. Recent work \cite{bresler2020sharp} gives sharp representation results for arbitrary depth networks which show that deeper networks are better at representing less smooth functions.

In this work, we show similar representation results to those achieved in \cite{yarotsky2017error} using deep networks, but for a two-layer neural network. Crucial to our approach is a careful choice of activation functions which are the same as $\R$ activation functions outside of a small neighborhood of zero and they are smoother near zero.  We note that the Sobolev space assumption for the target function in \cite{yarotsky2017error} is essentially the same as our assumption of fast enough decay in their Fourier transform (see Section~\ref{s:approximation}) due to the relationship between smoothness of a function and the decay of its Fourier transform. The experiments in~\cite{zheng2015improving} and~\cite{elfwing2018sigmoid} suggest that considering smoothed activation functions in some layers along with $\R$ in some others can in fact give measurably better results in various problems. Theoretical results in \cite{li2019better} show that smooth functions can be more efficiently represented using rectified power units ($\mathrm{RePU}$), which are smoother than $\R$. 

Despite the guarantees given by representation results, in practice finding the optimal parameters for a neural network for a given problem involves large-scale non-convex optimization, which is in general very hard. Therefore, stating representation results in conjunction with training guarantees is important, and as described next, we do so in the context of the memorization and learning low-degree polynomials.

\paragraph{Memorization.} 

Neural networks have the property that they can memorize (or interpolate) random labels quite easily \cite{zhang2016understanding,belkin2018understand}. In practice, neural networks are trained using SGD and a long line of papers aims to understand memorization in over-parametrized networks via the study of SGD/GD (see \cite{du2019gradient,allen2018convergence,jacot2018neural} and references therein). A recent line of work studies the problem of memorization of  arbitrary labels on $n$ arbitrary data points and provides polynomial guarantees (polynomial in $n$) for the number of non-linear units required (see \cite{zou2018stochastic,zou2019improved,oymak2019towards,song2019quadratic,ji2019polylogarithmic,panigrahi2019effect} and references therein).  
 These polynomials often have high degree ($O(n^{30})$ in \cite{allen2018convergence} and $O(n^6)$ as in \cite{du2019gradient}). \cite{oymak2019towards} and \cite{song2019quadratic} improve this to $O(n^2)$ under stronger assumptions on the data. Moreover, the bounds in \cite{du2019gradient}, \cite{oymak2019towards} and \cite{song2019quadratic} contain data and possibly dimension dependent condition number factors. \cite{panigrahi2019effect} obtains intelligible bounds for such condition number factors for various kinds of activation functions, but do not improve upon the $O(n^6)$ upper bound. \cite{ji2019polylogarithmic,chen2019much} show a polylogarithmic bound on the number of non-linear units required for memorization, but only under the condition of NTK separability. 
 
 We consider the problem of memorization of arbitrary labels via gradient descent for arbitrary $d$ dimensional data points under the assumption that any two of these points are separated by a Euclidean distance of at least $\theta$. Under the distance condition which we use here, the results of \cite{ji2019polylogarithmic} still require $O({n^{12}}/{\theta^4})$ non-linear units. Our results obtain a dependence of $\tilde{O}({n}/{\theta^4})$ for two-layer $\R$ networks. This is optimal in $n$ up to log factors. A similar bound is shown in \cite{kawaguchi2019gradient}, but with additional polynomial dependence on the dimension. Under additional distributional assumptions on the data, \cite{daniely2019neural} shows the optimal bound of $O({n}/{d})$ whenever  $n$ is polynomially large in $d$.  Subsequent to the present paper's appearance on arXiv, \cite{bubeck2020network} used a similar iterative corrective procedure as proposed in this paper to address the question of memorizing $n$ points with the smallest possible total \emph{weight} rather than number of units. Our memorization results also achieve the optimal dependence for weight in terms of number of points $n$, with a better dependence on the error $\epsilon$ and with fewer assumptions on the data, but a worse dependence on the dimension $d$.

\paragraph{Learning Low Degree Polynomials.}

An important toy problem studied in the neural networks literature is that of learning degree $q$ polynomials with $d$ variables via SGD/GD when $q \ll d$. This problem was first considered in  \cite{andoni2014learning}, and they showed that a two-layer neural network can be trained via Gradient Descent to achieve an error of at most $\epsilon$ whenever the number of non-linear units is $\Omega({d^{2q}}/{\epsilon^2})$ and \cite{yehudai2019power} gives a bound of $\Omega({d^{q^2}}/{\epsilon^4})$ using the random features model. All the currently known results for learning polynomials with SGD/GD require $\Omega\left(d^{2q} \poly(1/\epsilon)\right)$ non-linear units.

There are several representation theorems for low-degree polynomials with deep networks where the depth depends on the error $\epsilon$ (see ~\cite{liang2016deep,safran2017depth,yarotsky2017error}) by systematically implementing addition and multiplication. They require a total of $O(d^q\mathrm{polylog}(1/\epsilon))$ non-linear units. However, there are no training guarantees for these deep networks via any algorithm. We show that a two-layer neural network with $O(\mathrm{subpoly}(1/\epsilon))$ activation functions trained via GD/SGD suffices. In particular, the number of non-linear units we require is $O\big(C(a,q)d^{2q}\epsilon^{-\tfrac{1}{a+1}}\big)$ for arbitrary $a \in \mathbb{N}\cup \{0\}$, which is subpolynomial in $\epsilon$ when we take $a\to \infty$ slowly enough as $\epsilon \to 0$. 
To the best of our knowledge, these are the first subpolynomial bounds for learning low-degree polynomials via neural networks trained with SGD.

\subsection{The Corrective Mechanism}
\label{sec:main_idea}

We now describe the main theoretical tool developed in this work. Let $a,N \in \mathbb{N}$. With $aN$ non-linear units in total, under appropriate smoothness conditions on the function $f:\mathbb{R}^d\to \mathbb{R}$ being approximated, we describe a way to achieve a squared error of $O(1/N^a)$. The same basic methodology is used, with suitable modifications, to prove all of our results.

For any activation function $\sigma$, the construction given in \cite{barron1993universal} obtains $O(1/N)$ error guarantees for a two-layer network by picking $\Theta_1,\dots,\Theta_N$ i.i.d. from an appropriate distribution such that $\mathbb{E}\sigma(x;\Theta_1) \approx f(x)$ for every $x$ in some bounded domain. Then, the empirical sum $\hat{f}^{(1)}(x):=\frac{1}{N}\sum_{i=1}^{N}\sigma(x;\Theta_i)$ achieves an error of the form $C^2_f/{N}$ as shown by a simple variance computation, where $C_f$ is a norm on the Fourier transform of $f$. Since the Fourier transform is a linear operator, it turns out that the error (or remainder function) $f - \hat{f}^{(1)}(x)$ has a Fourier norm on the order of $C_f/\sqrt{N}$, which is much smaller than that of $f$. We let the next $N$ activation functions approximate this error function with  $\hat{f}^{(2)}$, so that $\hat{f}^{(1)}+\hat{f}^{(2)}$ achieves an error of at most $\frac{1}{N^2}$. We continue this argument inductively to obtain rates of $1/N^a$. We note that to carry out this argument, we need stronger conditions on $f$ than the ones used in \cite{barron1993universal} (see Section~\ref{s:approximation}). We next briefly describe some of the technical challenges and general proof strategy. \vspace{-1mm}

\paragraph{Overview of Proof Strategy.}
The main representation results are given in Theorems~\ref{thm:fast_rates_part_1} and \ref{thm:fast_rates_part_2} in Section~\ref{s:approximation}. 
We briefly describe our proof strategy:
\begin{enumerate}
\item The Fourier transform of the $\R$ function is not well-behaved, due to its non-differentiability at $0$. We construct an appropriate class of \emph{smoothed} $\R$ functions $\sR$, which is the same as $\R$ except in a small neighborhood around the origin, by convolving $\R$ with a specific probability density. This is done in Section~\ref{sec:smoothing_filter}.

\item Cosine functions are represented as a convolution of $\sR$ functions in Theorem~\ref{thm:cosine_representation}.

\item We prove a two-layer approximation theorem for $f$ under a Fourier norm condition using $\sR$ activation functions. This is done in Theorems~\ref{thm:smooth_relu_representation} and~\ref{thm:one_layer_approximation}.

\item In Theorem~\ref{thm:remainder_regularity}  we extend the error function $f^{\rem} := f - \hat{f}^{(1)}$ to all of $\mathbb{R}^d$ and show that its Fourier norm is smaller by a factor of $1/\sqrt{N}$ than that of $f$. Since activation functions used to construct $\hat{f}^{(1)}$ are \emph{one-dimensional} and their Fourier transforms are \emph{generalized functions}, we will use the ``mollification" trick from Fourier analysis to extend them to be $d$ dimensional functions with continuous Fourier transforms.

\item We use the next set of non-linear units to represent the error $f^{\rem}$ and continue recursively until the rate of $\frac{1}{N^a}$ is achieved. Since the remainder function becomes less smooth after each approximation step, we can only continue this procedure while the remainder is smooth enough to be effectively approximated by the class of activation functions considered. This depends on the smoothness of the original function $f$. (Roughly, an increased number of bounded derivatives of $f$ allows taking larger $a$.)
\end{enumerate}

The guarantees we obtain above contain dimension dependent factors which can be quite large. By considering functions with \emph{low-dimensional structure} -- that is, $d$ dimensional functions whose effective dimension is $q \ll d$ as described below, the dimension dependent factor can be improved to depend only on $q$ and not on $d$.
\vspace{-2mm}

\subsection{Functions with Low-Dimensional Structure}
 Let $d \in \mathbb{N}$ and $d\geq q$. 
We build a function $f:\mathbb{R}^{d} \to \mathbb{R}$ from real valued functions $f_i : \mathbb{R}^{q} \to \mathbb{R}$ for $i = 1,\dots,m$ as follows.
Let $B_i \subset \mathbb{R}^{d}$ be finite sets such that $|B_{i}| = q$ and for all $u,v \in {B}_i$, $\langle u,v\rangle = \delta_{u,v}$. We fix an ordering for the elements of each set $B_i$. For ease of notation, for every $x \in \mathbb{R}^d$, define $\langle x,B_i \rangle \in \mathbb{R}^q$ to be the vector whose elements are $(\langle x,v\rangle)_{v\in B_i}$. Define $f : \mathbb{R}^{d} \to \mathbb{R}$ as \vspace{-1mm}
\begin{equation}
\label{eq:low_dim_function}
f(x) = \frac{1}{m}\sum_{i=1}^{m}f_i(\langle x,B_i \rangle)\,. \vspace{-1mm}
\end{equation}
This is a rich class of functions that is dense over the set of $C_c(\mathbb{R}^d)$ equipped with the $L^2$ norm. This can be seen in various ways, including via universal approximation theorems for neural networks. Such low dimensional structure is often assumed to avoid overfitting in statistics and machine learning -- for instance, linear regression in which case $m=q=1$. \vspace{-3mm}

\paragraph{Low-Degree Polynomials.}
\label{subsubsec:low_degree_poly}
Low-degree polynomials are a special case of functions in the form of~\eqref{eq:low_dim_function}. For each $V: [d] \to \{0\}\cup [d]$ such that $\sum_{j\in [d]}V(j)\leq  q$ denote by $p_V :\mathbb{R}^d \to\mathbb{R}$ the corresponding monomial given by $p_V(x) = \prod_{j\in V}x_j^{V(j)}$. We note that each $p_V$ can depend on at most $q$ coordinates, and a standard dot and dash argument shows that the number of distinct $V$ are ${{q+d}\choose{q}}$. We consider the class of polynomials of $x \in \mathbb{R}^d$ with degree at most $q$, where $q \ll d$, which are of the form \begin{equation} \label{eq:low_degree_poly_def}
 f(x) = \sum_{V} J_V p_V(x)
 \end{equation}
for arbitrary $J_V \in \mathbb{R}$. Our results in Theorem~\ref{thm:polynomial_approximation} show how to approximate $f(x)$ for $x \in [0,1]^d$ under some given probability measure over this set. \vspace{-2mm}

\subsection{Preliminaries and Notation}\label{subsec:notations}

In this paper $d$ always denotes the dimension of some space like $\mathbb{R}^d$, which we take as the space of features of our data. We also consider $\mathbb{R}^q$ where $q \ll d $ and functions over them, especially when considering functions over $\mathbb{R}^d$ with a $q$ dimensional structure as defined just above.  $B_q^2(r)$ for $r >0$ denotes the Euclidean ball $\{x \in \mathbb{R}^q : \|x\|_2 \leq r \}$. In this paper, we consider approximating a function $f$ over some bounded set $B_d^2(r)$ or $B_q^2(r)$. Therefore, we are free to extend $f$ outside this. The standard $\ell^2$ Euclidean norm is denoted by $\|\cdot\|$.

We let capitals denote Fourier transforms. For example the Fourier transform of $g :\mathbb{R}^q \to \mathbb{R}$, $g \in L^1(\mathbb{R}^q)$ is denoted by
$G(\omega) = \int_{\mathbb{R}^q} g(x)e^{i\langle \omega,x\rangle}dx\,.$
Following the discussion in ~\cite{barron1993universal}, we scale $G$ to $\frac{G}{(2\pi)^q}$ to get the `Fourier distribution' of $g$. Whenever $G \in L^1(\mathbb{R}^q)$, the Fourier inversion formula implies that for all $x \in \mathbb{R}^q$,\vspace{-1mm}
\begin{equation}\label{eq:distribution_def}
g(x) = \int_{\mathbb{R}^q} \tfrac{G(\omega)}{(2\pi)^q}e^{-i\langle \omega,x\rangle}d\omega\,.\vspace{-2mm}
\end{equation}

Following~\cite{barron1993universal}, we also consider complex signed measures (instead of functions over $\mathbb{R}^q$) as ``Fourier distributions" corresponding to $g$  as long as Equation~\eqref{eq:distribution_def} holds for every $x$. In this case the formal integration against $\frac{G(\omega)}{(2\pi)^d}d\omega$ is understood  to be integration with respect to this signed measure. This broadens the class of functions $g$ that fall within the scope of our results. We denote the Schwartz space over $\mathbb{R}^q$ by $\mathcal{S}(\mathbb{R}^q)$. This space is closed under Fourier and inverse Fourier transforms. Finally, for real $x$ let $\R(x) = \max(0,x )$.


\subsection{ Random Features Model and Training}
The random features model was first studied in 
\cite{rahimi2008uniform,rahimi2008random,rahimi2009weighted} as an alternative to kernel methods. The representation results in \cite{barron1993universal,klusowski2018approximation,sun2018approximation};~\cite{bailey2019approximation,ji2019neural} and in this work use random features. In order to approximate a target function $f: \mathbb{R}^d \to \mathbb{R}$ we consider functions of the form
$\hat{f}(x;\mathbf{v}) =  \sum_{j=1}^{N}v_j \sigma (\langle \omega_j, x\rangle - T_j)\,. $
Here we have denoted $(v_j)  \in \mathbb{R}$ in the RHS collectively by $\mathbf{v}$ in the LHS, and $\omega_j \in \mathbb{R}^d$ and $T_j \in \mathbb{R}$ are random variables. We optimize over $\mathbf{v}$, keeping $\omega_j$'s and $T_j$'s fixed to find the best approximator for $f$.  More specifically, we want to solve the following loss minimization problem for some probability distribution $\zeta$ over $\mathbb{R}^d$:\vspace{-3mm}
\begin{equation}\label{eq:main_loss_function}
\mathbf{v}^{*} = \arg\inf_{\mathbf{v}\in \mathbb{R}^N}\int \big(f(x) - \hat{f}(x;\mathbf{v})\big)^2\zeta(dx)\,.\vspace{-2mm}
\end{equation}

 The problem above reduces to a least squares linear regression problem which can be easily and efficiently solved via gradient descent since this is an instance of a smooth convex optimization problem. By Theorem 3.3 in \cite{bubeck2015convex}, constant step-size gradient descent (GD) has an excess squared error $O(1/T)$ compared to the optimal parameter $\mathbf{v}^{*}$ after $T$ steps.  In this paper, whenever we prove a learning result, we first show that with high probability over the randomness in $\omega_j,T_j$, there exists a $\mathbf{v}_0$ such that the loss in approximating $f$ via $\hat{f}(\spacedot;\mathbf{v}_0)$ is at most $\epsilon/2$. Then, running GD for the objective in Equation~\eqref{eq:main_loss_function} for $T = \Omega(1/\epsilon)$ steps, we obtain $\mathbf{v}_T$ such that $ \int \big(f(x) - \hat{f}(x;\mathbf{v}_T)\big)^2\zeta(dx) \leq \epsilon \,.$

Since this paper mainly concerns the complexity in terms of the number of activation functions, we omit the details about time complexity of GD in our results, but it is understood throughout to be $O(1/\epsilon)$. The random features model is considered a good model for networks with a large number of activation functions since during training with SGD, the weights $\omega_j$ and $T_j$ do not change appreciably compared to the initial random value. Such a consideration has been used in the literature to obtain learning guarantees via SGD for large neural networks \cite{andoni2014learning,daniely2017sgd,du2019gradient}. \vspace{-2mm}

\subsection{Organization}
The paper is organized as follows. In Section~\ref{sec:memorization}, we illustrate the corrective mechanism by developing our results on memorization by two-layer $\R$ networks via SGD to conclude Theorem~\ref{thm:memorization_representation}. We then proceed to state our main results on function representation and learning polynomials in Section~\ref{s:approximation}. We give the construction of the smoothed $\R$ activation functions in Section~\ref{sec:smoothing_filter} and state an integral representation for cosine functions in terms of these activation functions. 
The proof of the main technical result of the paper, Theorem~\ref{thm:remainder_regularity}, is in Section~\ref{sec:unbiased_estimators}.  
Sections~\ref{sec:integral_representations} through~\ref{sec:main_thm_proofs} contain many of the proofs. 
\vspace{-1mm}

%
%

 \section{Memorization}\label{sec:memorization}


We first present our results on memorization, as they are the least technical yet suffice to illustrate the corrective mechanism. Suppose we are given $n$ labeled examples $(x_1,y_1), \dots, (x_n,y_n)$ where each data point $x_i\in \RR^d$ has label $y_i\in [0,1]$. In memorization (also known as interpolation), the goal is to construct a neural network which can be trained via SGD and which outputs $\hat f(x_i) = \hat{y}_i \approx y_i$ when the input is $x_i$, for every $i \in [n]$. 
The basic question is: how many neurons are needed?

\begin{theorem}\label{thm:memorization_representation}
Suppose $x_1,\dots,x_n\in \RR^d$ are such that $\|x_j\| \leq 1$ and $\min_{k\neq l}\|x_l-x_k\| \geq \theta$. For each $i=1,\dots,n$ let $y_i \in [0,1]$ be an arbitrary label for $x_i$. Let $(\omega_j,T_j)$ for $j=1,\dots, N$ be drawn i.i.d. from the distribution $\mathcal{N}(0,\sigma_0^2I_d)\times \unif[-2,2] $, where $\sigma_0 = {1}/{\sqrt{C_0 \times \log n \times \log \max(1/\theta,2)}}$ for some large enough constant $C_0$. Let $C$ be a sufficiently large universal constant and let $\epsilon,\delta \in  (0,1)$ be arbitrary. If $N \geq C n \frac{\log^4(\max(1/\theta,2))\log^4{n}}{\theta^4}\log{\tfrac{n}{\delta\epsilon}}$, then with probability at least $1-\delta$ there exist $a_1,\dots, a_N \in \mathbb{R}$ such that the function $\hat{f}^{\R}_N:= \sum_{j=1}^{N} a_j\R\left(\langle x, \omega_j\rangle - T_j\right)$ satisfies 
$$\sum_{k=1}^{n} \big(f(x_k)-\hat{f}^{\R}_N(x_k)\big)^2 \leq \epsilon 
\,.
$$
Moreover, if we consider only $a_1,\dots,a_N$ as the free parameters and keep the weights $(\omega_j,T_j)$ fixed, SGD/GD obtains the optimum because the objective is a convex function.
\end{theorem}
\vspace{-2mm}

\begin{remark}
In the initial version of this paper, there was an extra factor of $d^2$ in the guarantees given above. Based on reviewer comments, we have removed this dependence using a more refined analysis.
\end{remark}
%
In the remainder of this section we will prove Theorem~\ref{thm:memorization_representation}. 
 We will first show a Fourier-analytic representation. However, instead of using the regular Fourier transform, only in this section, we use the discrete Fourier Transform. For a function $f: \{x_1,\dots,x_n\} \to \mathbb{R}$, define $F:\mathbb{R}^d \to \mathbb{R}$ 
$$F(\xi) := \sum_{j=1}^{n}f(x_j)e^{i\langle \xi,x_j\rangle} \,. 
$$
 
The proof now proceeds in five steps. 
\paragraph{Step 1: Approximation via Fourier transform.} Let $\xi \sim \mathcal{N}(0,\sigma^2I_d)$ for $\sigma > 0$ to be specified momentarily and consider $\tilde{f} : \{x_1,\dots,x_n\} \to \mathbb{R}$ defined as \vspace{-2mm}
$$\tilde{f}(x_k):= \mathbb{E}F(\xi)e^{-i\langle\xi,x_k\rangle} = f(x_k) + \sum_{j\neq k}f(x_j)\mathbb{E}e^{i\langle\xi, x_j-x_k\rangle}
= f(x_k) +\sum_{j\neq k} f(x_j)e^{-\frac{\sigma^2d_{jk}^2}{2}}\,,\vspace{-2mm}
$$
where $d_{jk} = \|x_j - x_k\|_2$ and we have used the fact that the Gaussian $\xi\sim\mathcal{N}(0,\sigma^2I_d)$ has characteristic function $\mathbb{E}[e^{-i\langle t,\xi\rangle}]=\exp(-\frac12 \sigma^2\|t\|^2)$. Note that when $\sigma$ is large enough compared to ${1}/{\theta}$, we have $\tilde{f}(x_k) \approx f(x_k)$, so in what follows we will aim to approximate $\tilde f$. We will take $\sigma = \theta^{-1}\sqrt{2s\log{n}}$
 for some $s > 1$ to be fixed later. 
 
 We now record some properties of the random variable $F(\xi)$. Let $\|f\|_p$ denote the standard Euclidean $\ell^p$ norm when $f$ is viewed as a $n$-dimensional vector $(f(x_1),\dots,f(x_n))$. The proof of the following lemma is given in Section~\ref{sec:proofs_of_lemmas}.
 
\begin{lemma}\label{lem:fourier_expectation}
Let $\xi \sim \mathcal{N}(0,\sigma^2I_d)$ where $\sigma  = {\sqrt{2s\log{n}}}/{\theta}$. We have:
\begin{enumerate}
\item $|F(\xi)| \leq \|f\|_1$ almost surely,
\item $\mathbb{E}|F(\xi)|^2 \leq \|f\|^2_2 + {\|f\|_1^2}/{n^s} $, and 
\item $|f(x_k)-\tilde{f}(x_k)| \leq {\|f\|_1}/{n^s}$.
\end{enumerate}

\end{lemma}

\vspace{-2mm}
\paragraph{Step 2: Replacing sinusoids by ReLU.}
We first state a lemma which allows us to represent sinusoids in terms of $\R$ and $\step$ functions. The proof is given in Section~\ref{sec:integral_representations}. \vspace{-2mm}
\begin{lemma}\label{lem:relu_cosine_representation}
\sloppy Let $T \sim \unif[-2,2]$. There exist $C^{\infty}_c(\mathbb{R})$ functions $\eta(\spacedot;\alpha,\psi)$, (where $\alpha$ and $\psi$ are the parameters which define $\eta$) such that $\sup_{T\in \mathbb{R}}|\eta(T;\alpha,\psi)| \leq 1 $ and for every $t \in [-1,1]$ and for some absolute constant $C$, we have
\begin{align*}\cos(\alpha  t+ \psi) &= \mathbb{E}C(1+\alpha^2)\eta(T;\alpha,\psi)\R(t - T)
\end{align*}
\end{lemma}
Consider the event $\mathcal{A} = \{|\langle \xi, x_k\rangle| > \frac{2s\log n}{\theta} \text{ for some } k \in [n]\} $. By Gaussian concentration, we have $\mathbb{P}(\mathcal{A}) \leq {2}/{n^{s-1}}$. Write $F(\xi) = |F(\xi)|e^{-i\phi(\xi)}$ for some $\phi :\mathbb{R}^d \to \mathbb{R}$.   In Lemma~\ref{lem:relu_cosine_representation} we take $\alpha = \frac{2s\log n}{\theta}$, $t = \langle\xi,x_k\rangle/\alpha$, and $\psi = \phi(\xi)$ to conclude that if $T \sim \unif[-2,2]$ and independent of $\xi$, then on the event $\mathcal{A}^{c}$ 
\begin{equation*}
 \cos\big(\langle\xi,x_k\rangle + \phi(\xi)\big) = \mathbb{E}_{T}C(1+\tfrac{4s^2 \log^2 n}{\theta^2})\eta(T;\alpha,\psi)\R\Big(\theta \tfrac{\langle \xi,x_k\rangle}{2s\log n} -T\Big) 
\end{equation*}
Here $\mathbb{E}_T$ denotes the expectation only over the random variable $T$, $C$ is a universal constant and $\eta$ is as given by Lemma~\ref{lem:relu_cosine_representation}. We have used the fact that $\tfrac{\theta\langle \xi,x_k\rangle}{2s\log n} \in [-1,1]$ since the event $\mathcal{A}^{c}$ holds. Now, by definition of $\tilde{f}$, we have
$$\tilde{f}(x_k)=\mathbb{E}F(\xi)e^{-i\langle\xi,x_k\rangle} = \mathbb{E}|F(\xi)|e^{-i\phi(\xi) -i\langle\xi, x_k\rangle} = \mathbb{E}|F(\xi)|\cos\big(\langle\xi,x_k\rangle + \phi(\xi)\big)\,.$$
The last two equations lead to the following lemma, with details given in Section~\ref{sec:proofs_of_lemmas}.
\begin{lemma}\label{lem:almost_unbiased}
For some absolute constant $C_1$, we have
\begin{equation}
\biggr|\tilde{f}(x_k)- C \mathbb{E}|F(\xi)|\bigr(1+\tfrac{4s^2 \log^2 n}{\theta^2}\bigr)\eta(T;\alpha,\psi)\R\Big(\theta \tfrac{\langle \xi,x_k\rangle}{2s\log n} -T\Big)\biggr| \leq  C_1\tfrac{s^{3/2} \|f\|_1\log^{3/2} n}{\theta^{2}n^{s/2}}\,.
\label{eq:memorization_integral}
\end{equation}
\end{lemma}

\paragraph{Step 3: Empirical estimate.} Let $N_0 \in \mathbb{N}$. We draw $(\xi_l,T_l)$ for $l \in \{1,\dots,N_0\}$ i.i.d. from the distribution $\mathcal{N}(0,\sigma^2I_d)\times \unif[-2,2]$. 
We construct the following estimator for $\tilde{f}$, which is in turn an estimator for $f$: 
\begin{equation*}
\hat{f}_l(x) := C|F(\xi_l)|\bigr(1+\tfrac{4s^2 \log^2 n}{\theta^2}\bigr)\eta(T_l;\alpha,\phi(\xi_l))\R\Big(\theta \tfrac{\langle \xi_l,x_k\rangle}{2s\log n} -T_l\Big)\,. \label{eq:memorisation_estimator}
\end{equation*}
From Equation~\eqref{eq:memorization_integral}, we conclude that $\mathbb{E}\hat{f}_l(x_k) = \tilde{f}(x_k) + O\Big(\tfrac{s^{3/2} \|f\|_1\log^{3/2} n}{\theta^{2}n^{s/2}}\Big) $ and we construct the empirical estimate
\begin{equation}\label{e:empiricalMemorization}\hat{f}(x) := \frac{1}{N_0}\sum_{l=1}^{N_0}\hat{f}_l(x)\,.\end{equation}
\vspace{-2mm}
\begin{lemma}\label{lem:one_layer_contraction}
For some universal constant $C$ and $ L := C\frac{s^4 \log^4 n}{\theta^4}$, \vspace{-1mm}
$$\mathbb{E}\big(f(x_j)-\hat{f}(x_j)\big)^2 \leq \left[\frac{L}{N_0}+ \frac{Cs^3\log^3 n}{\theta^4 n^{s-1}}\right]\|f\|_2^2\,.\vspace{-1mm}$$
In particular, letting $s = C_1 + C_2 \log\left(\max({1}/{\theta},2)\right)$ for some constants $C_1,C_2$ and $N_0 = 2neL$ yields
\begin{equation}
\mathbb{E}\big(f(x_j)-\hat{f}(x_j)\big)^2 \leq \frac{\|f\|_2^2}{en}\,.
\label{eq:one_layer_contraction}
\end{equation}
\end{lemma}
The proof, given in Section~\ref{sec:proofs_of_lemmas},  follows from an application of Gaussian concentration.

\paragraph{Step 4: Iterative correction.}
We define $f^{0}:\{x_1,\dots,x_j\} \to \mathbb{R}$ by $f^{0}(x_j) := y_j$ where $y_j \in [0,1]$ are the desired labels for $x_j$. In the procedure above, we replace $f$ with $f^{0}$ and obtain the estimator $\hat{f}^{0}$ as per Equation~\eqref{e:empiricalMemorization}. We now define the remainder function $f^{\rem,1}:\{x_1,\dots,x_n\}\to \mathbb{R}$ as the error obtained by the approximation: $f^{\rem,1}(x_k):= f^{0}(x_k) - \hat{f}^{0}(x_k)$. Summing the bound in Equation~\ref{eq:one_layer_contraction} over $j\in [n]$ yields
\begin{equation}\label{eq:remainder_contraction}
\mathbb{E}\|f^{\rem,1}\|_2^2 \leq \frac{\|f^{0}\|_2^2}{e}\,.
\end{equation}
 We define higher order remainders $f^{\rem,l}$ for $l\geq 2$ inductively as follows. Suppose we have $f^{\rem,l-1}$. We replace $f$ in the procedure above with $f^{\rem,l-1}$ to obtain the estimator $\hat{f}^{\rem,l-1}$ as given in Equation~\eqref{e:empiricalMemorization}, independent of all the previous estimators. We define the remainder $f^{\rem,l} = f^{\rem,l-1} - \hat{f}^{\rem,l-1}$. Repeating the argument leading to Equation~\eqref{eq:remainder_contraction}, with the given choice of $s$ and $N_0$ we conclude that: $\mathbb{E}\|f^{\rem,l}\|_2^2 \leq e^{-l}\|f^{0}\|_2^2$. Take $N = lN_0$. Unrolling the recursion above, we note that $f^{\rem,l}(x)$ is $f^0(x) - \hat{f}^{l}(x)$, where $\hat{f}^{l}(x)$ is of the form 
\begin{equation}
\hat{f}^{(l)}(x)= \sum_{j=1}^{N} a_j \R\left(\tfrac{\theta\langle\xi_j,x\rangle}{2s\log n}-T_j\right)\,.\label{eq:final_memorization_estimator} 
\end{equation}
This is the output of a two-layer network with $lN$ $\R$ units. 
We recall that
 $(\xi_j,T_j)$ are i.i.d. $\mathcal{N}(0,\sigma^2I_d)\times \unif[-2,2]$ which agrees with the choice of weights in Theorem~\ref{thm:memorization_representation}. 
The remainder  $f^{\rem,l}(x)$ can be seen as the error of approximating $f^0$ using $N_0 l := N$ random activation functions as given in Equation~\eqref{eq:final_memorization_estimator}. By assumption, the labels $f^0(x_j) \in [0,1]$, so $\|f^0\|_2^{2} \leq n$. This gives us an error bound on the $L^2$ loss $\mathcal{E}_N(f^0):= \sum_{j=1}^{n}(f^{0}(x_j)-\hat{f}^{(l)}(x_j))^2$:
$$\mathbb{E}\mathcal{E}_{N}(f^0) \leq e^{-l}n \,.$$

\paragraph{Step 5: Markov's inequality.}
Denoting by $E_N(f^0)=e^{-l}n$ the RHS of the bound just above, Markov's inequality implies that for any $\delta \in (0,1)$
\begin{equation}\label{eq:markov_inequality}
\mathbb{P}\Big(\mathcal{E}_N(f^0) \geq \tfrac{E_N(f^0)}{\delta}\Big) \leq \delta\,.
\end{equation}
Now the choice $l \geq \log{n}+\log{\tfrac{1}{\delta}} + \log\tfrac{1}{\epsilon}$ gives
$\tfrac{E_N(f^0)}{\delta} \leq e^{-l}\tfrac{n}{\delta} \leq \tfrac{n}{\delta} e^{-\log{\tfrac{n}{\epsilon\delta}}} \leq \epsilon$ and plugging into Equation~\eqref{eq:markov_inequality} shows that when $s,l$ and $N$ are chosen as above, we have $\mathbb{P}(\mathcal{E}_N(f^0) \geq \epsilon) \leq \delta$ as claimed in Theorem~\ref{thm:memorization_representation}. 




 \section{Representation via the Corrective Mechanism}
 
 \label{s:approximation}
 We now turn to the representation problem.
Given a function $g :\mathbb{R}^{q}\to \mathbb{R}$, the goal is to construct a neural network whose output $\hat{g}$ is close to $g$. 
The arguments resemble those given in the previous section on memorization, but the details are more technically involved. 

The approximation guarantees of our theorems depend on certain \emph{Fourier norms}. These can be thought of as measures of the complexity of the function $g$ to be approximated.
 Let $g:\mathbb{R}^q \to \mathbb{R}$ be the function we are trying to approximate over the domain $B_q^2(r)$ and let $\frac{G(\omega)}{(2\pi)^q}$ be the `Fourier distribution' of $g:\mathbb{R}^q \to \mathbb{R}$ as defined in Equation~\eqref{eq:distribution_def}.  We take its magnitude-phase decomposition to be: $G(\omega) = |G(\omega)| e^{-i \psi(\omega)}$. For each integer $s\geq 0$ we define the \emph{Fourier norm} 
 $$C_g^{(s)} :=\frac{1}{(2\pi)^q}\int_{\mathbb{R}^q}|G(\omega)|\cdot\|\omega\|^{s} d\omega \,.
 $$
We will assume that $C_g^{(s)} < \infty$ for $s= 0,1,\dots,L$ for some $L \in \mathbb{N}$. 

Because having small Fourier norm can be thought of as a smoothness property, \emph{smoothed} $\R$ functions can be efficiently used for the task of approximating such functions. In Section~\ref{sec:smoothing_filter} we define a sequence of smoothed $\R$ functions $\sR_k$ for integers $k\geq 0$,
of increasing smoothness. These are obtained from the $\R$ by convolving with an appropriate function. The use of smoothed $\R$ functions is crucial in order that the \emph{remainder} following approximation is itself sufficiently smooth, which then allows the approximation procedure to be iterated. We start with the basic approximation theorem, which has an approximation guarantee as well as a smoothness guarantee on the remainder. \vspace{-2mm}
 
 \begin{theorem}
\label{thm:remainder_regularity}
Let  $k\geq \max(1,\frac{q-3}{4})$. Let $g :\mathbb{R}^{q}\to \mathbb{R}$ be such that $C_g^{(2k+2)},C_g^{(0)} < \infty$. Then, given any probability measure $\zeta$ over $B^2_q(r)$ there exists a two-layer $\sR_k$ network, with $N$  non-linear units, whose output is $\hat{g}(x)$ such that the following hold simultaneously:
\begin{enumerate}
\item $$\int (g(x) - \hat{g}(x))^2\zeta(dx) \leq \frac{C(r,k)\bigr(C_g^{(0)}+C_g^{(2k+2)}\bigr)^2}{N}\,.$$
\item There exists a function $g^{\rem}:\mathbb{R}^q \to \mathbb{R}$ such that:
\begin{enumerate}
\item For every $x\in B^2_q(r)$, $g^{\rem}(x) = g(x) - \hat{g}(x)$.
\item Its Fourier transform $G^{\rem} \in L^1(\mathbb{R}^q)\cap C(\mathbb{R}^q)$.
\item  For every $s < \frac{3-q}{2} + 2k$, $C_{g^{\rem}}^{(s)} \leq {C_1(s,r,q,k) (C_g^{(0)} + C_g^{2k+2})}/{\sqrt{N}}\,.$
\end{enumerate}   
\end{enumerate}
\end{theorem}

 We will use this theorem to give a faster approximation rate of $\frac{1}{N^{a+1}}$ for $g$, where $a \in \mathbb{N}\cup \{0\}$. We then extend this to functions of the from given in Equation~\eqref{eq:low_dim_function}. The main conclusion of the following theorem is that the approximating network achieves an error of at most $\epsilon$ with $N = O(C(a)\epsilon^{-\frac{1}{a+1}})$ activation functions. If the theorem below holds for every $a \in \mathbb{N}\cup\{0\}$, we note that if we take $a \to \infty$ slowly enough as $\epsilon \to 0$, we get subpolynomial dependence on $\epsilon$. 

\begin{theorem}\label{thm:fast_rates_part_1} 
Fix $q \in \mathbb{N}$ and for each $b \in \mathbb{N}\cup\{0\}$ let
\begin{equation}
k_b = \begin{cases}
b\bigr\lceil \tfrac{1+q}{4}\bigr\rceil  &\text{ if } q \not\equiv 3 \ (\mathrm{mod\ } 4)\\
b\left(\tfrac{1+q}{4}+1\right) &\text{ if } q \equiv 3\  (\mathrm{mod\ } 4)\,.
\end{cases}
\end{equation}
 Suppose $g:\mathbb{R}^q\to \mathbb{R}$ has bounded Fourier norms $C_g^{(0)} < \infty$ and $C_{g}^{(2k_a+2)} <\infty$ for some $a\in \{0\}\cup\mathbb{N}$. Then, for any probability measure $\zeta$ over $B^2_q(r)$, there exists a two-layer neural network with random weights and $N$ activation functions consisting of a mixture of  $\sR_{k_{b}}$ units for $b \in \{0,1,\dots,a\}$ with output $\hat{g}:\mathbb{R}^q \to \mathbb{R}$  such that
\begin{enumerate}
\item For every $x \in B_q^2(r)$, 
$\mathbb{E}\hat{g}(x) = g(x)$.
\item \vspace{-2mm}
$$ \mathbb{E}\int \big(g(x) -\hat{g}(x)\big)^2\zeta(dx) \leq C_0(q,r,a)\tfrac{\left(C_g^{(0)} + C_g^{(2k_a+2)}\right)^2}{N^{a+1}} \,.$$ 
\end{enumerate}
 The expectation here is with respect to the randomness in the weights of the neural network. Moreover, writing $\hat{g}$ in the form $\hat{g}(x) = \sum_{j=1}^{N}\kappa_j \sR_{k(j)}(\langle\omega_j,x\rangle-T_j)$, the $\kappa_j$ and $\omega_j$ satisfy
$\sum_{j=1}^{N}|\kappa_j| \leq C_1(q,r,a)(C_g^{0} + C_g^{2k_a+2})$ and $\|\omega_j\| \leq \frac{1}{r}$ almost surely.
\end{theorem}

\begin{proof}
	The main idea of the proof is to use Theorem~\ref{thm:remainder_regularity} repeatedly. We will first use $\sim N/(a+1)$ $\sR_{k_a}$ units to approximate $g$ by $\hat{g}^{(0)}$. This gives a squared error of the order $O(1/N)$. We then consider the error term $g - \hat{g}^{(0)}$ and approximate this error term using another $\sim N/(a+1)$ $\sR_{k_{a-1}}$ units and try to offset the first error to obtain a squared error guarantee of ${1}/{N^2}$, and repeat this procedure until we obtain the stated guarantees. We  reduce the smoothness parameter $k$ in every iteration as error terms become progressively less smooth. A complete proof is provided in Section~\ref{sec:main_thm_proofs}.
\end{proof}

\vspace{-2mm}
Now we prove the version of Theorem~\ref{thm:fast_rates_part_1} for functions of the form in Equation~\eqref{eq:low_dim_function}. The main advantage of Theorem~\ref{thm:fast_rates_part_2} is that the bounds do not depend on the dimension $d$, only on the effective dimension $q \ll d$. \vspace{-2mm}

\begin{theorem}\label{thm:fast_rates_part_2}
Consider the low-dimensional function defined in Equation~\eqref{eq:low_dim_function}. Assume that \vspace{-2mm}$$\sup_i \Big(C_{f_i}^{(0)}+C_{f_i}^{(2k_a+2)}\Big)^2 \leq M\,.\vspace{-2mm}$$ Then, for any probability measure $\zeta$ over $B^2_d(r)$, there exists a one non-linear layer neural network with $\R$ and $\sR_k$ units for $k \leq k_a$ with $N$ neurons with output $\hat{f}:\mathbb{R}^d \to \mathbb{R}$ such that
$$\int \big(f(x) -\hat{f}(x)\big)^2\zeta(dx) \leq C_0(q,r,a,l)\tfrac{Mm^a}{N^{a+1}}\,.$$

\end{theorem}\vspace{-2mm}
\begin{proof}
We use $N/m$ neurons to approximate each of the component functions $f_i$ just like in Theorem~\ref{thm:fast_rates_part_1}, and then average the outputs. The full proof is in Section~\ref{sec:main_thm_proofs}.
\end{proof}

%

We will now develop our results on learning low-degree polynomials. The results are based on Theorem~\ref{thm:fast_rates_sup_case_2} which is similar to Theorem~\ref{thm:fast_rates_part_2}, but with a stronger bounded sup norm type assumption on the Fourier transform instead. This has the advantage that we can sample the weights independent of the target function $g$ and $k$ to construct our network.  The proofs are developed in Section~\ref{sec:function_independent_sampling}, which is roughly similar to Section~\ref{sec:unbiased_estimators}.

 Let the probability measure $\mu_l$ over $\mathbb{R}$ be defined by $\mu_l(dt) \propto \frac{dt}{1+t^{2l}}$ for $l \in \mathbb{N}$.  Given $a,m, N \in \mathbb{N}\cup\{0\}$ such that $\frac{N}{(a+1)m}\in \mathbb{N}$ and the orthonormal sets $B_i$ be as used in Equation~\eqref{eq:low_dim_function}, consider the following sampling procedure:
\begin{enumerate}
\item Partition $[N] \subseteq \mathbb{N}$ into $m$ disjoint sets, each with $N/(m(a+1))$ elements. 
\item  For $i \in [m]$, $b\in \{0,\dots,a\}$, $j\in [\tfrac{N}{m(a+1)}]$, we draw $\omega^{0}_{i,j,b} \sim \unif\left(\mathbb{S}^{\mathrm{span}(B_i)}\right)$ and $T_{i,j,b} \sim \mu_l$ independently for some $l \geq \max(q+3,3a+3)$.
\end{enumerate}


We now specialize to the low degree polynomials defined in Equation~\eqref{eq:low_degree_poly_def}. Define the following orthonormal set associated with each $V$ in the summation:
$$B_V = \{e_j:  V(j)\neq 0\} \cup \bar{B}_V\,,$$
 where $e_j$ are the standard basis vectors in $\mathbb{R}^d$ and $\bar{B}_V \subseteq \{e_1,\dots,e_d\}$ is chosen such that $|B_V| = q$. Consider the sampling procedure given above with the bases $B_V$. Since the bases $B_V$ are known explicitly, this sampling can be done without knowledge of the polynomial. We have the following theorem about learning low-degree polynomials, proved in Section~\ref{sec:main_thm_proofs}.\vspace{-2mm}

\begin{theorem}\label{thm:polynomial_approximation}
Let  $m = {{q+d}\choose{q}}$, $r = \sqrt{q}$ and let $J$ be the $m$-dimensional vector whose entries are $J_V$. 
Let $a \in \mathbb{N}\cup\{0\}$, $\delta \in (0,1)$ and $\epsilon,R_{c} > 0$ be arbitrary. Let $N$ be chosen such that and $N/(a+1)m \in \mathbb{N}$. Let $\zeta$ be any probability measure over $[0,1]^d$. Generate the weights $(\omega_{i,j,b}^{0},T_{i,j,b})$ according to the sampling procedure described above. Construct the two-layer neural network with $N$ activation functions
\vspace{-3mm}
\begin{equation}
\hat{f}(x;\mathbf{v}) = \sum_{i=1}^{m}\sum_{b=0}^{a}\sum_{j=1}^{ \frac{N}{m(a+1)}} v_{i,j,b}\sR_{k^{S}_b}\left(\tfrac{\langle \omega_{i,j,b}^{0},x\rangle}{r} - T_{i,j,b}\right)\,.
\end{equation}
Here we have denoted the vector comprising of $v_{i,j,k}$ by $\mathbf{v}$. Let $\mathbf{v}^{*} \in \arg \inf_{\mathbf{v} \in B_N^2(R_c)}\int (f(x)-\hat{f}(x;\mathbf{v}))^2\zeta(dx) $. Let $b \in \mathbb{N}\cup\{0\}$, $b\leq a$. There exists a constant $C(a,q,l)$ such that if 
$$N \geq C(a,q,l)\max\left( \tfrac{\delta^{-1/(b+1)}\|J\|_{2(b+1)}^2m^{2-1/(b+1)}}{R_{c}^2}, \Big(\tfrac{\|J\|_2^{2}}{\epsilon\delta}\Big)^{\frac{1}{a+1}}m \right)\,,
$$
then with probability at least $1-\delta$, \vspace{-1mm}
 $$\int \big(f(x)-\hat{f}(x;\mathbf{v}^{*})\big)^2d\zeta \leq \epsilon \,.\vspace{-1mm}$$
Moreover, we can obtain the coefficients $ v^{*}_{i,j,b}$ using GD over the outer layer only since this is a convex optimization problem.
\end{theorem}

 \section{Acknowledgments}
This work was supported in part by MIT-IBM Watson AI Lab and NSF CAREER award CCF-1940205.
 
 \newpage

\bibliography{references}
\newpage
 \appendix
\section{Smoothed ReLU functions and Integral Representations}
\label{sec:smoothing_filter}
In this section we introduce the necessary technical results and constructions for function approximation by smoothed $\R$ units $\sR_k$, as used in Theorems~\ref{thm:remainder_regularity} and~\ref{thm:fast_rates_part_1}. In Theorem~\ref{thm:cosine_representation} we show that cosine functions can be represented in terms of $\sR_k$ functions similar to Step 2 in Section~\ref{sec:memorization}. This will later be used along with the Fourier inversion formula to represent the target function $g$ in terms of the activation functions $\sR_k$. We note that this idea is taken from~\cite{barron1993universal} and~\cite{klusowski2018approximation}. All of the results stated here are proved in Section~\ref{sec:integral_representations}. 
\subsection{Smoothing the ReLU} 
Consider the triangle function 
\begin{equation}
\lambda(t) = \begin{cases}
1 - |t| &\text{ for } |t| \leq 1 \\
0 & \text{ otherwise}\,.
\end{cases}
\end{equation}
Clearly, $\lambda$ is a symmetric, bounded and continuous probability density over $\mathbb{R}$. Denoting the Fourier transform of $\lambda$ by $\Lambda$, one can verify the standard fact that $\Lambda(\xi) = \frac{\sin^2 (\xi/2) }{(\xi/2)^2} $. 

We also consider $k$-fold convolution of $\lambda$ with itself: 
Let $\lambda_1 := \lambda$ and $\lambda_{l+1} := \lambda_1 \ast \lambda_l$ for $l \geq 1$. For each $k\geq1$ the function $\lambda_k$ has support $[-k,k]$, it is a symmetric, bounded and continuous probability density over $\mathbb{R}$, and its Fourier transform is $\Lambda_k(\xi) = \frac{\sin^{2k} (\xi/2) }{(\xi/2)^{2k}}$. For arbitrary $w_{0} > 0$, we define $\lambda_{k,w_{0}}(t) := \frac{k}{w_{0}}\lambda_k(\frac{tk}{w_{0}})$, which can also be verified to be a symmetric, continuous probability density over $\mathbb{R}$ with support $[-w_{0}, w_{0}]$, and its Fourier transform is given by $\Lambda_{k,w_{0}}(\xi) = \Lambda_k(\frac{\xi w_{0}}{k})$. 

We now ``cosine regularize'' $\lambda_k$ so that its Fourier transform is non-zero everywhere. This transformation is for purely technical reasons and is useful in the proof of Theorem~\ref{thm:cosine_representation} stated below. Let $\alpha_0 > 0$ and $w_{0} \leq \min(\frac{\pi}{2\alpha_0},\frac{\pi k}{4\alpha_0})$ and define 
\begin{equation}\label{eq:cosine_regularized_filter}
    \lambda^{\alpha_0}_{k,w_{0}}(t) := \cos(\alpha_0 t)\lambda_{k,w_{0}}(t)\bigg/\int_{-\infty}^{\infty}\cos(\alpha_0 T)\lambda_{k,w_{0}}(T)dT\,.
\end{equation}
 The constraints given on $\alpha_0$ and $w_0$ ensure that $\lambda^{\alpha_0}_{k,w_{0}}(t) \geq 0$ for every $t$. We will henceforth think of $\alpha_0$ and $w_{0}$ as fixed (say $w_0 = 0.5$ and $\alpha_0 = \frac{\pi}{16}$).  We define the \emph{smoothed} $\R$ functions $$\sR_k := \R \ast \lambda_{k,w_{0}}^{\alpha_0}\quad \text{for all } k\geq 1$$ and hide the dependence on $w_0,\alpha_0$. Clearly, $\sR_k$ is an increasing, positive function and $\sR_k(t) = \R(t)$ whenever $t \notin (-w_{0},w_{0})$. We follow the convention that for $k=0$, $\sR_k= \R$. In particular, $\sR_k(t) = 0$ whenever $t \leq -w_0$. We give an illustration of these functions in Figure~\ref{fig:example}. The higher the value of $k$, the smoother the function is at $0$. In the sequel, whenever we say ``smoothed by filter $\lambda_{k,w_{0}}^{\alpha_0}$", we mean convolution with the function $\lambda_{k,w_{0}}^{\alpha_0}$.


\begin{figure}
    \centering
    \subfigure[$\R$]{{\includegraphics[width=7cm]{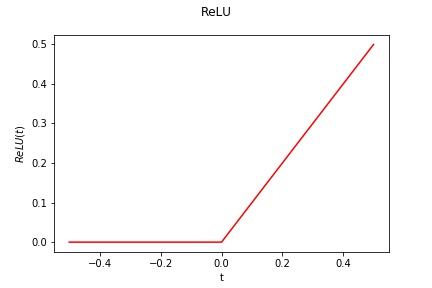} }}%
    \qquad
    \subfigure[$\sR_1$]{{\includegraphics[width=7cm]{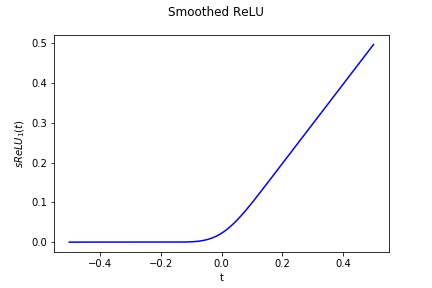} }}%
    \caption{Illustrating $\R$ and $\sR$ activation functions.}%
    \label{fig:example}%
\end{figure}

\begin{theorem}[Cosine Representation Theorem]\label{thm:cosine_representation}
 Consider the probability measure $\mu_l$ over $\mathbb{R}$ given by $\mu_l(dT) = \frac{c_{\mu_l}dT}{1+T^{2l}}$ (here $c_{\mu_l}$ is the normalizing constant). Let $\alpha,\psi \in \mathbb{R}$ be given. There exists a continuous function $\kappa :\mathbb{R} \to \mathbb{R}$ depending on $\alpha,\theta,l,k,w_{0}$ and $\alpha_0$ such that $\|\kappa\|_{\infty} \leq C(k,l)\left(1+|\alpha|^{2k+2}\right) $ and for every $t \in [-1,1]$
$$\cos(\alpha t + \psi) = \int_{-\infty}^{\infty} \kappa(T)\sR_k(t-T)\mu_l(dT)\,.$$
\end{theorem}

\begin{remark}
We note that the upperbound on $\kappa$ gets worse as the smoothness parameter $k$ gets larger. This is due to the fact that smoother activation functions find it harder to track fast oscillations in $\cos(\alpha t + \psi)$ as $\alpha$ gets larger.
\end{remark}



 \section{Proof of Theorem~\ref{thm:remainder_regularity}: Unbiased Estimator for the Function and its Fourier transform}
 \label{sec:unbiased_estimators}

Through the following steps, we describe the proof of Theorem~\ref{thm:remainder_regularity}, which was in turn used to prove Theorem~\ref{thm:fast_rates_part_1}.
\paragraph{Step 1: Representing $g$ in terms of $\sR_k$}
Consdier the setup in section~\ref{s:approximation} and assume $C_g^{(2k+2)}, C_g^{(0)} < \infty$. From Fourier inversion formula, using the fact that $g$ is real-valued, it follows that
\begin{align}
    g(x) &= \int_{\mathbb{R}^q}\cos(\langle\omega,x\rangle+\psi(\omega))\frac{|G(\omega)|}{(2\pi)^q}d\omega \,.\label{eq:fourier_representation}
\end{align}

We combine Theorem~\ref{thm:cosine_representation} and Equation~\eqref{eq:fourier_representation} to show the following integral representation for $g$. The proof is given in Section~\ref{sec:integral_representations}.

\begin{theorem}\label{thm:smooth_relu_representation}
Let $\mu_l$ be the probability measure defined by its density $\mu_l(dt) \propto \frac{dt}{1+t^{2l}} $ for a given $l\in \mathbb{N}$.  Define the probability distribution $\nu_{g,k}$ by $\nu_{g,k}(d\omega) = \frac{1+r^{2k+2}\|\omega\|^{2k+2}}{C_g^{(0)}+r^{2k+2}C_g^{(2k+2)}}\frac{|G(\omega)|}{(2\pi)^q}d\omega $.  For every $x \in B_q^{(2)}(r)$
\begin{equation}\label{eq:smooth_relu_representation}
g(x) =   \beta_{g,k}\iint \eta(T;r,\omega)\sR_k\left(\tfrac{\langle\omega,x\rangle}{r\|\omega\|}-T\right)\mu_l\times\nu_{g,k}(dT\times d\omega)\,,
\end{equation}
where $|\eta(T;r,\omega)| \leq 1$ almost surely with respect to measure $\mu_l\times \nu_{g,k}$ and $\eta(T;r,\omega) = 0$ whenever $T > 1+w_{0}$ and $ \beta_{g,k} := \big(C_g^{(0)}+r^{2k+2}C_g^{(2k+2)}\big)C(k,l)$
\end{theorem}

\begin{remark}
The case $\omega = 0$ might appear ambiguous in the integral representations above. But following our discussion preceding Theorem~\ref{thm:smooth_relu_representation}, we use the convention that $\frac{\langle\omega,x\rangle}{\|\omega\|} := 0$ whenever $\omega = 0$. We check that even the constant function can be represented as an integral in Theorem~\ref{thm:cosine_representation} by setting $\alpha = 0$.
\label{rem:omega_zeros_case}
\end{remark}

\begin{remark}
The probability measure $\nu_{g,k}$ depends on the function $g$ and can be complicated. Therefore, when training a neural network, it is not possible to sample from it since $g$ is unknown. We only use the existence of this measure to prove representation theorems as found in the literature (see~\cite{barron1993universal},~\cite{klusowski2018approximation}). To give the training results, we will impose more conditions on $G$ and show that we can get similar representation theorems when a known, fixed measure $\nu_0$  is used  instead of $\nu_{g,k}$. This is done in Section~\ref{sec:function_independent_sampling}.\end{remark}

We start by converting Theorem~6, the integral representation of $g$ in terms of $\sR_k$ units, into a statement about existence of a good approximating network $\hat g$. 

\paragraph{Step 2: Empirical estimate.}
Let $\mu_l \times \nu_{g,k}$, $\eta$ and $\beta_{g,k}$ be as given by Theorem~\ref{thm:smooth_relu_representation}. For $j \in \{1,\dots,N\}$, draw $(T_j,\omega_j)$ to be i.i.d. from the distribution $\mu_l\times \nu_{g,k}$.  Let $\theta^{u}_j$ for $j \in [N]$ be i.i.d. $\unif[-1,1]$ and independent of everything else. We define 
$$\theta_j := \ind\left(\theta_j^u < \eta(T_j;r,\omega_j)\right) - \ind\left(\theta_j^u \geq \eta(T_j;r,\omega_j)\right)$$
and observe that $\theta_j \in \{-1,1\}$ almost surely and $\mathbb{E}\left[\theta_j|T_j,\omega_j\right] = \eta(T_j;r,\omega_j)$. That is, it is an unbiased estimator for $\eta(T_j;r,\omega_j)$ and independent of other $\theta_{j^{\prime}}$ for $j\neq j^{\prime}$. 

Now define the estimate $\hat{g}_j(x)$ based on a single $\sR_k$ unit
\begin{equation}\label{eq:unbiased_single_estimator}
\hat{g}_j(x) := 
\beta_{g,k}\theta_j\sR_k\left(\tfrac{\langle\omega_j,x\rangle}{r\|\omega_j\|}-T_j\right)\ind(T_j \leq 1+w_{0} )
\end{equation}
where we have made the dependence of $\hat{g}_j$ on $T_j,\omega_j$ implicit. We also define the empirical estimator 
$$ \hat{g}(x) = \frac{1}{N}\sum_{j=1}^{N}\left[\hat{g}_j(x)\right]\,.$$
Note that $\hat{g}$ is the output of a two-layer neural network with one $\sR_k$ layer and one linear layer.

\begin{theorem}\label{thm:one_layer_approximation}
Consider the probability measure $\mu_l$ with density $\mu_l(dt) \propto dt/(1+t^{2l})$ and let $T_j \sim \mu_l$ for $l\geq 2$ (so that $\mathbb{E}|T_j|^2 < \infty$). Then
\begin{enumerate}
    \item For every $x\in B^2_q(r)$,
          $g(x) = \mathbb{E}\hat{g}_j(x)$.
    \item 
          For every $x\in B^2_q(r)$,
          $\mathbb{E}(g(x)-\hat{g}(x))^2 \leq  \frac{C\beta_{g,k}^2}{N}$.
    \item There is a constant $C$ depending on $l$ such that for any probability distribution $\zeta$ over $B^2_q(r)$, 
$$\mathbb{E}\int\big(g(x)-\hat{g}(x)\big)^2 \zeta(dx) \leq \frac{C\beta_{g,k}^2}{N} \,. $$    
    Therefore there exists a configuration a choice of $(T_j,\omega_j,\theta_j)$ such that
          $$\int\big(g(x)-\hat{g}(x)\big)^2 \zeta(dx) \leq \frac{C\beta_{g,k}^2}{N} \,.$$
\end{enumerate}
\end{theorem}


\begin{proof}
The first item follows from Theorem~\ref{thm:smooth_relu_representation}.
For the second item, let $x \in B^2_q(r)$. Since $\hat{g}_j(x)$ is an unbiased estimator for $g(x)$ as shown in Item 1, we conclude that :
    $$\mathbb{E}\big(g(x)-\hat{g}(x)\big)^2 = \frac{1}{N}\left[\mathbb{E}\big(\hat{g}_j(x))^2-(g(x)\big)^2\right] \leq \frac{1}{N}\mathbb{E}\big(\hat{g}_j(x)\big)^2$$
    Now $|\hat{g}_j(x)| \leq \beta_{g,k}(1+|T_j|+w_0)$.
Squaring and taking expectations on both sides yields the result, using that $\mathbb{E}|T_j|^2 <\infty$ since $ l \geq 2$. For Item 3, we use Fubini's theorem and Item 2 to conclude that
    $$\mathbb{E}\int(g(x)-\hat{g}(x))^2 \zeta(dx) \leq \frac{C\beta_{g,k}^2}{N} \,.$$
    The desired bound holds in expectation, so it must also hold for some configuration.
\end{proof}

Note that in Theorem~\ref{thm:one_layer_approximation}, the RHS of the error upper bounds depend on the Fourier norm $ C_{g}^{(s)}$. As explained in Section~\ref{sec:main_idea}, in order to apply the corrective mechanism we need to consider $g^{\rem}(x) = g(x) - \hat{g}(x)$ for $x\in B_q^2(r)$ and show that, roughly, the corresponding Fourier norm $C^{(s)}_{g^{\rem}} \leq C\frac{ C_g^{(s)}}{\sqrt{N}}$. Since Fourier transform is a linear mapping, an unbiased estimator for $g$ (i.e, $\hat{g}$) should be such that the Fourier transform of $\hat{g}$ (i.e, $\hat{G}(\xi)$) is an unbiased estimator for $G(\xi)$ for every $\xi \in \mathbb{R}^q$. There are several technical roadblocks to this argument:
\begin{enumerate}
\item $\hat{g}(x)$ is only an unbiased estimator when $x \in B_q^2(r)$.
\item $\hat{g}_j(x)$ is a `one dimensional function' - that is it depends only on $\langle\omega_j,x\rangle$. This makes its Fourier transform contain tempered distributions like dirac delta and we cannot apply a variance computation to show that the Fourier transform contracts by $1/\sqrt{N}$.
\item $\hat{g}_j(x)$, even along the direction $\langle\omega_j,x\rangle$ is not well behaved since $\sR_k(\cdot)$ is not compactly supported. Therefore this is not an $L^1$ function and hence its Fourier transform isn't very well behaved.  
\end{enumerate}

We resolve the issues above by considering the fact that we only care about the values of $g$ (and $\hat{g}$) in $B^2_q(r)$ and hence we are free to modify $g$ (and $\hat{g}$) outside this domain. Along these lines, we modify $g$ to $g(\,\cdot\,;R)$ and $\hat{g}_j$ to $\hat{g}_j(\,\cdot\,;R)$. Ultimately, we will show the existence of $g^{\rem}: \mathbb{R}^q \to \mathbb{R}$ such that $g^{\rem}(x) = g(x) - \hat{g}(x)$ whenever $x \in B_q^2(r)$ and such that its Fourier transform is `well behaved enough' to carry out the corrective mechanism describe above and in Section~\ref{sec:main_idea}.  As a first step towards modification, we resolve item 3 first above by replacing $\sR_k$ by smoothed triangles $\sDelta_k$ as defined below. This compactifies $\hat{g}_j$ along the direction $\omega_j$.

\paragraph{Step 3: Replacing $\sR$ by smoothed triangles.}

In the notations used below, we hide the dependence on $r,w_0,\alpha_0$ and $l$ for the sake of clarity. Consider the statement of Theorem~\ref{thm:smooth_relu_representation} for every $x\in B^2_q(r)$:
\begin{equation}\label{eq:pre_surgery_representation}
g(x) =   \beta_{g,k}\iint \eta(T;r,\omega)\sR_k\left(\tfrac{\langle\omega,x\rangle}{r\|\omega\|}-T\right)\mu_l(dT)\nu_{g,k}(d\omega)\,.
\end{equation}
  For $t\in \mathbb{R}$, let
  $$\sDelta_k\left(t;T\right):= \sR_k\left(t-T\right) - 2\sR_k\left(t-1-w_{0}\right) + \sR_k\left(t-2-2w_{0} + T\right)\,,$$ 
  and 
  $$\Delta\left(t;T\right):= \R\left(t-T\right) - 2\R\left(t-1-w_{0}\right) + \R\left(t-2-2w_{0} + T\right) \,.$$ 
  
  Note that $\Delta = \sDelta_0$. Clearly, when $T \leq 1+w_{0}$ and $x\in B_q^2(r)$, we have $$\sR_k\Big(\tfrac{\langle\omega,x\rangle}{r\|\omega\|}-T\Big) = \sDelta_k\Big(\tfrac{\langle\omega,x\rangle}{r\|\omega\|},T\Big) \,,$$
and $\eta(T;r,\omega) = 0$ whenever $T > 1+ w_0$. Therefore, we can replace $\sR_k$ with $\sDelta_k$ in Equation~\eqref{eq:pre_surgery_representation}. When $T \leq 1 + w_0$, $\Delta(\spacedot;T):\mathbb{R}\to \mathbb{R}$ gives a triangle graph as can be easily verified and hence is compactly supported. Its Fourier transform is an $L^1$ function. $\sDelta_k$ is obtained by convolving $\Delta$ with the filter $\lambda_{k,w_0}^{\alpha_0}$. We refer to Figure~\ref{fig:example_triangle} for an illustration. Lemma~\ref{lem:post_surgery_representation} below follows from the preceding discussion.


\begin{figure}[ht]
    \centering
    \subfigure[$\Delta$]{{\includegraphics[width=7cm]{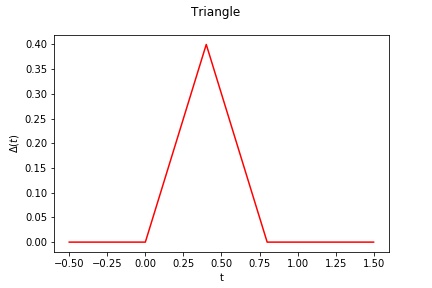} }}%
    \qquad
    \subfigure[$\sDelta_1$]{{\includegraphics[width=7cm]{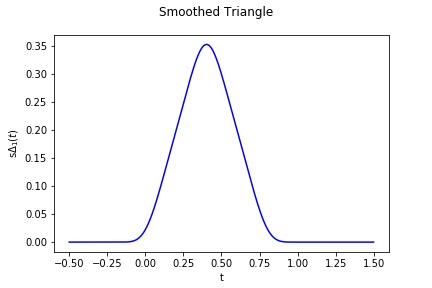} }}%
    \caption{Illustrating $\Delta$ and $\sDelta_1$ activation functions.}%
    \label{fig:example_triangle}%
\end{figure}

\begin{lemma}\label{lem:post_surgery_representation}
 For every $x\in B^2_q(r)$,
 $$g(x)  =  \beta_{g,k} \iint \eta(T;r,\omega)\sDelta_k\left(\tfrac{\langle\omega,x\rangle}{r\|\omega\|},T\right)\mu_l(dT)\nu_{g,k}(d\omega)\,.$$
 \end{lemma}
 
 Consider the technical issues listed before Step 3. We resolved item 3 in Step 3 above. In Step 4 below, will resolve item 2 by modifying $\hat{g}_j$ by first replacing $\sR_k$ with $\sDelta_k$ as in Step 3 to `compactify' it along the direction $\omega_j$ and then `mollify' it along the perpendicular directions by multiplying it with a function which is $1$ in $B_q^2(r)$ and vanishes outside a compact set to obtain $\hat{g}_j(\,\cdot\,;R)$. To resolve item 1, we define $g(x;R)$ to be the expectation of $\hat{g}_j(x;R)$ for every $x$.  As a consequence we have show that for $\xi \in \mathbb{R}^q$, the fourier transform of $\hat{g}_j(\,\cdot\,;R)$, given by $\hat{G}_j(\xi;R)$ is an unbiased estimator for $G(\xi;R)$ which is the Forier transfrom of $g(\,\cdot\,;R)$.
 \paragraph{Step 4: Truncation and modification}

 Let $\gamma \in \sch$ be the function defined in Section~\ref{subsec:notations} - such that $\gamma(t) \geq 0$ for every $t\in \mathbb{R}$, $\gamma(t) = 0 $ when $|t| \geq 2$ and $\gamma(t) = 1$ for every $t\in [-1,1]$. Let $R \geq r$  and $q > 1$.  For every $x\in \mathbb{R}^d$ and $\omega \neq 0$, we define $\gamma_{\omega}^{\perp}(x) := \gamma\left(\frac{\|x\|^2-\frac{1}{\|\omega\|^2}\langle \omega,x\rangle^2}{R^2}\right)$ when $q > 1$. We use the convention that when $\omega = 0$, $\frac{1}{\|\omega\|^2}\langle \omega,x\rangle^2 := 0$ as stated in Remark~\ref{rem:omega_zeros_case}. When $q =1$, we let $\gamma_{\omega}^{\perp}(x) := 1$ for every $x$.  Let $l\geq 2$. Draw $(T_j,\omega_j)$ i.i.d. from the distribution $\mu_l\times \nu_{g,k}$ and let the random variable $\theta_j$ be as in Equation~\eqref{eq:unbiased_single_estimator}. Define $\hat{g}_j(\spacedot;R):\mathbb{R}^q \to \mathbb{R}$:
 \begin{equation}\label{e:unbiased-single-estimator-R}
 \hat{g}_j(x;R) := \begin{cases}
 0 &\text{ when } T_j > 1 + w_{0} \\
  \beta_{g,k}  \theta_j\sDelta_k\left(0,T_j\right)\gamma(\frac{\|x\|^2}{R^2}) &\text{ othwerwise when } \omega_j = 0 \\
 \beta_{g,k}  \theta_j\sDelta_k\left(\tfrac{\langle\omega_j,x\rangle}{r\|\omega_j\|},T_j\right)\gamma_{\omega_j}^{\perp}(x) &\text{ otherwise}\,.
 \end{cases}
 \end{equation}
 
We also define $g(\spacedot;R):\mathbb{R}^q \to \mathbb{R}$ by
\begin{equation}\label{eq:extension_definition}
  g(x;R) =  \mathbb{E}\hat{g}_j(x;R)\,.
  \end{equation}
The expectation on the RHS exists for every $x$ whenever $l \geq 2$ because then $\mathbb{E}_{T\sim \mu_l}|T| <\infty$. We note that $g(x;R)$ and $\hat{g}_j(x;R)$ are both implicitly dependent on $k,l,\alpha_0,w_{0}$. 
 Let $\hat{G}_j(\xi;R)$ be the Fourier transform of $\hat{g}_j(x;R)$ and let $G(\xi;R)$ be the Fourier transform of $g(x;R)$. Even though we allowed the Fourier distribution $G/(2\pi)^d$ to be singular entities like $\delta$ measures, we will see that for our extension, we show below that the $G(\xi;R)$ is a $L^1(\mathbb{R}^q)\cap C(\mathbb{R}^q)$ function. This allows us to construct estimators for $G(\spacedot;R)$. In the lemma below we construct an unbiased estimator for $g(\spacedot;R)$, whose Fourier transform is an unbiased estimator for $G(\spacedot;R)$.  
 
 Let $\Gamma_{q,R}$ be the Fourier transform of $\gamma(\|x\|^2/R^2)$. We conclude from spherical symmetry of the function $\gamma(\|x\|^2/R^2)$ that $\Gamma_{q,R}(\xi)$ is a function of $\|\xi\|$ only. When convenient, we will abuse notation and replace $\Gamma_{q,R}(\xi)$ by $\Gamma_{q,R}(\|\xi\|)$. We note some useful identities in Lemma~\ref{lem:unbiased_estimators} and give its proof in Section~\ref{sec:proofs_of_lemmas}.

 \begin{lemma}\label{lem:unbiased_estimators}
Let $\mu_l$ be the probability measure defined in Theorem~\ref{thm:cosine_representation}. Let $l\geq 2$ so that $\mathbb{E}_{T\sim \mu_l}T^2 < \infty$. 
 \begin{enumerate} 
 \item For every $x \in B^2_q(r)$,
   $$g(x;R) = g(x)\quad \text{and} \quad \hat{g}_j(x;R) = \hat{g}_j(x)\,,$$
where $\hat{g}_j(x)$ is as defined in~\eqref{eq:unbiased_single_estimator}.
 \item   $\hat{g}_j(\spacedot;R) \in L^1(\mathbb{R}^q)$ almost surely and $g(\spacedot;R) \in L^1(\mathbb{R}^q)$

  \item For every $\xi,\omega \in \mathbb{R}^q$ such that $\omega \neq 0$ we define $\xi_{\omega} := \frac{\langle \xi,\omega\rangle}{\|\omega\|} \in \mathbb{R}$ and $\xi_{\omega}^{\perp} := \xi - \frac{\omega \langle \xi,\omega\rangle}{\|\omega\|^2}$. For any fixed value of $T_j$ and $\omega_j$:  
   \begin{equation}
    \hat{G}_j(\xi;R) = \begin{cases}
    0 \text{ if } T_j > 1+w_{0}  \\
    \beta_{g,k} \theta_j\Gamma_{q,R}(\|\xi\|)\sDelta_k(0;T_j) \text{ when } T_j \leq 1+w_0 \text{ and }\omega_j = 0 \\
    \beta_{g,k} \theta_j\Gamma_{q-1,R}(\|\xi_{\omega_j}^{\perp}\|)\Lambda_{k,w_{0}}^{\alpha_0}(\xi_{\omega_j})\left[\frac{4e^{i(1+w_{0})r\xi_{\omega_j} }}{\xi_{\omega_j}^2r}\sin^2((1+w_{0}-T)\xi_{\omega_j} r/2)\right] \\\text{ otherwise }
    \end{cases}
    \end{equation}
    Here we stick to the convention that RHS is $\frac{\langle\omega,x\rangle}{\|\omega\|}=0$ when $\omega_j =0$ and when $q =1$, we let $\Gamma_{q-1,R}(\cdot) = 1$. We recall that $\Lambda^{\alpha_0}_{k,w_0}$ is the Fourier transform of the filter $\lambda^{\alpha_0}_{k,w_0}$.
  \item $\hat{G}_j(\spacedot;R) \in L^1(\mathbb{R}^q)$ almost surely, $G(\spacedot;R) \in L^1(\mathbb{R}^q)$ and for every $\xi \in \mathbb{R}^q$,
   $$G(\xi;R) =  \mathbb{E}\hat{G}_j(\xi;R)\,.$$

 \end{enumerate}
 \end{lemma}

\paragraph{Step 5: Controlling Fourier norm of remainder term.}
 As per Theorem~\ref{thm:one_layer_approximation}, $g(x)$ is approximated by  $ \frac{1}{N}\sum_{j=1}^{N}\hat{g}_j(x)$ up to a squared error of the order $\frac{1}{N}$ and $\frac{1}{N}\sum_{j=1}^{N}\hat{g}_j(x)$ is the output of a two-layer $\sR_k$ network with $N$ non-linear activation functions. We will now consider the remainder term: $ g(x) - \frac{1}{N}\sum_{j=1}^{N}\hat{g}_{j}(x)$. Since we are only interested in $x \in B^2_q(r)$, we can define the following version of the remainder term using the truncated functions $g(x;R)$ and $\hat{g}_j(x;R)$:
$$g^{\rem}(x) := g(x;R)  - \frac{1}{N}\sum_{j=1}^{N} \hat{g}_j(x;R)\,.$$
 We will now show that the expected `Fourier norm'  of $g^{\rem}(x)$ is smaller by an order of $\frac{1}{\sqrt{N}}$.  We note that $g^{\rem}$ is a `random function' such that $\mathbb{E}g^{\rem}(x) = 0$ for every $x$. Let  $G^{\rem}$ be the Fourier transform of $g^{\rem}$.

\begin{lemma}\label{lem:fourier_contraction}
\sloppy Recall the probability measure $\mu_l$ from Theorem~\ref{thm:cosine_representation}. Let $l = 3$ so that $\mathbb{E}_{T\sim \mu_l}T^4 < \infty$ and let $R = r$. For $ s \in \{0\}\cup \mathbb{N}$, consider
$$C_{g^{\rem}}^{(s)} := \int_{\mathbb{R}^q} \|\xi\|^s \cdot|G^{\rem}(\xi)|d\xi \,.$$
Whenever $k \geq \max(1,\frac{q-3}{4})$ and $s < \frac{3-q}{2} + 2k$, we have that
$$\mathbb{E}C_{g^{\rem}}^{(s)} \leq \frac{C (C_g^{(0)} +C_g^{(2k+2)})}{\sqrt{N}}\,,$$
where $C$ is a constant depending only on $s,r,q$ and $k$.
\end{lemma}

We give the proof in Section~\ref{sec:proofs_of_lemmas}. It is based on Item 4 in Lemma~\ref{lem:unbiased_estimators}, which ensures that $|G^{\rem}|$ is of the order $\frac{1}{\sqrt{N}}$ in expectation. The technical part of the proof involves controlling the integral with respect to the Lebesgue measure using a polar decomposition. 

We now combine the results above to complete the proof of Theorem~\ref{thm:remainder_regularity}. The proof applies Markov's inequality to the results in Theorem~\ref{thm:one_layer_approximation} and Lemma~\ref{lem:fourier_contraction}.  Let $\hat{g}$ and $g^{\rem}$ be defined randomly as in the discussion above. By Markov's inequality:
\begin{enumerate}
\item There is a constant $C^{\prime}$ such that with probability at least $3/4$, 
$$\int(g(x)-\hat{g}(x))^2\zeta(dx) \leq \frac{C^{\prime}\left(C_g^{(0)}+C_g^{(2k+2)}\right)^2}{N}\,.$$
\item There is a constant 
$C_1^{\prime}$
such that with probability at least $3/4$,
$$C_{g^{\rem}}^{(s)} \leq \frac{C^{\prime}_1 (C_g^{(0)} + C_g^{2k+2})}{\sqrt{N}}\,.$$
\end{enumerate}
By the union bound, with probability at least $1/2$ both the inequalities above hold, and hence these must hold for some configuration.

\section{Integral Representations for Cosine Functions}
\label{sec:integral_representations}
 The objective of this section is to prove Theorem~\ref{thm:cosine_representation}. 
 
The Lemmas~\ref{lem:non_vanishing} and~\ref{lem:filter_bounds} below establish important properties of the the filter $\lambda_{k,w_{0}}^{\alpha_0}$ and will be used extensively in the sequel. Their proofs are given in Section~\ref{sec:integral_representations}.
\begin{lemma}\label{lem:non_vanishing}
$\lambda^{\alpha_0}_{k,w_{0}}(t)$ as defined in Equation~\eqref{eq:cosine_regularized_filter} is a symmetric, continuous probability density over $\mathbb{R}$ which is supported over $[-w_{0},w_{0}]$. Its Fourier transform $\Lambda^{\alpha_0}_{k,w_{0}}$ is such that $1 \geq \Lambda^{\alpha_0}_{k,w_{0}}(\xi) >0$ for every $\xi$.
\end{lemma}
\begin{proof}
The first part of the Lemma follows directly from the definition. Let $$C_{\alpha_0} := \int_{-\infty}^{\infty}\cos(\alpha_0 T)\lambda_{k,w_0}(T)dT > 0\,.$$ For the second part, we observe that 
\begin{align*}
\Lambda^{\alpha_0}_{k,w_{0}}(\xi) &= \frac{1}{2C_{\alpha_0}}\left[\Lambda_k\left(\tfrac{(\xi+\alpha_0)w_{0}}{k}\right) + \Lambda_k\left(\tfrac{(\xi-\alpha_0)w_{0}}{k}\right)\right]\\
&= \frac{1}{2C_{\alpha_0}}\left[\frac{\sin^{2k}\left(\tfrac{(\xi+\alpha_0)w_{0}}{2k}\right)}{\left(\tfrac{(\xi+\alpha_0)w_{0}}{2k}\right)^{2k}} + \frac{\sin^{2k}\left(\tfrac{(\xi-\alpha_0)w_{0}}{2k}\right)}{\left(\tfrac{(\xi-\alpha_0)w_{0}}{2k}\right)^{2k}}\right]\,.
\end{align*}
We observe that this vanishes only when both $\sin^{2k}\left(\tfrac{(\xi+\alpha_0)w_{0}}{2k}\right)$ and $\sin^{2k}\left(\tfrac{(\xi-\alpha_0)w_{0}}{2k}\right)$ vanish. This can happen only if $\alpha_0 = \frac{l\pi k}{w_{0}}$ for some $l \in \mathbb{Z}$. Since by assumption we have $0<\alpha_0 < \frac{\pi k}{2w_{0}}$, this condition cannot hold, which implies the result.
\end{proof}

\begin{lemma}\label{lem:filter_bounds}
Let $\alpha_0$ and $w_{0}$ be fixed. Then, there exist constants $C_0,C_1 > 0$ depending only on $\alpha_0$ and $w_{0}$ and $C_2$ depending only on $\alpha_0,w_0$ and $k$ such that for every $\xi \in \mathbb{R}$,
\begin{equation}
\label{eq:main_filter_lower_bound}
\frac{C_0}{C_1 +\max((\tfrac{\xi}{\alpha_0}-1)^{2k},(\tfrac{\xi}{\alpha_0}+1)^{2k}) }\leq \Lambda^{\alpha_0}_{k,w_{0}}(\xi)   \leq  \frac{C_2}{1 + \xi^{2k}}\,.
\end{equation}
For every $i \in \mathbb{N}$, denoting the $i$ times differentiation operator by $D^{(i)}$, 
$$\biggr|D^{(i)}\Big[\frac{1}{\Lambda^{\alpha_0}_{k,w_{0}}(\xi)}\Big]\biggr| \leq C(i,k,w_{0},\alpha_0)\left(1+\bigr|\xi\bigr|^{2k}\right) \,.$$
For every $\xi \in \RR$ and $i \in \mathbb{N}$ there is a constant $C_1(i,k,w_0,\alpha_0)$ such that
$$\bigr|D^{(i)}\Lambda^{\alpha_0}_{k,w_{0}}(\xi)\bigr| \leq \frac{C_1(i,k,w_0,\alpha_0)}{1 + \xi^{2k}}\,.$$
\end{lemma}

\begin{proof}
Let $\theta \leq \frac{\pi}{4}$. Define $\eta(x) := \tfrac{\sin^{2k}(x+\theta)}{(x+\theta)^{2k}} + \tfrac{\sin^{2k}(x-\theta)}{(x-\theta)^{2k}}$. We will use the following claim.
\begin{claim}\label{claim:sine_lower_bound}
 Let $\theta \in [0,\frac{\pi}{4}] $.  Then for every $x \in \mathbb{R}$, either $\sin^{2k}(x+\theta) \geq \sin^{2k}(\theta)$ or $\sin^{2k}(x-\theta) \geq \sin^{2k}(\theta)$. 
 \end{claim}
 \myproof{Proof of claim:} It is sufficient to show this for $x \in [0,\pi)$ because of periodicity. If $x \leq \pi - 2\theta$ then, $\theta \leq x+\theta \leq \pi - \theta$. Therefore, $ \sin^{2k}(x+\theta) \geq \sin^{2k}(\theta)$. If $x > \pi - 2\theta$ then $\pi -\theta > x -\theta > \pi - 3\theta \geq \theta$. Therefore, $\sin^{2k}(x-\theta) \geq \sin^{2k}(\theta)$. 
 \qedhere\\
Clearly, 
\begin{align}
    \eta(x) &\geq \tfrac{\sin^{2k}(x+\theta)}{\sin^{2k}(x+\theta)+(x+\theta)^{2k}} + \tfrac{\sin^{2k}(x-\theta)}{\sin^{2k}(x-\theta)+(x-\theta)^{2k}}\nonumber\\
    &\geq \min\left(\frac{\sin^{2k}(\theta)}{\sin^{2k}(\theta) + (x-\theta)^{2k}},\frac{\sin^{2k}(\theta)}{\sin^{2k}(\theta) + (x+\theta)^{2k}}\right) \nonumber\\
    &= \frac{\sin^{2k}(\theta)}{\sin^{2k}(\theta) + \max((x-\theta)^{2k},(x+\theta)^{2k})}\,.
\label{eq:filter_lower_bound}
\end{align}
In the second step we have used Claim~\ref{claim:sine_lower_bound}. We note that when $\theta = \frac{\alpha_0w_{0}}{2k}$, $\Lambda^{\alpha_0}_{k,w_{0}}(\xi) = \tfrac{c_0}{2}\eta(\frac{\xi w_{0}}{2k})$ where $c_0 = \frac{1}{\int_{-\infty}^{\infty}\cos(\alpha_0 T)\lambda_{k,w_{0}}(T)dT} \geq 1$. From equation~\eqref{eq:filter_lower_bound}, we conclude that
\begin{align}
    \Lambda^{\alpha_0}_{k,w_{0}}(\xi) &\geq \frac{c_0}{2}\frac{\sin^{2k}(\tfrac{\alpha_0w_{0}}{2k})/(\tfrac{\alpha_0w_{0}}{2k})^{2k}}{\sin^{2k}(\tfrac{\alpha_0w_{0}}{2k})/(\tfrac{\alpha_0w_{0}}{2k})^{2k} + \max((\tfrac{\xi}{\alpha_0}-1)^{2k},(\tfrac{\xi}{\alpha_0}+1)^{2k})} \nonumber \\
    &\geq \frac{1}{2}\frac{\sin^{2k}(\tfrac{\alpha_0w_{0}}{2k})/(\tfrac{\alpha_0w_{0}}{2k})^{2k}}{\sin^{2k}(\tfrac{\alpha_0w_{0}}{2k})/(\tfrac{\alpha_0w_{0}}{2k})^{2k} + \max((\tfrac{\xi}{\alpha_0}-1)^{2k},(\tfrac{\xi}{\alpha_0}+1)^{2k})}\,.\label{eq:filter_lower_bound_2}
\end{align}
In the second step we have used the fact that $c_0 \geq 1$. Now, using Taylor's theorem, we conclude that when $0\leq x \leq \tfrac{\pi}{2}$, $\frac{\sin x}{x} \geq 1 - \tfrac{x^2}{6}$. Therefore,
$$\lim_{k\to \infty}\sin^{2k}(\tfrac{\alpha_0w_{0}}{2k})/(\tfrac{\alpha_0w_{0}}{2k})^{2k} = 1\,.$$ 
Using this, we conclude that we can bound $\sin^{2k}(\tfrac{\alpha_0w_{0}}{2k})/(\tfrac{\alpha_0w_{0}}{2k})^{2k}$ away from $0$, uniformly for all $k$.
 Using this in the Equation~\eqref{eq:filter_lower_bound_2}, we conclude the first part of the lemma. Now, we will consider the derivatives. We first show the following claim:
\begin{claim} \label{claim:reciprocal_derivative}
Let $f \in C^{\infty}(\mathbb{R})$ such that $f(x) \neq 0$ for every $x \in \mathbb{R}$. Then, for any $i\geq 1$ $D^{(i)}(\frac{1}{f})$ is a linear combination of the functions of the form $\frac{1}{f^{r+1}}\prod_{l=1}^{r}D^{(n_l)}(f)$, where $1\leq r \leq i$, $n_l \in \mathbb{N}$, and $\sum_{l=1}^{r}n_l = i$.  The coefficients in the linear combination do not depend on $f$.
\end{claim}  
\myproof{Proof of claim:} We show this using induction with base case $D^{(1)}\frac{1}{f} = -\frac{1}{f^2} D^{(1)}f$, which satisfies the hypothesis. Suppose the hypothesis is true for $D^{(i)}\frac{1}{f}$. Then $D^{(i+1)}\frac{1}{f}$ is a linear combination of functions of the form $D^{(1)}\left( \frac{1}{f^{r+1}}\prod_{l=1}^{r}D^{(n_l)}(f)\right)$, where $1\leq r \leq i$, $r_l \in \mathbb{N}$, and $\sum_{l=1}^{r}n_l = i$. Now, 
\begin{align*}
D^{(1)}\left( \frac{1}{f^{r+1}}\prod_{l=1}^{r}D^{(n_l)}(f)\right) &= -\frac{r+1}{f^{r+2}}D^{(1)}(f)\prod_{l=1}^{r}D^{(n_l)}(f)  \\&\quad+ \frac{1}{f^{r+1}}\sum_{l_0=1}^{r} D^{(n_{l_0}+1)}(f)\prod_{l \neq l_0}D^{(n_l)}(f)\,.
\end{align*}
This is a linear combination with the required property for $i+1$. Therefore, we conclude the claim.
\qedhere\\

We now show another estimate necessary for the proof:
\begin{claim}\label{claim:derivative_bound}
For every $i \in \mathbb{N}$ and some constant $C(i,k,w_{0},\alpha_0) > 0$ depending only on $i,k,w_0,\alpha_0$,
$$|D^{(i)}\Lambda_{k,w_{0}}^{\alpha_0}(\xi)| \leq\frac{ C(i,k,w_{0},\alpha_0)}{(1+|\xi|^{2k})}$$
\end{claim}
\myproof{Proof of claim:} 
Let $g(\xi) = \frac{\sin^{2k}(\xi)}{\xi^{2k}}$.
 Since $\Lambda_{k,w_{0}}^{\alpha_0}$ is a linear combination of the scaled and shifted version of $g$, the same bounds hold for $\Lambda_{k,w_{0}}^{\alpha_0}$ up to constants depending on $k,w_{0},\alpha_0$ and $i$. Clearly, $g \in C^{\infty}(\mathbb{R})$. Therefore, $|D^{(i)}(g)(\xi)| \leq C(i)$ whenever $|\xi| \leq 1$. Now assume that $|\xi| \geq 1$. It is easy to show that $D^{(i)}(g)$ is a linear combination of the functions of the form $\frac{g_r(\xi)}{\xi^{2k+r}}$, where $g_r(\xi)$ is a bounded trigonometric function, and $r \in \{0,1,\dots, i\}$. Therefore, $|D^{(i)}(g)(\xi)| \leq \frac{C^{\prime}(i)}{|\xi|^{2k}} \leq \frac{2C^{\prime}(i)}{1+|\xi|^{2k}}$ whenever $|\xi| \geq 1$. Combining this with the case $|\xi| \leq 1$, we conclude the result.
\qedhere\\

From Claim~\ref{claim:reciprocal_derivative}, it is sufficient to upper bound terms of the form $|\frac{1}{f^{r+1}}\prod_{l=1}^{r}D^{(n_l)}(f)|$, where $1\leq r \leq i$, $n_l \in \mathbb{N}$, and $\sum_{l=1}^{r}n_l = i$ for  $f = \Lambda_{k,w_{0}}^{\alpha_0}$. From the bound in Equation~\eqref{eq:main_filter_lower_bound} on $\Lambda_{k,w_{0}}^{\alpha_0}$ and bounds on the derivatives in Claim~\ref{claim:derivative_bound}, we have $$\biggr|\frac{1}{f^{r+1}}\prod_{l=1}^{r}D^{(n_l)}(f)\biggr|(\xi) \leq C(k,i,w_{0},\alpha_0)(1+|\xi|^{2k})\,.$$
From this we conclude the upper bound on the derivatives. The proof of upper bound on $\Lambda_{k,w_{0}}^{\alpha_0}$ is similar to the proof of Claim~\ref{claim:derivative_bound} and the bounds on $D^{(i)}\Lambda_{k,w_{0}}^{\alpha_0}$ follows from Claim~\ref{claim:derivative_bound}. This completes the proof of Lemma~\ref{lem:filter_bounds}. 
\end{proof}

Let $C_c^{\infty}(\mathbb{R})$ denote the set of infinitely differentiable, compactly supported real valued functions. Let $p$ be any symmetric continuous probability density supported over $[-w_{0},w_{0}]$. Define \begin{equation}
\sR(t) = \int_{-\infty}^{\infty}\R(t-T)p(T)dT \,.
\label{eq:smooth_relu_general_def}
\end{equation}
We also define the convolution operator $\cP:C^{0}(\mathbb{R}) \to C^{0}(\mathbb{R}) $ by
$$\cP g (t) := \int_{-\infty}^{\infty}g(t-T)p(T)dT\,,$$
and let $\id$ denote the identity operator over $C^{0}(\mathbb{R})$.
\begin{lemma} \label{lem:relu_representation}
Let $h \in C^{\infty}_c(\mathbb{R})$ function such that $\supp(h) \subseteq [a,b]$ for some $a,b \in \mathbb{R}$. Then
\begin{enumerate}
\item For any $t \in [a,b]$,
$$h(t) = \int_{-\infty}^{\infty}h^{\dprime }(T)\R(t-T)dT\,.$$
\item Let $\sR$ be as defined in Equation~\eqref{eq:smooth_relu_general_def}. For every $n \in \mathbb{N}$,
\begin{align*}
h(t) &= \int_{-\infty}^{\infty}h^{\dprime }(T)\left[(\id-\cP)^{n+1}\right]\R(t-T)dT \\&\quad+ \sum_{i=0}^{n}\int_{-\infty}^{\infty}h^{\dprime }(T)\left[(\id-\cP)^i\sR\right](t-T)dT\,.   
\end{align*}
\end{enumerate} 
\end{lemma}

\begin{proof}
\paragraph{1.} Since $h$ is infinitely differentiable and supported over $[a,b]$, $\supp (h^{\dprime}) \subseteq [a,b]$. Therefore, the integral in question reduces to:
$$\int_{a}^{b}h^{\dprime }(T)\R(t-T)dT = \int_{a}^{t}h^{\dprime}(T)(t-T)dT\,.$$
The proof follows from integration by parts and using the fact that $h^{\prime}(a) = h(a) = 0$.
\paragraph{2.} Since $h^{\dprime}$ is compactly supported, it is sufficient to show that 
$$\sum_{i=0}^{n}\left[(\id-\cP)^i\sR\right]+\left[(\id-\cP)^{n+1}\right]\R = \R \,.$$
Since $\sR = \cP\left(\R\right)$, this reduces to showing that
$$\sum_{i=0}^{n}\left[(\id-\cP)^i\right] \cP + (\id-\cP)^{n+1} = \id\,,$$
which can be verified via a straightforward induction argument.
\end{proof}

\begin{lemma} \label{lem:infinite_series_convergence}Let $h$ be as defined in Lemma~\ref{lem:relu_representation}. Let $P$, the Fourier transform of density $p$ be such that $P(\xi) \in \mathbb{R}$ for every $\xi$ and $1 \geq P(\xi) > 0$ for almost all $\xi$ (w.r.t lebesgue measure over $\mathbb{R}$). Then for every $t \in [a,b]$ the following limit holds uniformly.
$$\lim_{n\to \infty}\int_{-\infty}^{\infty}h^{\dprime }(T)\left[(\id-\cP)^{n+1}\R\right](t-T)dT = 0 \,.$$
And for every $t \in [a,b]$ the following holds uniformly:
$$ h(t) = \lim_{n\to \infty} \sum_{i=0}^{n}\int_{-\infty}^{\infty}\left[(\id-\cP)^ih^{\dprime}\right](T)\sR(t-T)dT \,. $$
\end{lemma}

\begin{proof}
Fix $t \in [a,b]$. By a simple application of Fubini's theorem, the fact that $h^{\dprime}$ has compact support and that $p(\cdot)$ is compactly supported, it is easy to show the following ``self-adjointness'' of the operator $\cP$. For any continuous $f : \mathbb{R} \to \mathbb{R}$:
\begin{equation}\label{eq:self_adjointness}
    \int_{-\infty}^{\infty} h^{\dprime}(T)\left[\cP f\right](t-T)dT = \int_{-\infty}^{\infty} \left[\cP h^{\dprime}\right](T) f(t-T)dT\,.
\end{equation}
From Equation~\eqref{eq:self_adjointness} it follows that
$$\int_{-\infty}^{\infty}h^{\dprime }(T)\left[(\id-\cP)^{n+1}\R\right](t-T)dT = \int_{-\infty}^{\infty}\left[(\id-\cP)^{n}\right]h^{\dprime }(T)\left[(\id-\cP)\R\right](t-T)dT\,.$$
 
\sloppy From the definition of the $\R$ and the fact that $p$ is symmetric and of compact support, it is clear that $\left[(\id-\cP)\right]\R$ is a continuous function with compact support. $\|\left[(\id-\cP)\right]\R\|_2  < \infty$ where $\|\cdot \|_2$ is the standard $L^2$ norm of functions w.r.t Lebesgue measure. Hence, by the Cauchy-Schwarz inequality,
\begin{align}
   & \biggr|\int_{-\infty}^{\infty}h^{\dprime }(T)\left[(\id-\cP)^{n+1}\R\right](t-T)dT
   \biggr| \nonumber \\ &= \biggr|\int_{-\infty}^{\infty}\left[(\id-\cP)^{n}h^{\dprime }\right](T)\left[(\id - \cP)\R(t-T)\right]dT\biggr|\nonumber \\
    &\leq \|(\id-\cP)\R\|_2\|(\id-\cP)^{n}h^{\dprime }\|_2 \nonumber\\
    &\leq C\|(\id-\cP)^{n}h^{\dprime }\|_2 \,, \label{eq:pre_plancherel}
\end{align}
where $C$ is independent of $n$. To prove the lemma, it is sufficient to show that $\lim_{n \to \infty}\|(\id-\cP)^{n}h^{\dprime }\|_2 = 0$. We do this using Parseval's theorem. Let $H^{(2)}$ be the Fourier transform of $h^{\dprime}$. We note that $H^{(2)} \in L^2$ since $h \in \sch$. By the duality of convolution-multiplication with respect to Fourier transform, we conclude that the Fourier transform of $(\id-\cP)^{n}h^{\dprime }$ is $(1-P)^nH^{(2)}$. By Plancherel's theorem, 
\begin{equation}
\|(\id-\cP)^{n}h^{\dprime }\|_2 = \frac{1}{\sqrt{2\pi}}\|(1-P)^nH^{(2)}\|_2\,. \end{equation}

Since $0 < P(\xi) \leq 1$ almost everywhere, we conclude that $\lim_{n \to \infty}(1-P)^n H^{(2)} = 0 $ almost everywhere. Since $|(1-P)^nH^{(2)}| \leq  |H^{(2)}|$ almost everywhere and $H^{(2)} \in L^2$, we conclude by dominated convergence theorem that
$$\lim_{n \to \infty}\|(\id-\cP)^{n}h^{\dprime }\|_2 =  \frac{1}{\sqrt{2\pi}}\lim_{n \to \infty}\|(1-P)^nH^{(2)}\| =0\,.$$

Equation~\eqref{eq:pre_plancherel} along with item 2 of Lemma~\ref{lem:relu_representation}, this implies that for every $t \in [a,b]$, the following uniform convergence holds:
$$ h(t) = \lim_{n\to \infty} \sum_{i=0}^{n}\int_{-\infty}^{\infty}h^{\dprime }(T)\left[(\id-\cP)^i\sR\right](t-T)dT \,. $$

Using Equation~\eqref{eq:self_adjointness} along with the equation above, we get
$$ h(t) = \lim_{n\to \infty} \sum_{i=0}^{n}\int_{-\infty}^{\infty}\left[(\id-\cP)^ih^{\dprime }\right](T)\sR(t-T)dT \,.$$
\end{proof}

In Lemma~\ref{lem:limiting_function} below, we will show that when we choose the operator $\cP$ carefully, the sum $h^{(2)}_n :=\sum_{i=0}^{n} (\id-\cP)^ih^{\dprime }$ converges a.e. and in $L^2$ to a Schwartz function $\bar{h} :\mathbb{R} \to \mathbb{R}$. The proof is based on standard techniques from Fourier analysis. Let $D^{(n)}$ denote the $n$-fold differentiation operator over $\mathbb{R}$ and we take $D^{(0)}$ to be the identity operator. 
\begin{lemma} \label{lem:limiting_function}
Let the filter $p$ and its Fourier transform $P$ be such that 
\begin{enumerate}
    \item They obey all the conditions in Lemma~\ref{lem:infinite_series_convergence}
    \item $\frac{1}{P} \in C^{\infty}(\mathbb{R})$
    \item $\|D^{(i)}(P)\|_{\infty} \leq C_i$ for some constant $C_i$. 
    \item For every $n \in \mathbb{N}\cup \{0\}$ there exists a constant $C_n > 0$ such that $|D^{n}\frac{1}{P(\xi)}| \leq C_n(1+ \xi^{2m(n)})$ for some $m(n) \in \mathbb{N}$
\end{enumerate}
Let $\bar{h}$ be the inverse Fourier transform of $\frac{H^{(2)}}{P}$, where $H^{(2)}$ is the Fourier transform of $h^{\dprime}$. Then:
\begin{enumerate}
\item $\bar{h}\in \sch$
\item $(1+|T|^3)h^{(2)}_n(T) \to (1+|T|^3)\bar{h}(T)$ as $n\to \infty$ uniformly for all $ T \in \mathbb{R}$
\item For every $t \in [a,b]$, $h$ admits the integral representation
$$h(t) = \int_{-\infty}^{\infty}\bar{h}(T)\sR(t-T)dT\,.$$ 
\end{enumerate}
Furthermore, the filter $p = \lambda^{\alpha_0}_{k,w_{0}}$ (defined in Equation~\eqref{eq:cosine_regularized_filter}) satisfies the above conditions.
\end{lemma}

\begin{proof}
Since $h^{\dprime}\in \sch$, we conclude that $H^{(2)} \in \sch$ because Fourier transform maps Schwartz functions to Schwartz functions. It is easy to show from definitions that $\frac{H^{(2)}}{P} \in \sch$. By definition $\bar{h} := \mathcal{F}^{-1}\left(\frac{H^{(2)}}{P}\right)$ (where $\mathcal{F}^{-1}$ denotes the inverse Fourier transform). Therefore, $\bar{h}\in \sch$. We will first show that $h^{(2)}_n(T) \to \bar{h}(T)$ uniformly for every $T \in \mathbb{R}$.  By definition of $h^{(2)}_n \in \sch$, it is clear that $ h^{(2)}_n \in C_c^{\infty}(\mathbb{R})\subset \sch$  and hence its Fourier transform $H^{(2)}_n \in \sch$. Since $H^{(2)}_n(\xi) = \sum_{i=1}^{n} \left(1-P(\xi)\right)^i H(\xi)$. Since $0 < P(\xi) \leq 1$ for every $\xi \in \mathbb{R}$ by hypothesis, we conclude that $H^{(2)}_n(\xi) \to \frac{H}{P}(\xi)$  and $|H_n^{(2)}(\xi)| \leq \biggr|\frac{H}{P}(\xi)\biggr|$ for every $\xi \in \mathbb{R}$. Therefore, $\biggr|H^{(2)}_n(\xi) - \frac{H}{P}(\xi)\biggr| \leq 2\biggr|\frac{H}{P}(\xi)\biggr| \in L^{1}(\mathbb{R})$. From the Fourier inversion formula, the following holds for every $T \in \mathbb{R}$:
\begin{align*}
|h^{(2)}_n(T)- \bar{h}(T)| &= \frac{1}{2\pi} \biggr|\int_{\mathbb{R}}e^{-i\xi T}\left(\frac{H^{(2)}}{P}(\xi) - H^{(2)}_n(\xi)\right)d\xi\biggr| \\
&\leq \frac{1}{2\pi}\int_{\mathbb{R}}\biggr|\frac{H^{(2)}}{P}(\xi) - H^{(2)}_n(\xi)\biggr|d\xi\,.
\end{align*} 
By the dominated convergence theorem, the integral in the last step converges to $0$ as $n \to \infty$ and we conclude that $h_n^{(2)}(T) \to \bar{h}(T)$ uniformly for every $T$. To show the uniform convergence of $T^{3}h_n^{(2)}(T) \to T^{3}\bar{h}(T)$, we use the duality between multiplication by a polynomial and differentiation under Fourier transform. The Fourier transform of $T^{3}h_n^{(2)}(T)$ is $iD^{(3)}H_n^{(2)}(\xi)$ and that of $T^{3}\bar{h}(T)$ is $iD^{(3)}\frac{H^{(2)}}{P}$. We proceed just like above. We need to show that $D^{(3)}H_n^{(2)}(\xi) \to D^{(3)}\frac{H^{(2)}}{P}$ for every $\xi$ and that $D^{(3)}H_n^{(2)}(\xi)$ is dominated by a $L^{1}$ function uniformly for every $n$. It is clear that $H_n^{(2)}(\xi) - \frac{H^{(2)}}{P}(\xi) = -\frac{(1-P(\xi))^{n+1}}{P(\xi)}H^{(2)}(\xi)$. Differentiating both sides thrice and applying the product rule, we conclude that $D^{(i)}H_n^{(2)}(\xi) \to D^{(i)} \frac{H^{(2)}}{P}(\xi)$ for every $\xi$ and for every $i\leq 3$. Consider $D^{(3)}\left[\frac{(1-P(\xi))^{n+1}}{P(\xi)}H^{(2)}(\xi)\right]$, we get a finite linear combination of the functions of the form 
\begin{equation}\label{eq:differentiation_terms}
n^r \frac{(1-P)^{n+1-l}}{P^{c_0}}D^{(a)}(H^{(2)})\prod_{s =1}^{3}D^{(b_s)}(P)
\end{equation} 
for some $c_0,r,l,a,b_s,k \in \mathbb{N}\cup\{0\}$, all of them independent of $n$ and such that $l,r,b_s,a\leq 3$ and $c_0 \leq 4$. To show domination above from a $L^1$ function, it is sufficient to show that each of terms of the form described in Equation~\eqref{eq:differentiation_terms}. Now, by assumption, $\|D^{(b_s)}(P)\|_{\infty}\leq C$ for some constant $C$. $\frac{1}{P^{c_0}(\xi)} \leq C(1+|\xi|^{2m(0)})^{4} $ (where $m(0)$ is as given in the conditions of the lemma and $c_0\leq 4$ as given above) and $D^{(a)}H^{(2)}  \in \sch$. It is therefore sufficient to show that $n^r (1-P)^{n+l-1}$ is dominated by a fixed polynomial in $|\xi|$ for every $n$ large enough. Indeed, for $n \geq 3$, we have
\begin{align*}
n^r(1-P(\xi))^{n-l+1} &\leq n^r(1-P(\xi))^{n-2}\\
&\leq n^r e^{-P(\xi)(n-2)}\\
&\leq e^2 n^re^{-P(\xi)n}\\
&\leq e^2 \sup_{x \geq 0} x^r e^{-P(\xi)x}\\
&= \frac{e^2 r^r e^{-2}}{P(\xi)^r}\\
&\leq C(1+|\xi|^{2m(0)})^3\,.
\end{align*}  
Here we have used the fact that $r\leq 3$. Therefore, the remainder term for each $n$ is uniformly dominated by a product of a polynomial of $\xi$ and a Schwartz function. Therefore, we conclude that the sequence $H_n^{(2)}$ is dominated by a $L^1$ function and from the discussion above conclude that $(1+|T|^3)h_n^{(2)}(T) \to (1+|T|^3)h(T)$ uniformly for every $T \in \mathbb{R}$. To show the final result, we apply Lemma~\ref{lem:infinite_series_convergence} for $t \in [a,b]$ to obtain
$$h(t) = \int_{-\infty}^{\infty}h_n^{(2)}(T)\sR(t-T)dT + o_n(1)\,,$$
where $o_n(1)$ tends to $0$ uniformly for all $t \in [a,b]$. Plugging in this expression for $h(t)$ yields
\begin{align*}
\biggr|h(t) - \int_{-\infty}^{\infty}\bar{h}(T)\sR(t-T)dT\biggr| = \biggr|\int_{-\infty}^{\infty}(h_n^{(2)}(T)-\bar{h}(T))\sR(t-T)dT\biggr| + o_n(1) 
\end{align*}
which we upper bound by
\begin{align*}
&\leq \int_{-\infty}^{\infty}\bigr|h_n^{(2)}(T)-\bar{h}(T)\bigr|\sR(t-T)dT + o_n(1) \\
&=  \int_{-\infty}^{\infty}(1+|T|^3)\bigr|h_n^{(2)}(T)-\bar{h}(T)\bigr|\frac{\sR(t-T)}{1+|T|^3}dT + o_n(1) \\
&\leq \|(1+|\eta|^3)\bigr|h_n^{(2)}(\eta)-\bar{h}(\eta)\bigr|\|_{\infty}\int_{-\infty}^{\infty}\frac{\sR(t-T)}{1+|T|^3}dT +o_n(1) 
\end{align*}
Now using the fact that $|\sR(s)| = \int_{-w_0}^{w_0}\R(s-\tau)p(\tau)d\tau \leq |s| + w_0$ for every $s\in \mathbb{R}$, the above is bounded as
\begin{align*}
&\leq \|(1+|\eta|^3)\bigr|h_n^{(2)}(\eta)-\bar{h}(\eta)\bigr|\|_{\infty}\int_{-\infty}^{\infty}\frac{b + |T| + w_{0}}{1+|T|^3}dT +o_n(1)\\
&= \|(1+|\eta|^3)\bigr|h_n^{(2)}(\eta)-\bar{h}(\eta)\bigr|\|_{\infty}C +o_n(1)\\
&\to 0\,.
\end{align*}
 It is simple to verify that $\lambda_{k,w_{0}}^{\alpha_0}$ satisfies all the conditions of the lemma using the results from Lemma~\ref{lem:filter_bounds}.
\end{proof}

We will now specialize to the filter defined in Section~\ref{sec:smoothing_filter} and set $p := \lambda^{\alpha_0}_{k,w_{0}}$ as defined in Equation~\eqref{eq:cosine_regularized_filter} for some $k \in \mathbb{N}\cup \{0\}$. We denote the activation function obtained as $\sR_k$, in keeping with the notation defined in Section~\ref{sec:smoothing_filter}. A well known result from analysis shows the existence of a ``bump function'' $\gamma \in C^{\infty}_c(\mathbb{R}) \subset \sch$ such that $\gamma(t) = 1$ when $|t| \leq 1$, $\gamma(t) = 0$ when $|t| \geq 2$ and $\gamma(t) \geq 0$ for every $t \in \mathbb{R}$. Let $\Gamma$ be the Fourier transform of $\gamma$.
Henceforth, we let $h(t) = \gamma(t)\cos(\alpha t + \psi)$ for some $\alpha,\psi \in \mathbb{R}$. Clearly $h \in C_c^{\infty}(\mathbb{R})$. It is clear that for $t \in [-1,1]$, $h(t) = \cos(\alpha t + \psi) $. Therefore, from Lemma~\ref{lem:limiting_function}, we conclude that there exists $\bar{h} \in \sch$ such that for every $t \in [-1,1]$,
\begin{equation}\label{eq:cosine_rep_first}
\cos(\alpha t + \psi) = \int_{\mathbb{R}}\bar{h}(T)\sR_k(t-T)dT \,.
\end{equation}

In the following discussion, we will estimate about how `large' $\bar{h}$ is in terms of $\alpha$. Let $H$ denote the Fourier transform of $h$.
A simple calculation shows that:
\begin{enumerate}
    \item \begin{equation}
H(\xi) = \frac{1}{2}\left[e^{i\psi}\Gamma(\xi +\alpha)+e^{-i\psi}\Gamma(\xi-\alpha)\right]\label{eq:fourier_transform_formula}
\end{equation}
    \item \begin{equation}\label{eq:fourier_transform_formula_1}
    H^{(2)}(\xi) = -\frac{\xi^2}{2}\left[e^{i\psi}\Gamma(\xi +\alpha)+e^{-i\psi}\Gamma(\xi-\alpha)\right]\end{equation}
\end{enumerate}

\begin{lemma}\label{lem:sup_bound}
Let $h(t) = \gamma(t)\cos(\alpha t + \psi) $ and $\bar{h}$ be the corresponding limiting function given by Lemma~\ref{lem:limiting_function}. Then for all $T\in \mathbb{R}$ and $l \in \mathbb{N}$, we have $$ |(1+T^{2l})\bar{h}(T)| \leq C(k,\alpha_0,w_{0},l)(1+|\alpha|^{2k+2})\,.$$
\end{lemma}

\begin{proof}
Let $\bar{H}$ be the Fourier transform of $\bar{h}$. By the inversion formula we have that for every $T$
\begin{equation}
    |\bar{h}(T)| \leq \frac{1}{2\pi}\int_{-\infty}^{\infty}|\bar{H}(\xi)|d\xi \,.
    \label{eq:bound_using_ft}
\end{equation}
By Lemma~\ref{lem:limiting_function}, it is clear that $\bar{H}(\xi) = \frac{H^{(2)}(\xi)}{\Lambda^{\alpha_0}_{k,w_{0}}(\xi)}$. Using Lemma~\ref{lem:filter_bounds}, there exists a constant $C(k,\alpha_0,\omega_0)$ such that:
\begin{align}
&|\bar{H}(\xi)| = \biggr|\frac{H^{(2)}(\xi)}{\Lambda^{\alpha_0}_{k,w_{0}}(\xi)}\biggr|\nonumber \\ 
&\leq C(k,\alpha_0,w_{0})(1+|\xi|^{2k})\xi^2\biggr(|\Gamma(\xi-\alpha)|+|\Gamma(\xi+\alpha)|\biggr)\nonumber\\
&\leq C(k,\alpha_0,w_{0})(1+|\xi|^{2k+2})\biggr(|\Gamma(\xi-\alpha)|+ |\Gamma(\xi+\alpha)|\biggr)\nonumber\\
&\leq C(k,\alpha_0,w_{0})(1+|\xi|^{2k+2})\biggr(\frac{1}{1+|\xi-\alpha|^{2k+4}} +\frac{1}{1+|\xi+\alpha|^{2k+4}}\biggr)\,.
\label{eq:bound_on_ft}
\end{align}

\sloppy We have absorbed universal constants  and constants depending only on $k$ into $C(k,\alpha_0,w_0)$ throughout. In the second step we have used the fact that $|\xi|^2 \leq 1 + |\xi|^{2k+2}$ for every $\xi \in \mathbb{R}$ and used the expressions for $H^{(2)}(\xi)$ given in Equation~\eqref{eq:fourier_transform_formula_1}. In the last step, we have used the fact that since $\Gamma\in \sch$, there exists a constant $C_k$ such that $|\Gamma(\xi)| \leq \frac{C_k}{1+|\xi|^{2k+4}}$ for every $\xi \in \mathbb{R}$. Using Equations~\eqref{eq:bound_using_ft} and~\eqref{eq:bound_on_ft} along with an elementary application of Jensen's inequality to the function $x \to |x|^{2k+2}$, we have
\begin{equation}
\label{eq:zeroth_power_ub}
|\bar{h}(T)| \leq C(k,\alpha_0,w_{0})\left(1+|\alpha|^{2k+2}\right)\,.
\end{equation}

 To bound $|T^{2l}\bar{h}(T)|$, we consider the derivatives of its Fourier transform. Clearly, the Fourier transform of $T^{2l}\bar{h}(T)$ is  $(-1)^{l}D^{(2l)}\bar{H}(\xi)$. Therefore, for all $T$, we have from the inversion formula that
$$|T^{2l}\bar{h}(T)| \leq \frac{1}{2\pi}\int_{-\infty}^{\infty}|D^{(2l)}\bar{H}(\xi)|d\xi\,.$$

Now, $D^{(2l)}\bar{H}(\xi) = D^{(2l)}\left(-\frac{\xi^2}{2\Lambda^{\alpha_0}_{k,w_{0}}(\xi)}\left[e^{i\psi}\Gamma(\xi +\alpha)+e^{-i\psi}\Gamma(\xi-\alpha)\right]\right) $. Using the product rule here results in a sum of the form
\begin{align*}
&D^{(2l)}\bar{H}(\xi) \\& = -\frac{1}{2}\sum_{\substack{a,b,c \in \mathbb{Z}^{+}\\a+b+c = 2l}}N_{a,b,c}\left(D^{(a)}\xi^2\right)\left( D^{(b)}\frac{1}{\Lambda^{\alpha_0}_{k,w_{0}}(\xi)}\right)\left[e^{i\psi}D^{(c)}\Gamma(\xi + \alpha) + e^{-i\psi}D^{(c)}\Gamma(\xi - \alpha)\right]
\end{align*}
for some positive integers $N_{a,b,c}$.
We consider each term separately. 

Using Lemma~\ref{lem:filter_bounds}, we conclude for every $a,b$ in the summation,
$$\biggr|D^{(a)}\xi^2 D^{(b)}\frac{1}{\Lambda^{\alpha_0}_{k,w_{0}}(\xi)}\biggr| \leq C(l,k,\alpha_0,w_{0})(1+|\xi|^{2k+2})\,.$$
Now, $D^{(c)}\Gamma \in \sch$ for every $c$. Therefore we can find a constant $C_k$ such that $\bigr|D^{(c)}\Gamma(\xi)\bigr| \leq \frac{C_k}{1+|\xi|^{2k+4}}$. Therefore, using similar integration as the previous case, we conclude that:
\begin{equation}\label{eq:fourth_power_ub}
    |T^{2l}\bar{h}(T)| \leq C(l,k,\alpha_0,w_{0})(1+|\alpha|^{2k+2})
\end{equation}
Combining equations~\eqref{eq:fourth_power_ub} and~\eqref{eq:zeroth_power_ub} we conclude the result.
\end{proof}
We will now give the proof of Theorem~\ref{thm:cosine_representation} and Lemma~\ref{lem:relu_cosine_representation}:

\myproof{Proof of Theorem~\ref{thm:cosine_representation}:}
From Lemma~\ref{lem:limiting_function} and Equation~\eqref{eq:cosine_rep_first} we conclude that for every $t \in [-1,1]$:
\begin{equation}
    \cos(\alpha t+\psi) = \int_{-\infty}^{\infty}\bar{h}(T)\sR_k(t-T)dT\,. 
\label{eq:lebesgue_representation}
\end{equation}
\sloppy For some $\bar{h} \in \sch$. From Lemma~\ref{lem:sup_bound} we conclude that $$\|(1+T^{2l})\bar{h}(T)\|_{\infty} \leq C(k,w_{0},\alpha_0,l)(1+|\alpha|^{2k+2})\,.$$ Taking $\kappa(T) := \frac{(1+T^{2l})}{c_{\mu}}\bar{h}(T)$ in Equation~\eqref{eq:lebesgue_representation}, we conclude the result. 
\qedhere\\

\myproof{Proof of Lemma~\ref{lem:relu_cosine_representation}:}
The proof follows from an application of Lemma~\ref{lem:relu_representation} with $h(t) = \gamma(t)\cos(\alpha t + \psi)$.
\qedhere\\

\myproof{Proof of Theorem~\ref{thm:smooth_relu_representation}:}
 From Equation~\eqref{eq:fourier_representation} and the definition of $\nu_{g,k}$, 
 $$g(x)=  \int\frac{C_g^{(0)}+r^{2k+2}C_g^{(2k+2)}}{1+r^{2k+2}\|\omega\|^{2k+2}}\cos\left(r\|\omega\|\tfrac{\langle\omega,x\rangle}{r\|\omega\|}+\psi(\omega)\right)\nu_{g,k}(d\omega)\,.$$
We follow the convention that $\tfrac{\langle\omega,x\rangle}{r\|\omega\|} = 0$ when $\omega = 0$ without loss of meaning in the equation above. When $x \in B^2_q(r)$, Cauchy-Schwarz inequality implies that $\tfrac{\langle\omega,x\rangle}{r\|\omega\|} \in [-1,1]$. In Theorem~\ref{thm:cosine_representation},  we take $\alpha =r \|\omega\|$ and $\psi = \psi(\omega)$ to conclude that there exists a continuous function $\kappa(T;r,\omega)$ such that for every $x \in B^2_q(r)$
$$g(x) =   \left(C_g^{(0)}+r^{2k+2}C_g^{(2k+2)}\right)\iint \frac{\kappa(T;r,\omega)}{1+r^{2k+2}\|\omega\|^{2k+2}}\sR_k\left(\tfrac{\langle\omega,x\rangle}{r\|\omega\|}-T\right)\mu_l(dT)\nu_{g,k}(d\omega)\,,$$ where 
$\bigr|\frac{\kappa(T;r,\omega)}{1+r^{2k+2}\|\omega\|^{2k+2}}\bigr| \leq C(k,l)$ a.s. In order to make the notation more compact we define $$\eta(T;r,\omega) := \frac{1}{C(k,l)} \frac{\kappa(T;r,\omega)\ind(T\leq 1 + w_{0})}{1+r^{2k+2}\|\omega\|^{2k+2}} $$
and  $ \beta_{g,k} := \big(C_g^{(0)}+r^{2k+2}C_g^{(2k+2)}\big)C(k,l)$ (we hide the dependence on $l$).

The theorem follows from the discussion above when, in the definition of $\eta$, the extra factor of $\ind(T \leq 1+w_{0})$ is removed. However, we note that when $x \in B^2_q(r)$, $\tfrac{\langle\omega,x\rangle}{r\|\omega\|} \leq 1$ and it follows that when $T > 1+w_{0}$, $$\sR_k\left(\tfrac{\langle\omega,x\rangle}{r\|\omega\|}-T\right) =0\,.$$
Therefore, we can include the factor of $\ind(T \leq 1+w_{0})$ without altering the equality.
\qedhere

\section{Neural Network Approximation with Function Independent Sampling}
\label{sec:function_independent_sampling}
We consider a similar setup as in Section~\ref{s:approximation}. Let $g :\mathbb{R}^q \to \mathbb{R}$ be such that $g \in L^{1}(\mathbb{R}^q)$ and its Fourier transform $G \in L^{1}(\mathbb{R}^q)\cap C(\mathbb{R}^q)$. We define the following norms for $G$:
\begin{equation}\label{eq:sup_fourier_norm}
S_g^{(l)} = \sup_{\omega\in \mathbb{R}^q} \|\omega\|^{l}(1+\|\omega\|^{q+1})\frac{|G(\omega)|}{(2\pi)^q}\,.
\end{equation}
We assume that $S_g^{(l)} < \infty$ for $l=0,1,\dots, L$ for some $L$ to be chosen later. We consider the spherically symmetric probability measure $\nu_0$ over $\mathbb{R}^q$ defined by its Randon-Nikodym derivative: $\nu_0(d\omega) = C_q\frac{d\omega}{1+\|\omega\|^{q+1}}$, where $C_q$ is the normalizing constant. 
\begin{remark}
We note that $G$ has to be a function and not a generalized function/measure (like dirac delta) for the norms $S_g^{(l)}$ to make sense. Unlike $\nu_{g,k}$, $\nu_0$ depends neither on $g$ nor on $k$. We intend to draw the weights $\omega_j \sim \nu_0$. Clearly $\omega_j \neq 0$ almost surely. We therefore skip the corner cases for $\omega_j = 0$ as considered in Section~\ref{sec:unbiased_estimators}. 
\end{remark}

 We let $\mu_l$ be as defined in Theorem~\ref{thm:cosine_representation}. We again consider Equation~\eqref{eq:fourier_representation}. Assume $S_g^{2k+2}, S_g^{0} < \infty$. Suppose $x \in B_q^{2}(r)$
\begin{align*}
&g(x)  =  \int_{\mathbb{R}^q}\cos(\langle\omega,x\rangle+\psi(\omega))\frac{|G(\omega)|}{(2\pi)^q}d\omega \\
&= \int\cos(\langle\omega,x\rangle+\psi(\omega))\frac{|G(\omega)|}{(2\pi)^q}\frac{(1+\|\omega\|^{q+1})}{C_q}\nu_0(d\omega) \\
&= \int\frac{|G(\omega)|(1+\|\omega\|^{q+1})(1+ r^{2k+2}\|\omega\|^{2k+2})}{C_q (2\pi)^q/C(k,l)}\eta(T;r,\omega)\sR_k\left(\frac{\langle\omega,x\rangle}{r\|\omega\|}-T\right)\mu_l(dT)\nu_0(d\omega) 
\end{align*}
 Here we have used Theorem~\ref{thm:cosine_representation} in the third step where $|\eta| \leq 1$ almost surely. For the sake of clarity, we will abuse notation and redefine 
 $$\eta(T;r,\omega) \leftarrow \frac{|G(\omega)|(1+\|\omega\|^{q+1})(1+r^{2k+2}\|\omega\|^{2k+2})}{(S^{(0)}_g+r^{2k+2}S^{2k+2}_g)(2\pi)^q}\eta(T;r,\omega)\,.$$ By similar considerations as in Theorem~\ref{thm:smooth_relu_representation}, we can replace $\eta(T;r,\omega)$ with $\eta(T;r,\omega)\ind(T \leq 1+w_{0})$. Clearly $|\eta| \leq 1$ almost surely even under this redefinition. We will take 
 $\beta^{S}_{g,k} :=\frac{C(k,l)}{C_q} (S_g^{0} + r^{2k+2}S_g^{2k+2}) $. We conclude that for every $x \in B_2^q(r)$
\begin{equation}
g(x)  = \beta^S_{g,k}\int\eta(T;r,\omega)\sR_k\left(\frac{\langle\omega,x\rangle}{r\|\omega\|}-T\right)\mu_l(dT)\nu_0(d\omega) \,.
\end{equation}

 For $j \in \{1,\dots,N\}$, draw $(T_j,\omega_j)$ to be i.i.d. from the distribution $\mu_l\times \nu_0$.  Let $\theta^{u}_j$ for $j \in [N]$ be i.i.d. $\unif[-1,1]$ and independent of everything else. We define 

$$\theta_j := \ind\left(\theta_j^u < \eta(T_j;r,\omega_j)\right) - \ind\left(\theta_j^u \geq \eta(T_j;r,\omega_j)\right)\,.$$

Clearly, $\theta_j \in \{-1,1\}$ almost surely and $\mathbb{E}\left[\theta_j|T_j,\omega_j\right] = \eta(T_j;r,\omega_j)$. That is, it is an unbiased estimator for $\eta(T_j;r,\omega_j)$ and independent of other $\theta_{j^{\prime}}$ for $j\neq j^{\prime}$. Define the estimator 
\begin{equation}\label{eq:unbiased_single_estimator_sup_case}
\hat{g}_j(x) := \begin{cases}
0 &\text{ when } T > 1+w_{0} \\
 \beta^{S}_{g,k}\theta_j\sR_k\left(\tfrac{\langle\omega_j,x\rangle}{r\|\omega_j\|}-T_j\right) &\text{ otherwise}\,.
 \end{cases}
\end{equation}
Recall the definition of $\sDelta_k$ in the discussion preceding Lemma~\ref{lem:post_surgery_representation}. We give a similar lemma below. The proof is the same as the proof of Lemma~\ref{lem:post_surgery_representation}.
\begin{lemma}\label{lem:post_surgery_representation_sup_case}
 For every $x\in B^2_q(r)$,
 $$g(x) =  \beta^{S}_{g,k} \iint \eta(T;r,\omega)\sDelta_k\left(\tfrac{\langle\omega,x\rangle}{r\|\omega\|},T\right)\mu_l(dT)\nu_0(d\omega)\,.$$
 \end{lemma}
Recall $\gamma \in \sch$ , $\gamma_{\omega}^{\perp}$, $R$ and $\Gamma_{q,R}$ as used in Section~\ref{sec:unbiased_estimators}. We define $g(x;R)$ and $\hat{g}_j(x;R)$ similarly. Draw $(T_j,\omega_j)$ i.i.d. from the distribution $\mu_l\times \nu_0$. 
 Let \begin{equation}\label{eq:post_surgery_estimator_sup_case}
 \hat{g}_j(x;R) := \begin{cases}
 0 &\text{ when } T_j > 1 + w_{0} \\
 \beta^{S}_{g,k}  \theta_j\sDelta_k\left(\tfrac{\langle\omega_j,x\rangle}{r\|\omega_j\|},T_j\right)\gamma_{\omega_j}^{\perp}(x) &\text{ otherwise}\,.
 \end{cases}
 \end{equation}
 Define for $x\in \mathbb{R}^q$
  $$g(x;R) = \mathbb{E}\hat{g}_j(x;R)\,.$$
The definition makes sense when $l\geq 2$ in $\mu_l$ since $\mathbb{E}|T_j| <\infty$. We note that $g(x;R)$ is implicitly dependent on $k,l,\alpha_0,w_{0}$. Let $\hat{G}_j(\xi;R)$ be the Fourier transform of $\hat{g}_j(x;R)$ and let $G(\xi;R)$ be the Fourier transform of $g(x;R)$. In the Lemma below we show that through the truncation modification above, we can construct an unbiased estimator for both $g$ such that the estimator's derivatives are unbiased estimators for the respective derivatives of $g$.  We give a result similar to Lemma~\ref{lem:unbiased_estimators} below. The discussion diverges from that in Section \ref{sec:unbiased_estimators} henceforth.  
 Let $\mathbf{b} = (b_1,\dots,b_q)$ such that $b_1,\dots, b_q \in \mathbb{N}\cup \{0\}$. By $\partial^{\mathrm{b}}$ we denote the differential operator where we differentiate partially with respect to $i$-th co-ordinate $b_i$ times. We define $|\mathbf{b}| = \sum_{i=1}^{q}b_i$.
\begin{lemma}\label{lem:unbiased_estimators_sup_case}
 Let $R \geq r$ and $l\geq 2$ (where $l$ determines the measure $\mu_l$) so that $\mathbb{E}_{T\sim \mu_l}|T|^2 < \infty$. We also assume that $\beta_{g,k} < \infty$.
 
 \begin{enumerate} 
 \item For every $x \in B^2_q(r)$,
   $$g(x;R) = g(x)$$
   $$\hat{g}_j(x;R) = \hat{g}_j(x)\,.$$
   Where $\hat{g}_j(x)$ is as defined in Equation~\eqref{eq:unbiased_single_estimator_sup_case}.

  \item For every $\xi,\omega_j \in \mathbb{R}^q$ such that $\omega_j \neq 0$ we define $\xi_{\omega} := \frac{\langle \xi,\omega\rangle}{\|\omega\|} \in \mathbb{R}$ and $\xi_{\omega}^{\perp} := \xi - \frac{\omega \langle \xi,\omega\rangle}{\|\omega\|^2}$. We have, for any fixed value of $T_j$ and $\omega_j$:
    \begin{equation}
    \hat{G}_j(\xi;R) = \begin{cases}
    0 \text{ if } T_j > 1+w_{0} \\
  
    \beta^S_{g,k} \theta_j\Gamma_{q-1,R}(\|\xi_{\omega_j}^{\perp}\|)\Lambda_{k,w_{0}}^{\alpha_0}(\xi_{\omega_j})\left[\frac{4e^{i(1+w_{0})r\xi_{\omega_j} }\sin^2\bigr(\tfrac{(1+w_{0}-T)\xi_{\omega_j} r}{2}\bigr)}{\xi_{\omega_j}^2r}\right] \text{ o/w}
    \end{cases}
    \end{equation}
   
 When $q =1$, we let $\Gamma_{q-1,R}(\cdot) = 1$ identically. We recall that $\Lambda_{k,w_0}^{\alpha_0}$ is the Fourier transform of the filter $\lambda_{k,w_0}^{\alpha_0}$.
 \item The functions $\hat{g}_j(x;R) \in C^{2k}(\mathbb{R}^q)$ a.s. and for every $\mathbf{b} \in \left(\mathbb{N}\cup\{0\}\right)^q$ such that $|\mathbf{b}| \leq 2k$, almost surely the following holds:
    $$\partial^{\mathbf{b}}\hat{g}_j(\spacedot;R) \in L^1(\mathbb{R}^q) \text{ a.s.}$$
For some constant $B_k$ and for every $x\in \mathbb{R}^q$, we have:
$$|\partial^{\mathbf{b}}\hat{g}_j(x;R)| \leq \beta_{g,k}^S B_k(1+|T_j|)\ind\left(T_j \leq -\bigr|\tfrac{\langle\omega_j,x\rangle}{r\|\omega_j\|}\bigr| +2+3w_{0}\right)\ind(\|x_{\omega_j}^{\perp}\| \leq 2R) $$

    Where $B_k$ is a constant which depends on $q,r,k$ and $R$ but not on $g$, $T_j$ or $\omega_j$. 
    
  \item $g(x;R) \in C^{2k}(\mathbb{R}^q)$. For every $\mathbf{b} \in \left(\mathbb{N}\cup\{0\}\right)^q$ such that $|\mathbf{b}| \leq 2k$. Then $\partial^{\mathbf{b}}g(\spacedot;R) \in L^1(\mathbb{R}^q)$ and for every $x\in \mathbb{R}^q$,
    \begin{equation}\label{eq:derivative_expectation}
    \partial^{\mathbf{b}}g(x;R) =  \mathbb{E} \partial^{\mathbf{b}}\hat{g}_j(x;R)\,.
    \end{equation}

 \end{enumerate}
 \end{lemma}
 
Some parts of the proof are similar to the proof of Lemma~\ref{lem:unbiased_estimators}. Items 3 and 4 use the duality between multiplication by polynomials and differentiation under Fourier transform.

 We define the remainder function similarly as in Section~\ref{sec:unbiased_estimators}. 
 \begin{equation} \label{eq:remainder_def_sup_case}
 g^{\rem}(x) := g(x;R)  - \frac{1}{N}\sum_{j=1}^{N} \hat{g}_j(x;R)\,.
 \end{equation}

 Clearly $g^{\rem}(x) = g(x) - \frac{1}{N}\sum_{j=1}^{N}\hat{g}_j(x)$ whenever $x\in B_q^2(r)$. Let $G^{\rem}$ be its Fourier transform. Lemma~\ref{lem:unbiased_estimators_sup_case} implies that $g^{\rem}(x)$ is continuous and $L^1$. Therefore, it is clear that $G^{\rem}$ is continuous. The following lemma is the sup type norm variant of Lemma~\ref{lem:fourier_contraction}.

 \begin{lemma}\label{lem:fourier_contraction_sup_case}
Let $l\geq 2 + q$ so that $\mathbb{E}_{T\sim \mu_l}T^2 < \infty$. Assume $\beta^S_{g,k}<\infty$. For $ s \in \{0\}\cup \mathbb{N}$, consider
$$S_{g^{\rem}}^{(s)} := \sup_{\xi \in \mathbb{R}^q}(1+\|\xi\|^{q+1}) \|\xi\|^s \frac{|G^{\rem}(\xi)|}{(2\pi)^q} \,.$$

Assume $2k \geq q+1$ and $s \leq 2k - q-1 $. We have:
\begin{enumerate}
\item $$\mathbb{E}S_{g^{\rem}}^{(s)} \leq \frac{C (S_g^{(0)} + S_g^{(2k+2)})}{\sqrt{N}}$$
Where $C$ is a constant depending only on $l,s,r,q$ and $k$.
\item $S_{g^{\rem}}^{(s)} \leq C (S_g^{(0)} + S_g^{2k+2}) \left(\frac{1}{N}\sum_{j=1}^{N}1+|T_j|^2\right)$ almost surely.
\end{enumerate}

\end{lemma}
 \begin{remark}
 Instead of the $s \leq 2k - q-1$ above, a more delicate proof would only require $s < 2k - (q+1)/2$. We will prove the weaker version for the sake of clarity.
 \end{remark}
\begin{proof}
We first consider the expectation bound in item 1. We begin by giving a bound on $\mathbb{E}\int_{\mathbb{R}^q}|\partial^{\mathbf{b}}g^{\rem}(x;R)|dx$ when $|b| \leq 2k$: 
\begin{align}
\mathbb{E}\int_{\mathbb{R}^q}|\partial^{\mathbf{b}}g^{\rem}(x;R)|dx &\leq \int_{\mathbb{R}^q} \sqrt{\mathbb{E}|\partial^{\mathbf{b}}g^{\rem}(x;R)|^2}dx \nonumber\\
&\leq  \frac{1}{\sqrt{N}}\int_{\mathbb{R}^q} \sqrt{\mathbb{E}|\partial^{\mathbf{b}}\hat{g}_1(x;R)|^2}dx \,.\label{eq:L1_expectation_bound_1}
\end{align}
Here we have used the fact that $\hat{g}_j(x;R)$ are i.i.d. unbiased estimators for $g(x;R)$. Using the bound in item 3 of Lemma~\ref{lem:unbiased_estimators_sup_case}, we conclude that
\begin{equation}\label{eq:variance_calculation_for_derivative_1}
\mathbb{E}|\partial^{\mathbf{b}}\hat{g}_j(x;R)|^2 \leq (\beta_{g,k}^S B_k)^2\mathbb{E}\left[(1+|T_j|^2)\ind\left(T_j \leq -\bigr|\tfrac{\langle\omega_j,x\rangle}{r\|\omega_j\|}\bigr| +2+3w_{0}\right)\ind(\|x_{\omega_j}^{\perp}\| \leq 2R) \right]\,.
\end{equation}
 It is clear from integrating tails that
$$\mathbb{E}\left[(1+|T_j|^2)\ind\left(T_j \leq  -\bigr|\tfrac{\langle\omega_j,x\rangle}{r\|\omega_j\|}\bigr| +2+3w_{0}\right)\biggr|\omega_j\right] \leq \frac{C(l)}{1+\biggr|\frac{\langle\omega_j,x\rangle}{r\|\omega_j\|}\biggr|^{(2l-3)}}\,.
$$
Using this in Equation~\eqref{eq:variance_calculation_for_derivative_1} and absorbing the constant $C(l,w_0)$ into $B_k$ gives
\begin{equation}\label{eq:variance_calculation_for_derivative_2}
\mathbb{E}|\partial^{\mathbf{b}}\hat{g}_j(x;R)|^2 \leq (\beta_{g,k}^S B_k)^2\mathbb{E}\left[\frac{\ind(\|x_{\omega_j}^{\perp}\| \leq 2R) }{1+\biggr|\frac{\langle\omega_j,x\rangle}{r\|\omega_j\|}\biggr|^{(2l-5)}}\right]\,.
\end{equation}
Let $$\tau(\omega_j,x):=\frac{\ind(\|x_{\omega_j}^{\perp}\| \leq 2R) }{1+\biggr|\frac{\langle\omega_j,x\rangle}{r\|\omega_j\|}\biggr|^{(2l-3)}}\,.$$ 
Clearly, $|\tau(\omega_j,x)| \leq 1$ almost surely for every $x$ and $\tau(\omega_j,x)$ is non-zero only when $\|x_{\omega_j}^{\perp}\| \leq 2R$. Consider the following conditions on $x$:
\begin{enumerate}
\item  $\|x\| \geq 3R$.
\item $\|x_{\omega_j}^{\perp}\| \leq 2R$
\end{enumerate}
It is clear that under these conditions, we have the following:
\begin{align*}
\frac{5\|x\|^2}{9} &= \|x\|^2(1-4/9) \leq \|x\|^2(1 - \tfrac{4R^2}{\|x\|^2})\\
&= \|x\|^2 - 4R^2 \leq \bigr|\tfrac{\langle x,\omega_j\rangle}{\|\omega_j\|}\bigr|^2\,.
\end{align*}
Therefore, for some universal constant $c>0$,
\begin{equation}\label{eq:bound_on_tau}
\tau(\omega_j,x) \leq \ind(\|x\| \leq 3R) + \frac{\ind(\|x\| > 3R)}{1+ \left(\frac{c\|x\|}{r}\right)^{2l-3}}\,.
\end{equation}
Plugging Equation~\eqref{eq:variance_calculation_for_derivative_2}  into Equation~\eqref{eq:L1_expectation_bound_1} gives \begin{equation*}
\mathbb{E}\int_{\mathbb{R}^q}|\partial^{\mathbf{b}}g^{\rem}(x;R)|dx  \leq \frac{\beta^{S}_gB_k}{\sqrt{N}}\int_{\mathbb{R}^q} \sqrt{\mathbb{E}\tau(\omega_j,x)}dx\,,
\end{equation*}
and now using Equation~\eqref{eq:bound_on_tau}, we obtain
\begin{align*}
\mathbb{E}\int_{\mathbb{R}^q}|\partial^{\mathbf{b}}g^{\rem}(x;R)|dx &\leq \frac{\beta^{S}_gB_k}{\sqrt{N}}\int_{\mathbb{R}^q} \sqrt{\ind(\|x\| \leq 3R) + \frac{\ind(\|x\| > 3R)}{1+ \left(\frac{c\|x\|}{r}\right)^{2l-5}}}dx \\
&= \frac{\beta^{S}_gB_k}{\sqrt{N}}\int_{\rho =0 }^{\infty}C_q\rho^{q-1}\sqrt{\ind(\rho\leq 3R) + \frac{\ind(\rho > 3R)}{1+ \left(\frac{c\rho}{r}\right)^{2l-3}}}d\rho \,.
\end{align*}
The integral on the right is smaller than some constant $C(q,l,R,r)$ if $l \geq q + 2$.  Absorbing this constant into $B_k$ too we have that
\begin{equation}\label{eq:L1_expectation_bound_2}
\mathbb{E}\int_{\mathbb{R}^q}|\partial^{\mathbf{b}}g^{\rem}(x;R)|dx \leq \frac{\beta^{S}_gB_k}{\sqrt{N}}\,.
\end{equation}

By item 4 of Lemma~\ref{lem:unbiased_estimators_sup_case}, $\partial^{\mathbf{b}}g^{\rem}$ is a continuous $L^1$ function for $|\mathbf{b}|\leq 2k$. We conclude by the Fourier duality of multiplication and differentiation that
\begin{equation}\label{eq:duality_fourier}
G^{\rem}(\xi)\prod_{j=1}^{q}\xi_j^{b_j} = (i)^{|b|} \int_{\mathbb{R}^q}\partial^{\mathbf{b}}g^{\rem}(x)e^{i\langle\xi,x\rangle} dx\,.
\end{equation}
Consider any integer $ k\geq u \geq 0$. Now, from Equation~\eqref{eq:duality_fourier},
$$\|\xi\|^{2u}G^{\rem}(\xi) = \sum_{\mathbf{b}: |\mathbf{b}|\leq 2u}C_{\mathbf{b}} i^{|\mathbf{b}|}\int_{\mathbb{R}^q}\partial^{\mathbf{b}}g^{\rem}(x)e^{i\langle\xi,x\rangle} dx $$
for some constants $C_{\mathbf{b}}$ depending only on $u$ and $\mathbf{b}$. Therefore, we have
\begin{align}
\mathbb{E}\sup_{\xi \in \mathbb{R}^q}\|\xi\|^{2u}|G^{\rem}(\xi)| &\leq \sum_{\mathbf{b}: |\mathbf{b}|\leq 2u}|C_{\mathbf{b}}|\mathbb{E}\int_{\mathbb{R}^q}|\partial^{\mathbf{b}}g^{\rem}(x)|dx \nonumber \\
&\leq \frac{\beta_{g,k}^{S}B_k}{\sqrt{N}} \,.\label{eq:general_contraction}
\end{align}
In the second step we have used Equation~\eqref{eq:L1_expectation_bound_2}. We have absorbed the constants $|C_{\mathbf{b}}|$ into $B_k$. It is clear that taking $B_k$ large enough, we can make it depend only on $k$ and not on $u$. Suppose $2k \geq q+1$. We let $0\leq s \leq 2k - q-1$.  For any $t\geq 0$ we have that $t^s(1+t^{q+1}) \leq 2(1 + t^{2k})$. This follows from the fact that if $c_1\geq c_0 >0$, we have $t^{c_0} \leq t^{c_1}+1$.
\begin{align*}
\mathbb{E}S_{g^{\rem}}^{(s)}  &= \mathbb{E}\sup_{\xi} \|\xi\|^s (1+\|\xi\|^{q+1})|G^{\rem}(\xi)| \\
&\leq  \mathbb{E}\sup_{\xi} 2(1 + \|\xi\|^{2k})|G^{\rem}(\xi)| \\ 
&\leq  \frac{\beta_{g,k}^{S}B_k}{\sqrt{N}}\,.
\end{align*}
Here we have absorbed more constants into $B_k$. From this we conclude the statement of the lemma in item 1. We now consider the almost sure bound in item 2. Clearly, $$|G^{\rem}(\xi)| \leq |G(\xi;R)| + \frac{1}{N}\sum_{j=1}^{N}|\hat{G}_j(\xi;R)|\,.$$

We will first bound $\sup_{\xi \in \mathbb{R}^q}|\hat{G}_j(\xi;R)|\|\xi\|^s(1+\|\xi\|^{q+1})$. Integrating the bound in item 3 of Lemma~\ref{lem:unbiased_estimators_sup_case}, we conclude that the following holds almost surely whenever $|\mathbf{b}| \leq 2k$:
$$\int|\partial^{\mathbf{b}}\hat{g}_j(x,R)|dx \leq B_k\beta_{g,k}^S(1+|T_j|^2)\,.$$
Using similar considerations as in Equation~\eqref{eq:general_contraction}, we conclude that whenever $0 \leq u \leq k$, almost surely:
$$\sup_{\xi \in \mathbb{R}^q}\|\xi\|^{2u}|\hat{G}_j(\xi;R)| \leq B_k \beta_{g,k}^S (1+|T_j|^2)\,.$$
Since $G(\xi;R) = \mathbb{E}\hat{G}_j(\xi;R)$, taking an expectation of the equation above yields that $$\sup_{\xi \in \mathbb{R}^q}\|\xi\|^{2u}|G(\xi;R)| \leq B_k \beta_{g,k}\,.$$

Combining the results above proves item 2.
\end{proof}

For $b \in \mathbb{N}\cup \{0\}$, define 
\begin{equation}
k^{S}_b := b\lceil \tfrac{q+3}{2}\rceil\,.
\end{equation}
Henceforth, we fix $R = r$ for the sake of clarity. We proceed with the corrective mechanism similar to the one in Theorem~\ref{thm:fast_rates_part_1}. Suppose for some $ a \in \mathbb{N}\cup\{0\}$ we have $S_g^{(0)} + S^{(2k^{S}_a+2)}_g < \infty $. Suppose $a =0$. Then, it is clear that there exists a $\R$ network with 1 non-linear layer and $N$ non-linear units which achieves a squared error of the order $\frac{1}{N}$. Now consider $a \geq 1$. Define $g^{\rem,0}$ to be the remainder for $g$ as defined in equation~\eqref{eq:remainder_def_sup_case} with $k = k^{S}_a$ and $N$ replaced with $N/(a+1)$. Now, by Lemma~\ref{lem:fourier_contraction_sup_case},
$$\mathbb{E}\left(S^{(0)}_{g^{\rem,0}} + S^{(2k^S_{a-1}+2)}_{g^{\rem,0}}\right) \leq B_a \frac{S_g^{0} + S^{2k^{S}_a+2}_g }{\sqrt{N}}\,.$$

We recursively obtain $g^{\rem,j}$ by replacing $g$ in Equation~\eqref{eq:remainder_def_sup_case} with $g^{\rem,j-1}$, the estimators $\hat{g}_j$ by outputs of $\sR_{k^{S}_{a-j}}$ units which estimate $g^{\rem,j-1}$ and with $N$ replaced with $N/(a+1)$. Continuing this way, we deduce that
$$\mathbb{E}\left(S^{(0)}_{g^{\rem,a-1}} + S^{(2)}_{g^{\rem,a-1}}\right) \leq B_a \frac{S_g^{0} + S^{2k^{S}_a+2}_g }{N^{a/2}}\,.$$ 

Now, $g^{\rem,a-1}$ can be estimated by a $N/(a+1)$ unit $\R$ network with squared error of the order $\frac{1}{N^{a+1}}$. We note that $g^{\rem,a-1}(x)$ is equal to $g(x)$ minus the output of smoothed $\R$s. This implies the following theorem. \begin{theorem}\label{thm:fast_rates_sup_case_1}

 There exists a random neural network with one non-linear layer and $N$ non-linear activations of type $\R$ and $\sR_{k^{S}_b}$ for $b \leq a$ such that for any probability distribution $\zeta$ on $B_q^2(r)$, we have
$$\mathbb{E}\int (g(x)-\hat{g}(x))^2\zeta(dx) \leq B_a \frac{(S_g^{0} + S^{2k^{S}_a+2}_g )^2}{N^{a+1}}\,.$$
Here the non-linear activation functions $\sR_k$ are of the form $\sR_k(\frac{\langle\omega_j,x}{r\|\omega_j\|} -T_j\rangle)$ for $j\in [N]$ such that $(\frac{\omega_j}{\|\omega_j\|},T_j)$ are drawn i.i.d. from probability measure $\unif (\mathbb{S}^{q-1}) \times \mu_l$ where $\mu_l$ is the probability measure defined in Theorem~\ref{thm:cosine_representation} with $l \geq q+2$.
\end{theorem}

Consider functions of the form defined in Equation~\eqref{eq:low_dim_function}. Let $\nu_i$ be the uniform distribution over the sphere embedded in $X_i := \mathrm{span}(B_i)$. Clearly, $X_i$ is isomorphic to $\mathbb{R}^q$. Let $N/(a+1)m$ be an integer. We can find a random neural network, according to Theorem~\ref{thm:fast_rates_sup_case_1} with $N/m$ neurons such that $\mathbb{E}\hat{f}_i(x) = f_i(x)$ and $$\mathbb{E}\int (f_i(\langle B_i,x\rangle)-\hat{f}_i(x))^2\zeta(dx) \leq B_a \left(S_{f_i}^{0} + S^{2k^{S}_a+2}_{f_i} \right)^2\frac{m^{a+1} }{N^{a+1}}\,.$$

To consider functions of the form given by Equation~\eqref{eq:low_dim_function} to obtain Theorem~\ref{thm:fast_rates_sup_case_2} we need to modify Theorem~\ref{thm:fast_rates_sup_case_1} a bit since $x\in \mathbb{R}^d$ instead of $\mathbb{R}^q$ in this case. It is clear that this can be mitigated if we choose the weights according $\omega_j$ such that $\omega_j \sim \unif(\mathbb{S}_i)$ where $\mathbb{S}_i$ is the sphere embedded in $\mathrm{span}(B_i)$. 
\begin{theorem}\label{thm:fast_rates_sup_case_2}
Let $f:\mathbb{R}^d \to \mathbb{R}$ be a function of the form given by Equation~\eqref{eq:low_dim_function}. We assume that $\left(S_{f_i}^{0} + S^{2k^{S}_a+2}_{f_i} \right)^2 =: M_i$ for some $M_i <\infty$ and define $L = \sum_{i=1}^{m}M_i$. Let the probability measure $\mu_l$ over $\mathbb{R}$ be defined by $\mu_l(dt) \propto \frac{dt}{1+t^{2l}}$ for $l \in \mathbb{N}$. Consider the following sampling procedure:
\begin{enumerate}
\item Partition $[N] \subseteq \mathbb{N}$ into $m$ disjoint sets, each with $N/(m(a+1))$ elements. 
\item  For $i \in [m]$, $b\in \{0,\dots,a\}$, $j\in [\tfrac{N}{m(a+1)}]$, we draw $\omega^{0}_{i,j,b} \sim \unif\left(\mathbb{S}^{\mathrm{span}(B_i)}\right)$ and $T_{i,j,b} \sim \mu_l$ independently for some $l \geq \max(q+3,3a+3)$.
\end{enumerate}
Let $\zeta$ be any probability distribution over $\mathbb{R}^d$ such that $\langle\zeta(dx),B_i\rangle$ is supported over $B_q^2(r)$. There exist random $\kappa_1,\dots,\kappa_N \in \mathbb{R}$, depending only on $\omega_{i,j,b},T_{i,j,b}$ such that for
\begin{equation}
\hat{f}(x) = \sum_{i=1}^{m}\sum_{b=0}^{a}\sum_{j=1}^{ \frac{N}{m(a+1)}} \kappa_{i,j,b}\sR_{k^{S}_b}\left(\tfrac{\langle \omega_{i,j,b}^{0},x\rangle}{r} - T_{i,j,b}\right)\,,
\end{equation}
where $\kappa_{i,j,b} = \kappa_{(i-1)\frac{N}{m} + \frac{bN}{m(a+1)} + j}$, we have:
\begin{enumerate}
\item \vspace{-2mm}$$\mathbb{E}\int (f-\hat{f})^2\zeta(dx)\leq B(l,q,r,a) L\frac{m^{a-1}}{N^{a+1}}\,,$$

\item Whenever $\delta \in (0,1)$ and $\epsilon > 0$ are given, then with probability at least $1-\delta$, \vspace{-2mm}
$$\int (f-\hat{f})^2\zeta(dx)\leq \epsilon \,,\vspace{-1mm}$$
whenever $N = \Omega\left((\frac{L}{\epsilon\delta})^{\frac{1}{a+1}}m^{\tfrac{a-1}{a+1}}\right)$. Here $\Omega(\spacedot)$ hides factors depending on $l,q,r$ and $a$
\item With probability at least $1-\delta$, for any $b \in \mathbb{N}$ such that $b\leq a$,\vspace{-2mm}
$$\sum_{j=1}^{N}\kappa_j^2 \leq \tfrac{C(l,q,r,a)\delta^{-1/(b+1)}}{N} \left(\tfrac{1}{m}\sum_{i=1}^{m}M_i^{(b+1)}\right)^{1/(b+1)}\,.$$
\end{enumerate}

\end{theorem}
\begin{proof}
As in the proof of Theorem~\ref{thm:fast_rates_part_2}, we dedicate $N/m$ activation functions to approximate each of the functions $f_i$ with neural network output $\hat{f}_i$ using the procedure in the proof of Theorem~\ref{thm:fast_rates_sup_case_1}. We then approximate $f(x) :=\frac{1}{m}\sum_{i=1}^{m}f_i(\langle B_i,x\rangle)$ by $\frac{1}{m}\sum_{i=1}^{m}\hat{f}_i(x)$. We choose $\kappa_j$ as described in the discussion preceding the statement of Theorem~\ref{thm:fast_rates_sup_case_1}.  

\paragraph{1.} The proof is similar to the proof of Theorem~\ref{thm:fast_rates_part_2}.
\paragraph{2.} The proof follows from a direct application of Markov's inequality on item 1.
\paragraph{3.}  Consider the random variable $K := \sum_{j=1}^{N}\kappa_j^2$. Consider  
\begin{align}
\mathbb{E}K^{b+1} &= N^{b+1} \left(\frac{1}{N}\sum_{j=1}^{N}\kappa_j^{2}\right)^{b+1} \nonumber \\
&\leq N^{b} \mathbb{E}\sum_{j=1}^{N}\kappa_j^{2b+2} \label{eq:outer_layer_markov_inequality_1}
\end{align}
We have applied Jensen's inequality in the second step. We will control $\mathbb{E}\kappa_j^{2(b+1)}$. Let $\kappa_j$ be the coefficient of the activation function approximating $f_i$. By the preceding the theorem statement, item 2 in Lemma~\ref{lem:fourier_contraction_sup_case} and the definition of $\hat{g}_j$ given Equation~\ref{eq:unbiased_single_estimator_sup_case}, which gives the $\kappa_j$ corresponding to $\omega_j,T_j$, it is clear that $|\kappa_j| \preceq \frac{B_am\sqrt{M_i}}{N^2}\sum_{s=1}^{\tfrac{N}{(a+1)m}}(1+|T^{\prime}_s|^2)$ where $T^{\prime}_s$ are chosen i.i.d. from $\mu_l$ and $\preceq$ denotes stochastic domination. Here the extra factor of $N/m$ in the denominator is due to the fact that when we construct the estimator $\hat{g}(x) :=\frac{1}{N^{\prime}} \sum_{j=1}^{N^{\prime}}\hat{g}_j(x)$ - there is a division by $N^{\prime}$. Therefore, 
\begin{align}
\mathbb{E}|\kappa_j|^{2(b+1)} &\leq\mathbb{E} \frac{B_a^{2(b+1)}M_i^{(b+1)}m^{2(b+1)}}{N^{4(b+1)}}\left(\sum_{s=1}^{\frac{N}{m(a+1)}}(1+|T_s^{\prime}|^2)\right)^{2(b+1)} \nonumber\\
&= \mathbb{E}\frac{B_a^{2(b+1)}M_i^{(b+1)}}{N^{2(b+1)}(a+1)^{2(b+1)}}\left(\frac{(a+1)m}{N}\sum_{s=1}^{\frac{N}{m(a+1)}}(1+|T_s^{\prime}|^2)\right)^{2(b+1)}\,.\nonumber 
\end{align}
Now by Jensen's inequality and the fact that $\mathbb{E}(1+|T^{\prime}_1|^2)^{2(b+1)} < \infty$ by our choice $l\geq 3a +3$, the above is
\begin{align}
&\leq \frac{B_a^{2(b+1)}M_i^{(b+1)}}{N^{2(b+1)}}\mathbb{E}(1+|T_1^{\prime}|^2)^{2(b+1)}\nonumber\\
&\leq \frac{B_a^{2(b+1)}M_i^{(b+1)}}{N^{2(b+1)}}C(l,a)\nonumber\\
&= \frac{B_a^{2(b+1)}M_i^{(b+1)}}{N^{2(b+1)}} \label{eq:moment_bound_coefficient}\,.
\end{align}
In the last step we absorbed factors not depending on $m$ or $N$ into $B_a$.  
Using Equation~\eqref{eq:moment_bound_coefficient} in Equation~\eqref{eq:outer_layer_markov_inequality_1}, we have
$$\mathbb{E}K^{b+1} \leq \sum_{i=1}^{m}\frac{B_a^{2(b+1)}M_i^{(b+1)}}{N^{b+1}m}\,,$$
where we have used that fact that there are exactly $N/m$ coefficients $\kappa_j$ which corresponding to the activation functions which approximate $f_i$ for any $i \in [m]$.
By an application of Markov's inequality, for any $t\geq 0$,
$$\mathbb{P}(K \geq t)\leq \frac{\mathbb{E}K^{b+1}}{t^{b+1}}\,.$$
Setting the RHS above to $\delta$ completes the proof.
\end{proof}

\section{Proofs of Lemmas}
\label{sec:proofs_of_lemmas}
\subsection{Proof of Lemma~\ref{lem:fourier_expectation}}
The first item follows from the definition of $F$ and the triangle inequality.
For the second item, observe that  $|F(\xi)|^2 = \sum_{j=1}^n f(x_j)^2 + \sum_{j\neq k}f(x_j)f(x_k) e^{i \langle\xi,x_j-x_k\rangle}$. Taking expectation on both sides, we obtain 
$$\mathbb{E}|F(\xi)|^2 \leq \|f\|^2_2 + \|f\|^2_1\exp{\left(-\tfrac{\sigma^2\theta^2}{2}\right)} \leq \|f\|^2_2 + \frac{\|f\|_1^2}{n^s}\,.$$
The third item follows directly from the definition of $\tilde{f}$ and the choice of $\sigma$. \qedhere
\subsection{Proof of Lemma~\ref{lem:almost_unbiased}}

\begin{align}
\tilde{f}(x_k)&=\mathbb{E}F(\xi)e^{-i\langle\xi,x_k\rangle} = \mathbb{E}|F(\xi)|e^{-i\phi(\xi) -i\langle\xi, x_k\rangle} = \mathbb{E}|F(\xi)|\cos\big(\langle\xi,x_k\rangle - \phi(\xi)\big) \nonumber\\
&= \mathbb{E}|F(\xi)|\cos\big(\langle\xi,x_k\rangle - \phi(\xi)\big)\mathbbm{1}(\mathcal{A}) \nonumber \\ &\quad+C \mathbb{E}|F(\xi)|(1+\tfrac{4s^2 \log^2 n}{\theta^2})\eta(T;\alpha,\psi)\R\Big(\theta \tfrac{\langle \xi,x_k\rangle}{2s\log n} -T\Big) \mathbbm{1}(\mathcal{A}^c) \nonumber \\
&= O(\|f\|_1\mathbb{P}(\mathcal{A})) - C \mathbb{E}|F(\xi)|(1+\tfrac{4s^2 \log^2 n}{\theta^2})\eta(T;\alpha,\psi)\R\Big(\theta \tfrac{\langle \xi,x_k\rangle}{2s\log n} -T\Big) \mathbbm{1}(\mathcal{A}) \nonumber \\ &\quad+C \mathbb{E}|F(\xi)|(1+\tfrac{4s^2 \log^2 n}{\theta^2})\eta(T;\alpha,\psi)\R\Big(\theta \tfrac{\langle \xi,x_k\rangle}{2s\log n} -T\Big)  \nonumber \\
&= O(\|f\|_1\mathbb{P}(\mathcal{A})) + O\Big(\tfrac{s \|f\|_1\log n}{\theta}\mathbb{E}  |\langle \xi,x_k\rangle| \mathbbm{1}(\mathcal{A})\Big) \nonumber \\ &\quad+C \mathbb{E}|F(\xi)|(1+\tfrac{4s^2 \log^2 n}{\theta^2})\eta(T;\alpha,\psi)\R\Big(\theta \tfrac{\langle \xi,x_k\rangle}{2s\log n} -T\Big)  \nonumber \\
&=  O(\|f\|_1\mathbb{P}(\mathcal{A})) + O\Big(\tfrac{s^{3/2} \|f\|_1\log^{3/2} n}{\theta^{2}}\sqrt{\mathbb{P}(\mathcal{A})}\Big) \nonumber \\ &\quad+C \mathbb{E}|F(\xi)|(1+\tfrac{4s^2 \log^2 n}{\theta^2})\eta(T;\alpha,\psi)\R\Big(\theta \tfrac{\langle \xi,x_k\rangle}{2s\log n} -T\Big)  \nonumber \\
&=  O\Big(\tfrac{s^{3/2} \|f\|_1\log^{3/2} n}{\theta^{2}n^{s/2}}\Big) +C \mathbb{E}|F(\xi)|\bigr(1+\tfrac{4s^2 \log^2 n}{\theta^2}\bigr)\eta(T;\alpha,\psi)\R\Big(\theta \tfrac{\langle \xi,x_k\rangle}{2s\log n} -T\Big)\,. \nonumber
\end{align}
Steps three through five are justified by Item 1 of Lemma~\ref{lem:fourier_expectation} to bound $|F(\xi)|$, the fact that $\R(x) \leq |x|$ and Item 1 of Lemma~\ref{lem:fourier_expectation}, and an application of the Cauchy-Schwarz inequality to show that $\mathbb{E}  |\langle \xi,x_k\rangle| \mathbbm{1}(\mathcal{A}) \leq \sqrt{\mathbb{P}(\mathcal{A})}\sqrt{\mathbb{E}  |\langle \xi,x_k\rangle|^2} \leq\sigma \sqrt{\mathbb{P}(\mathcal{A})} $. \qedhere

\subsection{Proof of Lemma~\ref{lem:one_layer_contraction}}

We begin with a chain of inequalities, justified right afterward: \begin{align}
&\mathbb{E}(f(x_k)-\hat{f}(x_k))^2 = \frac{\mathbb{E}(\hat{f}_1(x_k))^2 -\left(\mathbb{E}\hat{f}_1(x_k)\right)^2 }{N_0} +  (\tilde{f}(x_k) - \mathbb{E}\hat{f}_1(x_k))^2+(f(x_k)-\tilde{f}(x_k))^2 \nonumber \\
&\leq  \frac{\mathbb{E}(\hat{f}_1(x_k))^2}{N_0} + (f(x_k)-\tilde{f}(x_k))^2+ (\tilde{f}(x_k) - \mathbb{E}\hat{f}_1(x_k))^2\nonumber\\
&\leq \frac{\mathbb{E}(\hat{f}_1(x_k))^2}{N_0} + \frac{\|f\|^2_1}{n^{2s}} + (\tilde{f}(x_k) - \mathbb{E}\hat{f}_1(x_k))^2 \nonumber\\
&\leq \frac{C \mathbb{E} s^4\log^4 n|F(\xi_1)|^2}{N_0\theta^4}\left(1+\theta^2 \tfrac{|\langle \xi_1,x_k\rangle|^2}{s^2\log^2 n} \right) + \frac{\|f\|^2_1}{n^{2s}} + (\tilde{f}(x_k) - \mathbb{E}\hat{f}_1(x_k))^2 \nonumber \\
&\leq \frac{C \mathbb{E} s^4\log^4 n|F(\xi_1)|^2}{N_0\theta^4}\left(1+\theta^2 \tfrac{|\langle \xi_1,x_k\rangle|^2}{s^2\log^2 n} \right) + \frac{\|f\|^2_1}{n^{2s}} + C\frac{s^3\|f\|_1^2 \log^3 n}{\theta^4 n^s} \nonumber \\
&= \frac{C  s^4\log^4 n}{N_0\theta^4}\left(\mathbb{E}|F(\xi_1)|^2+\theta^2\mathbb{E}|F(\xi_1)|^2 \tfrac{|\langle \xi_1,x_k\rangle|^2}{s^2\log^2 n} \right) + \frac{\|f\|^2_1}{n^{2s}} + C\frac{s^3\|f\|_1^2 \log^3 n}{\theta^4 n^s} \nonumber \\
&\leq \frac{C  s^4\log^4 n}{N_0\theta^4}\left(\|f\|_2^2 + \frac{\|f\|^2_1}{n^s}+\theta^2\mathbb{E}|F(\xi_1)|^2 \tfrac{|\langle \xi_1,x_k\rangle|^2}{s^2\log^2 n} \right) + \frac{\|f\|^2_1}{n^{2s}} + C\frac{s^3\|f\|_1^2 \log^3 n}{\theta^4 n^s} \,.
\label{eq:function_bound_fourier}
\end{align}
The first step is the bias-variance decomposition of the squared error. In the third step we have used item 3 of Lemma~\ref{lem:fourier_expectation}. In the fourth step we have used the fact that $\R\left(\theta \tfrac{\langle \xi,x_k\rangle}{2s\log n}-T\right) \leq 1 + \theta \tfrac{|\langle \xi,x_k\rangle|}{2s\log n}$ almost surely and have absorbed this into the constant $C$. In the fifth step we have used Lemma~\ref{lem:almost_unbiased}.

We will now bound $\mathbb{E}|\langle \xi_1,x_k\rangle|^2|F(\xi_1)|^2$ to obtain the stated result. By Gaussian concentration, we have for some universal constant $c >0$ and every $t \geq 0$ that
$$\mathbb{P}\left(|\langle\xi_1,x_k\rangle| \geq  \sigma t\right) \leq 2e^{-ct^2}\,.$$
Consider the event $A_t = \{|\langle\xi_1,x_k\rangle| \leq  \sigma t\}$ for some $t >0$. Decomposing based on $A_t$ gives
\begin{align}
\mathbb{E}|\langle\xi_1,x_k\rangle|^2|F(\xi_1)|^2 &= \mathbb{E}|\langle\xi_1,x_k\rangle|^2|F(\xi_1)|^2\ind(A_t) + \mathbb{E}|\langle\xi_1,x_k\rangle|^2|F(\xi_1)|^2\ind(A_t^c) \nonumber \\
&\leq \mathbb{E}\sigma^2t^2|F(\xi_1)|^2\ind(A_t)+  \mathbb{E}|\langle\xi_1,x_k\rangle|^2|F(\xi_1)|^2\ind(A_t^c)\nonumber\\
&\leq \mathbb{E}\sigma^2t^2|F(\xi_1)|^2+  \mathbb{E}|\langle\xi_1,x_k\rangle|^2|F(\xi_1)|^2\ind(A_t^c)\nonumber\\
&\leq \sigma^2t^2\mathbb{E}|F(\xi_1)|^2 + \|f\|_1^2\mathbb{E}|\langle\xi_1,x_k\rangle|^2\ind(A_t^c)\nonumber\\
&\leq \sigma^2 t^2\mathbb{E}|F(\xi_1)|^2 + \|f\|_1^2\sqrt{\mathbb{E}|\langle\xi_1,x_k\rangle|^4}\sqrt{\mathbb{P}(A_t^c)}\nonumber\\
&\leq \sigma^2t^2\left(\|f\|_2^2 +\frac{\|f\|_1^2}{n^s}\right) + C\|f\|_1^2\sigma^2 e^{-ct^2}  \label{eq:gaussian_concentration_bound_1}
\end{align}
In the second step we have used the fact that $|\langle\xi_1,x_k\rangle| \leq  \sigma t$ whenever $\ind(A_t)  = 1$. In the third step we have used the fact that $|\ind(A_t)|  \leq 1$.  In the fourth step we have used item 1 of Lemma~\ref{lem:fourier_expectation}. In the fifth step we have used the Cauchy-Schwarz inequality. In the sixth step we have used item 2 of Lemma~\ref{lem:fourier_expectation} to bound $\mathbb{E}|F(\xi_1)|^2$ and the fact that for Gaussian random variables $ \mathbb{E}|\langle\xi_1,x_k\rangle|^4 \leq C\sigma^4$ for some universal constant $C$.  We have also used the Gaussian concentration inequality to conclude that $\mathbb{P}(A_t^c) \leq 2e^{-ct^2}$ for some universal constant $c$ and redefined and absorbed universal constants where necessary. We take $t = \sqrt{\frac{2s\log n}{c}}$ where $c$ is the constant in the exponent of Equation~\eqref{eq:gaussian_concentration_bound_1} and $\sigma = \frac{\sqrt{2s\log n}}{\theta}$ to get
\begin{equation}\label{eq:fourier_expectation_bound}
\mathbb{E}|\langle\xi_1,x_k\rangle|^2|F(\xi_1)|^2 \leq \frac{Cs^2\log^2 n}{\theta^2}\left(\|f\|_2^2+\frac{\|f\|^2_1}{n^s}\right)\,.
\end{equation}
Using Equation~\eqref{eq:function_bound_fourier} along with Equation~\eqref{eq:fourier_expectation_bound} gives
$$\mathbb{E}(f(x_j)-\hat{f}(x_j))^2 \leq \frac{Cs^4\log^4 n}{\theta^4 N_0}\left(\|f\|_2^2+\frac{\|f\|^2_1}{n^s}\right) + \frac{\|f\|_1^2}{n^{2s}}  + C\frac{s^3\|f\|_1^2 \log^3 n}{\theta^4 n^s}\,.$$
Clearly, $\|f\|^2_1 \leq n\|f\|^2_2$. Plugging this into the equation above completes the proof. \qedhere

\subsection{Proof of Lemma~\ref{lem:unbiased_estimators}}
We first prove the following estimates before delving into the proof of Lemma~\ref{lem:unbiased_estimators}.
\begin{lemma}\label{lem:extension_integral_approximation}
The following holds almost surely:
\begin{equation}        
\int dx |\hat{g}_j(x;R)| \leq
\begin{cases}
 \beta_{g,k}r(1+w_0-T)^2 \mathrm{vol}(B_{q-1}^2(2R)) &\text{ when } \omega_j \neq 0\\
 \beta_{g,k}|1+w_0-T|\mathrm{vol}(B_{q}^2(2R))  &\text{ otherwise}\,,
\end{cases}
\end{equation}
where $B_{q-1}^{2}(2R)$ is seen as a subset of $\mathbb{R}^{q-1}$ and $\mathrm{vol}$ denotes the Lebesgue measure of the set. Whenever $T_j \leq 1+w_0$,
$$\sup_{t\in \mathbb{R}}|\sDelta_k(t;T_j)| \leq 1+w_0 - T_j\,.$$
\end{lemma}

\begin{proof}
When $T_j > 1+w_0$, the bound above holds trivially since $\hat{g}_j = 0$. Now assume $T_j \leq 1+w_0$.
We first note that  $\int_{-\infty}^{\infty}|\Delta(t/r;T_j)|dt = r(1+w_0-T_j)^2$ and $\sup_{t\in \mathbb{R}}|\Delta(t/r;T_j)| = 1+w_0 - T_j$. Since $\sDelta = \lambda_{k,w_0}^{\alpha_0}\ast \Delta$ and $\lambda_{k,w_0}^{\alpha_0}$ is a probability density function, we apply Jensen's inequality to conclude the following inequalities:
\begin{enumerate}
\item $$\int_{-\infty}^{\infty}|\sDelta_k(t/r;T_j)|dt \leq r(1+w_0-T_j)^2\,,$$ 
\item $$\sup_{t\in \mathbb{R}}|\sDelta_k(t;T_j)| \leq 1+w_0 - T_j\,.$$
\end{enumerate}
 To prove the inequality on the $L^1$ norm of $\hat{g}_j$, we first consider the case $\omega_j = 0$. We conclude the corresponding bound by noting that $\theta_j \in \{-1,1\}$ (recall $\theta_j$ from the definition of $\hat{g}_j$), $0\leq \gamma(\frac{\|x\|^2}{R^2}) \leq 1$ and $\gamma(\frac{\|x\|^2}{R^2}) = 0$ when $x \notin B_q^2(2R)$ and $\sup_{t\in \mathbb{R}}|\sDelta_k(t;T_j)| \leq 1+w_0 - T_j$. 

Now consider the case $\omega_j \neq 0$ and $T_j\leq 1+w_0$. Clearly, $\gamma_{\omega_j}^{\perp}$ is a function of only the component of $x$ perpendicular to $\omega_j$. Therefore, we decompose the Lebesgue measure $dx$ over $\mathbb{R}^q$ into the product measure $dx_{\omega_j} \times dx_{\omega_j}^{\perp}$ where $dx_{\omega_j}$ is the lebesgue measure over $\mathrm{span}(\omega_j)$ and $dx_{\omega_j}^{\perp}$ is the Lebesgue measure over the space perpendicular to $\omega_j$, which is isomorphic to $\mathbb{R}^{q-1}$. The following bound holds:
$$\|\hat{g}_j\|_1 \leq \beta_{g,k}\int|\sDelta_k(x_{\omega_j}/r;T_j)|dx_{\omega_j} \int\gamma_{\omega_j}^{\perp}(x) dx_{\omega_j}^{\perp}\,.$$
We conclude the result using the fact that $ 0\leq \gamma_{\omega_j}^{\perp}(x)\leq 1$, and it vanishes outside $B_{q-1}^2(2R)$ and the fact that $\int_{-\infty}^{\infty}|\sDelta_k(t/r;T_j)|dt \leq r(1+w_0-T_j)^2$ as shown above.
\end{proof}

\noindent
\textbf{Proof of Lemma~\ref{lem:unbiased_estimators} }
\paragraph{1.}
 Follows from Lemma~\ref{lem:post_surgery_representation} and the preceding discussion. 
\paragraph{2.} From definition, it is clear that $\hat{g}_j(\spacedot;R)$ has compact support almost surely. Therefore $\hat{g}_j(\spacedot;R) \in L^1(\mathbb{R}^q)$ almost surely. To show that $g(\spacedot;R) \in L^1(\mathbb{R}^q)$, it is sufficient to show that $\hat{g}_j(x;R)$ is integrable with respect to the measure $\mu_l\times\nu_{g,k}\times dx$ where $dx$ denotes the Lebesgue measure over $\mathbb{R}^q$. First consider the case $\omega_j \neq 0$:
\begin{align*}
&\int |\hat{g}_j(x;R)|\mu_l(dT_j)\times \nu_{g,k}(d\omega_j)\times dx = 
\int\left(\int |\hat{g}_j(x;R)| dx\right)\mu_l(dT_j)\times \nu_{g,k}(d\omega_j)\\
&\leq  \int \beta_{g,k}r(1+w_0-T_j)^2 \mathrm{vol}(B_{q-1}^2(2R))\mu_l(dT_j)\times \nu_{g,k}(d\omega_j)\\
&< \infty\,.
\end{align*} 
We have used Fubini's theorem for positive functions in the first step, Lemma~\ref{lem:extension_integral_approximation} in the second step and we have used the fact that $\mathbb{E}|T_j|^2 < \infty$ in the third step. This shows that $g(\spacedot;R) \in L^1(\mathbb{R}^q)$. The case $\omega_j = 0$ follows similarly.

\paragraph{3.}
The case $T_j > 1+w_0$ is trivial. The case $\omega_j =0$ and $T_j \leq 1+w_0$ is simple to prove from the definitions. We now consider the case $\omega_j \neq 0$, $T_j \leq 1 + w_0$  and $q > 1$. The $q = 1$ case is similar to the one below, but we set $\gamma_{\omega_j}^{\perp}(x) = 1$ all along. We first note that $\omega_j \perp x_{\omega_j}^{\perp}$ and that $\gamma_{\omega_j}^{\perp}(x)$ is a function of $x_{\omega_j}^{\perp}$ only. Therefore, we decompose the Lebesgue measure $dx$ over $\mathbb{R}^d$ into the product measure $dx_{\omega_j} \times dx_{\omega_j}^{\perp}$, where $dx_{\omega_j}$ is the Lebesgue measure over $\mathrm{span}(\omega_j)$ and $dx_{\omega_j}^{\perp}$ is the Lebesgue measure over the space perpendicular to $\omega_j$, which is isomorphic to $\mathbb{R}^{q-1}$:
\begin{align}
\hat{G}_j(\xi;R) &= \beta_{g,k} \theta_j \int \gamma_{\omega_j}^{\perp}(x)\sDelta_k\left(\tfrac{x_{\omega_j}}{r},T_j\right) e^{i\langle x,\xi\rangle}dx_{\omega_j} \times dx_{\omega_j}^{\perp}\nonumber  \\
&=  \beta_{g,k}  \theta_j \int \gamma_{\omega_j}^{\perp}(x)e^{ i\langle x_{\omega_j}^{\perp},\xi_{\omega_j}^{\perp}\rangle}\sDelta_k\left(\tfrac{x_{\omega_j}}{r},T_j\right) e^{i x_{\omega_j}\xi_{\omega_j}}dx_{\omega_j} \times dx_{\omega_j}^{\perp} \nonumber \\
&=  \beta_{g,k}  \theta_j \int \gamma_{\omega_j}^{\perp}(x)e^{ i\langle x_{\omega_j}^{\perp},\xi_{\omega_j}^{\perp}\rangle} dx_{\omega_j}^{\perp} \int \sDelta_k\left(\tfrac{x_{\omega_j}}{r},T_j\right) e^{i x_{\omega_j}\xi_{\omega_j}}dx_{\omega_j} \,.\label{eq:ft_evaluation_initial}
\end{align}
In the third step, we have used the fact that $\gamma_{\omega_j}^{\perp}$ depends only on $x_{\omega_j}^{\perp}$ and that $\sDelta\left(\tfrac{x_{\omega_j}}{r},T_j\right) $ depends only on $x_{\omega_j}$. Now, we consider $\sDelta_k$ and its Fourier transform. For ease of notation, we replace $x_{\omega_j}$ by just $t \in \mathbb{R}$. Let $1+w_{0} \geq T\in \mathbb{R}$. Consider the function  $$\Delta(t;T) := \R(t-T) - 2\R(t - 1-w_{0}) + \R(t-2-2w_{0} + T)\,.$$
 It is simple to check that the Fourier transform of $\Delta(t/r;T)$ is 
$\frac{4e^{i(1+w_{0})\xi r}}{\xi^2r}\sin^2((1+w_{0}-T)\xi r/2)$. $\sDelta(x/r,T)$ is obtained from $\Delta(x/r;T)$ by convolving it with $\lambda_{k,w_{0}}^{\alpha_0}$.  Therefore, from the convolution theorem we conclude that the Fourier transform of $\sDelta(x/r;T)$ is $\frac{4e^{i(1+w_{0})\xi r}}{\xi^2r}\sin^2((1+w_{0}-T)\xi r/2) \Lambda_{k,w_{0}}^{\alpha_0}(\xi)$.

Now, $\gamma_{\omega_j}^{\perp}(x)$ is a function of $x_{\omega_j}^{\perp}$ only. Therefore, we can see this as a function with domain $\mathbb{R}^{q-1}$. In Equation~\eqref{eq:ft_evaluation_initial}, we conclude that the first integral, involving $\gamma_{\omega_j}^{\perp}$ infact gives its Fourier transform over $\Gamma_{q-1,R}$. Using these results in Equation~\eqref{eq:ft_evaluation_initial}, we obtain
$$\hat{G}_j(\xi;R) =  \beta_{g,k} \theta_j\Gamma_{q-1,R}(\|\xi_{\omega_j}^{\perp}\|)\Lambda_{k,w_{0}}^{\alpha_0}(\xi_{\omega_j})\left[\frac{4e^{i(1+w_{0})r\xi_{\omega_j} }}{\xi_{\omega_j}^2r}\sin^2((1+w_{0}-T)\xi_{\omega_j} r/2)\right]\,.$$

\paragraph{4.}
The fact that $\hat{G}_j \in L^1(\mathbb{R}^q)$ follows from item 3. The fact that $G(\xi;R) = \mathbb{E}\hat{G}_j(\xi,R)$ 
follows from  Fubini's theorem after checking that $|\hat{g}_j|$ is integrable with respect to the product measure $\mu_l\times \nu_{g,k} \times dx$ (where $dx$ denotes the Lebesgue measure over $\mathbb{R}^q$) as shown in the proof of item 2. Similar to the proof of item 2, we will conclude that $G(\xi;R)\in L^1(\mathbb{R}^q)$ by showing that $|\hat{G}_j(\xi;R)|$ is integrable with respect to the measure $\mu_l\times \nu_{g,k} \times d\xi$. 
In the cases $T_j > 1+w_0$, $|\hat{G}_j(\spacedot;R)| = 0$. When $T_j \leq 1+w_0$ and $\omega_j = 0$, we know that $\Gamma_{q,R}(\|\xi\|) \in \mathcal{S}(\mathbb{R}^q)$ and therefore an $L^1$ function. Using the fact that $|\sDelta(0;T_j)| \leq 1+w_0-T_j$, we conclude that in this case:
$\int_{\mathbb{R}^q} |\hat{G}_j(\xi;R)| d\xi \leq \beta_{g,k}\|\Gamma_{q,R}\|_1(1+w_0-T_j)$. Now consider the case $T_j \leq 1+w_0 $ and $\omega_j \neq 0$. 
We first note an inequality which follows from elementary considerations for every $a > 0$ and $\xi \in \mathbb{R}$:
\begin{equation}\label{eq:sinc_inequality}
\frac{\sin^2(a\xi)}{\xi^2} \leq \min\left(a^2,\frac{1}{\xi^2}\right)\,.
\end{equation}
By Lemma~\ref{lem:filter_bounds}, 
\begin{equation} \label{eq:filter_upper_bound_application}
|\Lambda_{k,w_{0}}^{\alpha_0}(\xi_{\omega_j})| \leq \frac{C(k,\omega_0,\alpha_0)}{1+\xi_{\omega_j}^{2k}}\,.
\end{equation}
Using Equations~\eqref{eq:sinc_inequality} and~\eqref{eq:filter_upper_bound_application}, along with the expression for $\hat{G}_j(\spacedot;R)$ in item 3, we have
$$|\hat{G}_j(\xi;R)| \leq \beta_{g,k} \Gamma_{q-1,R}(\|\xi_{\omega_j}^{\perp}\|)\frac{C(k,\omega_0,\alpha_0)}{1+\xi_{\omega_j}^{2k}} \min\left(r(1+w_0-T_j)^2,\frac{1}{r\xi_{\omega_j}^2}\right)\,.$$
Integrating this over $\mathbb{R}^q$, we get
$$\|\hat{G}_j(\spacedot;R)\|_1 \leq C(k,\omega_0,\alpha_0,r)\beta_{g,k} \|\Gamma_{q-1,R}\|_1  |1+w_0-T_j|\,.$$
Here we have abused notation to denote by $\|\Gamma_{q-1,R}\|_1 $ the $L^1$ norm of $\Gamma_{q-1,R}$ when seen as a function over $\mathbb{R}^{q-1}$.

Combining the various cases, we conclude that $\hat{G}_j(\spacedot;R)$ is integrable with respect to the measure $\mu_l\times \nu_{g,k} \times dx$ if $\mathbb{E}|1+w_0-T_j|<\infty$.  This is true since we have chosen $l\geq 2$ in the statement of the lemma.\qedhere


\subsection{Proof of Lemma~\ref{lem:fourier_contraction}}

We first state the following useful lemma before delving into the proof of Lemma~\ref{lem:fourier_contraction}. 

\begin{lemma}\label{lem:integral_over_sphere_bound}
Let $Z$ be uniformly distributed on the sphere $\mathbb{S}^{q-1}$ for $q\geq 2$ and let $\rho > 0$ and $a,b \in \mathbb{R}^{+}$ be such that $b > \frac{q-1}{2}$. Let $Z_1$ denote the component of $Z$ along the direction of the standard basis vector $e_1$. Then
$$\int_{\mathbb{S}^{q-1}}\frac{1}{1+\rho^{2a}Z_{1}^{2a}}\frac{1}{1+(1-Z_{1}^2)^b\rho^{2b}}p_{\theta}(dZ) \leq C(q,a,b)\left[\frac{1}{1+\rho^{2b}} + \frac{\rho^{-q+1}}{1+\rho^{2a}}\right]\,.$$
\end{lemma}

\begin{proof}
From standard results, it is clear that $Z_1$ is distributed over $[-1,1]$ with the density function $\psi_q(x) := C_q (1-x^2)^{\tfrac{q-3}{2}}$. Here $C_q$ is the normalizing constant. Therefore, the integral in the statement of the lemma becomes
\begin{align*}
&\int_{-1}^{1}\frac{\psi_q(x)}{1+\rho^{2a}x^{2a}}\frac{dx}{1+(1-x^2)^b\rho^{2b}} = 2\int_{0}^{1}\frac{\psi_q(x)}{1+\rho^{2a}x^{2a}}\frac{dx}{1+(1-x^2)^b\rho^{2b}} \\
&= 2\int_{0}^{1/2}\frac{\psi_q(x)}{1+\rho^{2a}x^{2a}}\frac{dx}{1+(1-x^2)^b\rho^{2b}}+ 2\int_{1/2}^{1}\frac{\psi_q(x)}{1+\rho^{2a}x^{2a}}\frac{dx}{1+(1-x^2)^b\rho^{2b}} \\
&\leq \frac{2}{1+2^{-2b}\rho^{2b}}\int_{0}^{1/2}\psi_q(x)dx + \int_{1/2}^{1} \frac{\psi_q(x)}{1+\rho^{2a}2^{-2a}}\frac{dx}{1+(1-x^2)^b\rho^{2b}} \\
&\leq \frac{C(q,b)}{1+ \rho^{2b}} + \frac{C(q,a)}{1+\rho^{2a}} \int_{1/2}^{1} \frac{(1-x^2)^{\tfrac{q-3}{2}}}{1+(1-x^2)^b\rho^{2b}}dx\,.
\end{align*}
In the integral from $1/2$ to $1$, $2x \geq 1$. Therefore, from the equation above,
\begin{align*}
\int_{-1}^{1}\frac{\psi_q(x)}{1+\rho^{2a}x^{2a}}\frac{dx}{1+(1-x^2)^b\rho^{2b}} &\leq 
\frac{C(q,b)}{1+ \rho^{2b}} + \frac{C(q,a)}{1+\rho^{2a}} \int_{1/2}^{1} \frac{(1-x^2)^{\tfrac{q-3}{2}}}{1+(1-x^2)^b\rho^{2b}}2xdx\,.
\end{align*}
We now make the change of variable $t = (1-x^2)\rho^2$, yielding
\begin{align*}
\int_{-1}^{1}\frac{\psi_q(x)}{1+\rho^{2a}x^{2a}}\frac{dx}{1+(1-x^2)^b\rho^{2b}} &\leq 
\frac{C(q,b)}{1+ \rho^{2b}} + \frac{C(q,a)\rho^{-q+1}}{1+\rho^{2a}} \int_{0}^{\frac{3\rho^2}{4}} \frac{t^{\tfrac{q-3}{2}}}{1+t^b}dt \\
&\leq \frac{C(q,b)}{1+ \rho^{2b}} + \frac{C(q,a)\rho^{-q+1}}{1+\rho^{2a}} \int_{0}^{\infty} \frac{t^{\tfrac{q-3}{2}}}{1+t^b}dt\,.
\end{align*}
The integral on the RHS is finite if $b > \frac{q-1}{2}$ for every $q\geq 2$. Using this fact in the equation above, we conclude the statement of the lemma. 
\end{proof}

\begin{proof}{\bf (of Lemma~\ref{lem:fourier_contraction})}
To begin, Fubini's theorem for positive functions and Jensen's inequality imply that
\begin{align}
\mathbb{E}C_{g^{\rem}}^{(s)}  &= \mathbb{E} \int_{\mathbb{R}^q} \|\xi\|^s |G^{\rem}(\xi)|d\xi \nonumber\\
&=  \int_{\mathbb{R}^q} \|\xi\|^s \mathbb{E}|G^{\rem}(\xi)|d\xi \nonumber\\
&\leq \int_{\mathbb{R}^q} \|\xi\|^s \sqrt{\mathbb{E}|G^{\rem}(\xi)|^2}d\xi \,.\label{eq:expected_norm_first_bound}
\end{align}
By linearity of Fourier transform, we have $G^{\rem}(\xi) = G(\xi;R) - \frac{1}{N}\sum_{i=1}^N \hat{G}_j(\xi;R)$. By item 4 of Lemma~\ref{lem:unbiased_estimators}, we know that for every $\xi \in \mathbb{R}^q$, 
$$\mathbb{E}|G^{\rem}(\xi)|^2 = \frac{1}{N}\left[ \mathbb{E}\bigr|\hat{G}_j(\xi;R)\bigr|^2 - |G(\xi;R) |^2\right] \leq \frac{1}{N}\left[ \mathbb{E}\bigr|\hat{G}_j(\xi;R)\bigr|^2\right] \,.$$
Using this in Equation~\eqref{eq:expected_norm_first_bound}, we have that
\begin{equation}
\mathbb{E}C_{g^{\rem}}^{(s)} \leq \frac{1}{\sqrt{N}}\int_{\mathbb{R}^q} \|\xi\|^s \sqrt{\mathbb{E}|\hat{G}_j(\xi;R)|^2}d\xi\,.
\label{eq:expected_norm_second_bound}
\end{equation}

We use the polar decomposition of $\mathbb{R}^q$. Let $p_{\theta}$ be the uniform probability measure on $\mathbb{S}^{q-1}$, the sphere embedded in $\mathbb{R}^q$. Continuing Equation~\eqref{eq:expected_norm_second_bound},
\begin{align}
\mathbb{E}C_{g^{\rem}}^{(s)} &\leq \frac{1}{\sqrt{N}}\int_{\mathbb{R}^q} \|\xi\|^s \sqrt{\mathbb{E}|\hat{G}_j(\xi;R)|^2}d\xi \nonumber\\
&= 
\frac{C(q)}{\sqrt{N}}\int_{\rho = 0}^{\infty}\int_{\mathbb{S}^{q-1}} \rho^{s+q-1} \sqrt{\mathbb{E}|\hat{G}_j(\rho Z;R)|^2} p_{\theta}(dZ)d\rho \nonumber\\
&\leq \frac{C(q)}{\sqrt{N}}\int_{\rho = 0}^{\infty} \rho^{s+q-1} \sqrt{\int_{\mathbb{S}^{q-1}}\mathbb{E}|\hat{G}_j(\rho Z;R)|^2p_{\theta}(dZ)} d\rho \nonumber \\
&= \frac{C(q)}{\sqrt{N}}\int_{\rho = 0}^{\infty} \rho^{s+q-1} \sqrt{\mathbb{E}\int_{\mathbb{S}^{q-1}}|\hat{G}_j(\rho Z;R)|^2p_{\theta}(dZ)} d\rho\,.
\label{eq:expected_norm_third_bound}
\end{align}
 The third step above follows from Jensen's inequality applied to the probability measure $p_{\theta}$. We first consider the case $q \geq 2$.  We will now upper bound $\int_{\mathbb{S}^{q-1}}|\hat{G}_j(\rho Z;R)|^2p_{\theta}(dZ)$ as a function of $\rho$. We note the following inequalities:
\begin{enumerate}
\item
 Now, by definition of the Schwartz space, for every integer $n$, there exists a constant $C(n,q,R)$ such that for every $\xi \in \mathbb{R}^d$
$$|\Gamma_{q-1,R}(\xi^{\perp}_{\omega})| \leq \frac{C(n,q,R)}{1 + \|\xi^{\perp}_{\omega}\|^n} \,.$$ 
\item
From Lemma~\ref{lem:filter_bounds}, we have
$$|\Lambda^{\alpha_0}_{k,w_{0}}(\xi_{\omega})| \leq \frac{C_2}{1 + \xi_{\omega}^{2k}}\,.$$
\item Similar to item 1, we have for every $\xi \in \mathbb{R}^q$:
$$|\Gamma_{q,R}(\xi)| \leq \frac{C(n,q,R)}{1 + \|\xi\|^n} \,.$$ 
\end{enumerate}

From the proof of item 4 of Lemma~\ref{lem:unbiased_estimators}, we have
\begin{equation}
|\hat{G}_j(\xi)| \leq \begin{cases} \beta_{g,k}|\Gamma_{q,R}(\xi)| |1+w_0-T_j| &\text{ when } \omega_j =0 \\
C(k)\beta_{g,k}|\Gamma_{q-1,R}(\xi^{\perp}_{\omega})|\Lambda_{k,w_0}^{\alpha_0}(\xi)||\min\left(r(1+w_0-T_j)^2,\frac{1}{r\xi_{\omega_j}^2}\right) &\text{ when } \omega_j \neq 0\,.
\end{cases}
\end{equation}
We use the inequality $\min(a^2,\frac{1}{x^2}) \leq \frac{1+a^2}{1+x^2}$ along with the inequalities above to show that for every $\xi \in \mathbb{R}^q$
\begin{equation}
|\hat{G}_j(\xi)| \leq \begin{cases}  \frac{C\beta_{g,k}}{1 + \|\xi\|^n} |1+w_0-T_j| &\text{ when } \omega_j =0 \\
 \frac{C \beta_{g,k}}{1 + \|\xi^{\perp}_{\omega}\|^n} \frac{1+(1+w_0-T)^2}{1+\xi_{\omega_j}^{2k+2}} &\text{ when } \omega_j \neq 0\,,
\end{cases}
\end{equation}
where $C$ depends on $k,q,n,R$ and $r$.  Therefore
\begin{equation}
\int_{\mathbb{S}^{q-1}}|\hat{G}_j(\rho Z;R)|^2p_{\theta}(dZ) \leq \begin{cases}
\frac{C\beta^2_{g,k}(1+w_0-T_j)^2}{1+\rho^{2n}} \text{ when } \omega_j = 0
\\
 C\beta_{g,k}^2 (1+(1+w_{0}-T_j)^4)\int_{\mathbb{S}^{q-1}}\frac{p_{\theta}(dZ)}{1+\rho^{4k+4}Z_{\omega_j}^{4k+4}}\frac{1}{1+(1-Z_{\omega_j}^2)^n\rho^{2n}} \\\text{ when } \omega_j \neq 0\,.
 \end{cases}
\end{equation}
Using the rotational invariance of $p_{\theta}$, we invoke Lemma~\ref{lem:integral_over_sphere_bound} and conclude that when $n > \frac{q-1}{2}$
\begin{equation}
\int_{\mathbb{S}^{q-1}}|\hat{G}_j(\rho Z;R)|^2p_{\theta}(dZ) \leq \begin{cases}
\frac{C\beta^2_{g,k}(1+w_0-T_j)^2}{1+\rho^{2n}} \text{ when } \omega_j = 0
\\
 C \beta_{g,k}^2\left(1+(1+w_{0}-T_j)^4\right)\left[\frac{1}{1+\rho^{2n}} + \frac{\rho^{-q+1}}{1+\rho^{4k+4}}\right] \\ \text{ when } \omega_j \neq 0\,.
 \end{cases}
\end{equation}

Since $n$ can be arbitrarily large (and this only changes the multiplicative constant), we can pick $n = 2k+2 + q-1$. Now taking expectation with respect to $T$ and noting that when $l\geq 3$, $\mathbb{E}T^4 < \infty$, we have that
\begin{equation}
\mathbb{E}\int_{\mathbb{S}^{q-1}}|\hat{G}_j(\rho Z;R)|^2p_{\theta}(dZ) \leq C \beta_{g,k}^2 \left[\frac{\rho^{-q+1}}{1+\rho^{4k+4}}\right]\,.
\label{eq:uniform_sphere_jensen_bound}
\end{equation}

Consider the case $q = 1$: it is easy to show from the techniques above that the same bound as in Equation~\eqref{eq:uniform_sphere_jensen_bound} holds. Plugging this into Equation~\eqref{eq:expected_norm_third_bound}, we conclude that
$$\mathbb{E}C_{g^{\rem}}^{(s)} \leq \frac{C\beta_{g,k} }{\sqrt{N}}\int_{0}^{\infty}d\rho \frac{\rho^{s+\frac{q-1}{2}}}{1+\rho^{2k+2}}\,.$$
Now, it is clear that the integral on the RHS is finite when $s < \frac{3-q}{2} + 2k$. 
Using the definition of $\beta_{g,k}$ we conclude the result.
\end{proof}

\subsection{Proof of Lemma~\ref{lem:unbiased_estimators_sup_case}}

 The first 2 items are similar as in the proof of Lemma~\ref{lem:unbiased_estimators}. 
 We will show items 3 and 4 below.

 \paragraph{3.}  In Equation~\eqref{eq:post_surgery_estimator_sup_case}, $\gamma_{\omega_j}^{\perp}(x)$ is infinitely differentiable. Therefore, to show that $\hat{g}_j(\spacedot;R) \in C^{2k}(\mathbb{R}^q)$, it is sufficient to show that $\sDelta_k\left(\tfrac{\langle\omega_j,x\rangle}{r\|\omega_j\|},T_j\right)$ is $2k$ times continuously differentiable. This reduces to showing that $t \to \sDelta_k(t,T) \in C^{2k}(\mathbb{R})$ for $T \leq 1+w_{0}$. (We only need to worry about the case $T_j \leq 1+w_{0}$ because otherwise $\hat{g}_j(x;R) = 0$ identically). Consider the Fourier transform of $\sDelta_k(t,T)$:
 $$\sDelta_k^{F}(\upsilon) = \frac{4e^{i(1+w_{0})\upsilon }}{\upsilon^2}\sin^2((1+w_{0}-T)\upsilon/2) \Lambda_{k,w_{0}}^{\alpha_0}(\upsilon) \,.$$
Using the upper bounds on $\Lambda_{k,w_{0}}^{\alpha_0}(\upsilon)$ in Lemma~\ref{lem:filter_bounds}, $\upsilon^{2k}\sDelta_k^{F}(\upsilon)$ is a $L^1$ function with respect to Lebesgue measure. By duality between multiplication by $\upsilon$ of the Fourier transform and differentiation of the function, we conclude that $\sDelta_k(t,T)$ is $2k$ times continuously differentiable and and hence that $\hat{g}_j(x;R) \in C^{2k}(\mathbb{R}^q)$ almost surely. Further, for every $l\leq 2k$, we have
 $$D^{(l)}\sDelta_k(t;T) = \frac{1}{2\pi}\int (-i)^l(\upsilon)^l\sDelta_k^{F}(\upsilon)e^{-i\upsilon t}d\upsilon\,.$$ 
Therefore, 
\begin{align}
\sup_{t\in \mathbb{R}}|D^{(l)}\sDelta_k(t;T)| &\leq  \frac{1}{2\pi}\int |\upsilon|^l|\sDelta_k^{F}(\upsilon)|d\upsilon \nonumber \\
&\leq \int_{-\infty}^{\infty}  B^{0}_k\min\left((1+w_0-T)^2,\frac{1}{\upsilon^2}\right)\frac{|\upsilon|^l}{1+|\upsilon|^{2k}}d\upsilon \nonumber\\
&\leq B^{0}_{k}|1+w_0-T|\leq B_k^{0}(1+ |T|)\,,\label{eq:one_d_derivative_uniform_bound}
\end{align}
where $B^{0}_k < \infty$ is a constant depending only on $\alpha_0,w_{0}$ and $k$. We have absorbed constants involving $\alpha_0,w_0$ and $k$ into other constants throughout and used the inequality $\frac{\sin^2(\upsilon(1+w_0-T)/2)}{\upsilon^2} \leq \min\left((1+w_0-T)^2,\frac{1}{\upsilon^2}\right)$ and the upper bound on $\Lambda_{k,w_{0}}^{\alpha_0}(\upsilon)$ in Lemma~\ref{lem:filter_bounds}. We can in fact improve this bound further because of the fact that $\sDelta(t;T)$ is supported between $[T-w_{0}, 2+3w_{0} - T]$. Therefore, $|D^{(l)}\sDelta_k(t;T)|$ is non zero only when $t \in [T-w_{0}, 2+3w_{0} - T] $. That is when $T - w_{0} \leq t \leq 2+3w_{0} - T$. These inequalities along with the assumption that $T \leq 1+w_{0}$ imply that $|D^{(l)}\sDelta_k(t;T)|$ is non-zero only when $T \leq -|t|+2+3w_{0}$. Therefore, from Equation~\eqref{eq:one_d_derivative_uniform_bound}, we conclude:

\begin{equation}\label{eq:one_d_derivative_compact_bound}
|D^{(l)}\sDelta_k(t;T)| \leq  B^{0}_k(1+|T|)\ind(T \leq -|t|  + 2+3w_{0})\,.
\end{equation}
Consider the element wise partial order $\leq$ on $ \left(\mathbb{N}\cup\{0\}\right)^q$ where $\mathbf{a} \leq \mathbf{b}$ iff $a_i \leq b_i$ for $i \in [q]$. By the chain rule, we conclude that $\partial^{\mathbf{b}}\hat{g}_j(x;R)$ is a finite linear combination of terms of the form \begin{equation}\label{eq:chain_rule}
\beta^{S}_{g,k}  \theta_j\partial^{\mathbf{a}}\left( \sDelta_k\left(\tfrac{\langle\omega_j,x\rangle}{r\|\omega_j\|},T_j\right)\right)\partial^{\mathbf{b}-\mathbf{a}}\gamma_{\omega_j}^{\perp}(x)\,,
\end{equation} for every $\mathbf{a} \leq \mathbf{b}$ such that the coefficients depend only on $\mathbf{a}$ and $\mathbf{b}$. Now,\begin{equation}\label{eq:partial_derivative_sdelta}
\beta^{S}_{g,k}  \theta_j\partial^{\mathbf{a}}\left( \sDelta_k\left(\tfrac{\langle\omega_j,x\rangle}{r\|\omega_j\|},T_j\right)\right) = \beta^{S}_{g,k}  \theta_j \frac{\prod_{s=1}^{q}\langle\omega_j,e_s\rangle^{a_s}}{r^{|\mathbf{a}|}\|\omega_j\|^{|\mathbf{a}|}}D^{|\mathbf{a}|}\sDelta_k\left(\tfrac{\langle\omega_j,x\rangle}{r\|\omega_j\|},T_j\right) \,.
\end{equation}
From Equation~\eqref{eq:one_d_derivative_compact_bound}, the quantity above is nonzero only when $|x_{\omega_j}| \leq  2r+3rw_{0} - rT_j$. $\gamma_{\omega^{\perp}}(x)$ is a $C^{\infty}$ function which vanishes when $\|x_{\omega^{\perp}}\| \geq 2R$, we conclude that $\partial^{\mathbf{b}-\mathbf{a}}\gamma_{\omega_j}^{\perp}(x)$ also vanishes when $\|x_{\omega^{\perp}}\| \geq 2R$. Therefore, we conclude that $\partial^{\mathbf{b}}\hat{g}_j(x)$ is continuous and compactly supported almost surely and hence in $L^1(\mathbb{R}^q)$.

Now for the bound on $\partial^{\mathbf{b}}\hat{g}_j(x)$, we proceed as above by noting that this is a linear combination of the terms of the form given in Equation~\eqref{eq:chain_rule} for $\mathbf{a} \leq \mathbf{b}$. Now, $\partial^{\mathbf{b}-\mathbf{a}}\gamma_{\omega_j}^{\perp}(x)$ is bounded uniformly by a constant $H_k$ for every $x$ and $a$ where $H_{k}$ doesn't depend on $\omega_j$. The function $\partial^{\mathbf{b}-\mathbf{a}}\gamma_{\omega_j}^{\perp}(x)$ vanishes when $\|x_{\omega_j}^{\perp}\| \geq 2R$. From Equations~\eqref{eq:one_d_derivative_compact_bound} and~\eqref{eq:partial_derivative_sdelta} we get that
 $$\biggr|\beta^{S}_{g,k}  \theta_j\partial^{\mathbf{a}}\left( \sDelta_k\left(\tfrac{\langle\omega_j,x\rangle}{r\|\omega_j\|},T_j\right)\right)\biggr| \leq \beta_{g,k}^{S}B_k(1+|T_j|)\ind(rT_j \leq -|x_{\omega_j}|  + 2r+3w_{0}r)\,.$$
Here $B_k$ depends on $\alpha_0,q,r,k,R$ and $w_{0}$ but not on $g$, $T_j$ or $\omega_j$. Therefore, we obtain the desired bound (where we have absorbed all the constants into $B_k$, redefining as necessary):
\begin{equation}\label{eq:almost_sure_derivative_bound}
|\partial^{\mathbf{b}}\hat{g}_j(x;R)| \leq \beta_{g,k}^S B_k(1+|T_j|)\ind(rT_j \leq -|x_{\omega_j}|  + 2r+3w_{0}r)\ind(\|x_{\omega_j}^{\perp}\| \leq 2R)\,. 
\end{equation}

 \paragraph{4.} The proof follows through an induction over $|\mathbf{b}|$ and use of item 3. We will show this for one differentiation here but the argument can be extended to $2k$ times differentiation. By standard results in probability theory, $\frac{\partial g(x;R)}{\partial x_1}$ exists and equal to $\mathbb{E}\frac{\partial \hat{g}_j(x;R)}{\partial x_1}$ if $\frac{\partial \hat{g}_j(x;R)}{\partial x_1} $ exists and for every $x$, $|\frac{\partial \hat{g}_j(x;R)}{\partial x_1}| \leq Z$ for some integrable random variable $Z$. From item 3, we conclude that $\frac{\partial \hat{g}_j(x;R)}{\partial x_1}$ exists and take $Z = \beta_{g,k}^S B_k(1+|T_j|)$ where $\beta_{g,k}^S B_k$ are constants as used in the statement of item 3. This shows that $\frac{\partial g(x;R)}{\partial x_1} = \mathbb{E}\frac{\partial \hat{g}_j(x;R)}{\partial x_1}$. We show that it is continuous by using dominated convergence theorem after noting the fact that $\frac{\partial \hat{g}_j(x;R)}{\partial x_1}$ is continuous and dominated by  $Z = \beta_{g,k}^S B_k(1+|T_j|)$, which is integrable.

To show that $\partial^{\mathbf{b}}g(x;R) \in L^1(\mathbb{R}^{q})$, it is sufficient to show that $\partial^{\mathbf{b}}\hat{g}_j(x;R)$ is integrable with respect to the measure $\mu_l\times\nu_0\times dx$ where $dx$ denotes the Lebesgue measure over $\mathbb{R}^q$. From Fubini's theorem for positive functions, we conclude that 
$$\int |\partial^{\mathbf{b}}\hat{g}_j(x;R)|\mu_l(dT_j)\times \nu_0(d\omega_j)\times dx = \int \mu_l(dT_j)\times \nu_0(d\omega_j)\int |\partial^{\mathbf{b}}\hat{g}_j(x;R)|dx\,.$$ 
 
 Integrating Equation~\eqref{eq:almost_sure_derivative_bound} over $\mathbb{R}^q$, we conclude that $\int |\partial^{\mathbf{b}}\hat{g}_j(x;R)|dx \leq C(1+|T_j|^2)$ for some non-random constant $C$. Since $\mathbb{E}|T_j|^2 <\infty$ by assumption in the statement of the lemma, we conclude that $\partial^{\mathbf{b}}\hat{g}_j(\spacedot;R)$ is integrable with respect to $\mu_l\times\nu_0\times dx$ which implies the desired result. \hfill\qed

\section{Proof of Main Theorems}
\label{sec:main_thm_proofs}

\subsection{Proof of Theorem~\ref{thm:fast_rates_part_1}}
We now prove Theorem~\ref{thm:fast_rates_part_1}. 
For the case $a =0$, we can obtain this error using a $\R$ network as shown in Theorem~\ref{thm:one_layer_approximation}. By Equation~\eqref{eq:unbiased_single_estimator}, $|\kappa_j| \leq \beta_{g,0} \leq \frac{1}{N}C_1\big(C_g^{0}+C_g^{(2)}\big)$ almost surely and the bound on $\sum_{j=1}^{N} |\kappa_j|$ follows. Now we let $a \geq 1$. For the sake of clarity, we will assume that $\frac{N}{a+1}$ is an integer.

Item 2 of Theorem~\ref{thm:remainder_regularity} implies that there exists a two-layer $\sR_{k_a}$ network with $N/(a+1)$ activation functions with output $\hat{g}^{0}(x)$ and there exists a remainder function 
$g^{\rem,0} :\mathbb{R}^q  \to \mathbb{R}$ such that for every $x \in B^2_q(r)$, we have $g^{\rem,0}(x) = g(x) -\hat{g}^{(0)}(x)$ and
$$ C^{(0)}_{g^{\rem,0}}+C^{(2k_{a-1}+2)}_{g^{\rem,0}} \leq C\frac{\big(C^{(0)}_g+C^{(2k_a+2)}_g\big)}{\sqrt{N}}\,.$$

Supposing that $\hat{g}^{(0)}(x) = \sum_{j=1}^{N/(a+1)}\kappa_j^{a} \sR_{k_a}(\langle\omega^a_j,x\rangle-T^{a}_j)$, by similar considerations as the $a=0$ case we conclude that $\sum_{j=1}^{N/a}|\kappa_j^{a}| \leq C_1\big(C^{(0)}_g+C^{(2k_a+2)}\big)$ almost surely. The fact that $\|\omega_j^a\| \leq 1/r$ follows from Equation~\eqref{eq:unbiased_single_estimator}, which is used to construct the estimators in Theorem~\ref{thm:remainder_regularity}.

Invoking Theorem~\ref{thm:remainder_regularity} again, we conclude that we can approximate $g^{\rem,0}$ by  $\hat{g}^{(1)}$, which is the output two-layer $\sR_{k_{a-1}}$ network with $\frac{N}{a+1}$ non-linear activation functions and there exists $g^{\rem,1}:\mathbb{R}^q \to \mathbb{R}$ such that  $g^{\rem,1}(x) = g^{\rem,0}(x)  - \hat{g}^{(1)}(x)$ and
$$C^{(2k_{a-2}+2)}_{g^{\rem,1}}+C^{(0)}_{g^{\rem,1}} \leq C\frac{C^{(0)}_{g^{\rem,0}}+C^{(2k_{a-1}+2)}_{g^{\rem,0}}}{\sqrt{N}} \leq C\frac{\big(C^{(0)}_g+C^{(2k_a+2)}_g\big)}{N}\,.$$

Continuing similarly, for $1 \leq b \leq a-1$ we obtain $\hat{g}^{(b)}$ which is the output of some $\sR_{k_{a-b}}$ units with $\frac{N}{a+1}$ neurons and remainders $g^{\rem,b} :\mathbb{R}^q  \to \mathbb{R}$ such that for every $x \in B^2_q(r)$, we have $g^{\rem,b}(x) = g^{\rem,b-1}(x)  -\hat{g}^{(b)}(x)$
and  
$$C^{(2k_{a-b-1}+2)}_{g^{\rem,b}}+C^{(0)}_{g^{\rem,b}} \leq  C\frac{\big(C^{(0)}_g+C^{(2k_a+2)}_g\big)}{N^{\frac{b+1}{2}}}\,.$$

Now, writing $\hat{g}^{(b)}(x) = \sum_{j=1}^{N/(a+1)}\kappa_j^{a-b} \sR_{k_{a-b}}(\langle\omega^{a-b}_j,x\rangle-T^{a-b}_j)$, we conclude that $\|\omega_j^{a-b}\| \leq 1/r$ and
$$\sum_{j=1}^{N/(a+1)}|\kappa_j^{a-b}| \leq C_1\frac{\left(C^{(0)}_g+C^{(2k_a+2)}_g\right)}{N^{b/2}}\,.$$
In particular, we have $g^{\rem,a-1}$ such that $C^{(2)}_{g^{\rem,a-1}}+C^{(0)}_{g^{\rem,a-1}} \leq C ({C_g^{(0)}+ C_g^{2k_a+2}})/({N^{\frac{a}{2}}})$. Therefore, by Theorem~\ref{thm:one_layer_approximation}, there exists a random $\R$ network with ${N}/({a+1})$ neurons which approximates $g^{\rem,a-1}$ with output $\hat{g}^{(a)}$ such that:
\begin{enumerate}
\item \begin{align*}
\mathbb{E}\int \left(g^{\rem,a-1}(x)  - \hat{g}^{(a)}(x)\right)^2 \zeta(dx) &\leq C\frac{\left(C^{(2)}_{g^{\rem,a-1}}+C^{(0)}_{g^{\rem,a-1}}\right)^2}{N} \\ &\leq C \frac{\big(C_g^{(0)}+ C_g^{2k_a+2}\big)^2}{N^{a+1}}\,.
\end{align*}
\item 
 $$g^{\rem,a-1}  - \mathbb{E}\hat{g}^{(a)}_j(x) = 0\,,$$
where $\hat{g}_j^{(a)}$ is the $j$-th component of $\hat{g}^{(a)}$.
\item 
Assuming $\hat{g}^{(a)}(x) = \sum_{j=1}^{N/(a+1)}\kappa_j^{0} \R(\langle\omega^0_j,x\rangle-T^0_j)$, it is clear that $\|\omega^0_j\| \leq 1/r$:
$$\sum_{j=1}^{N/(a+1)}|\kappa_j^{0}| \leq C_1 \frac{C_g^{(0)} + C_g^{(2k_a+2)}}{N^{a/2}}\,.$$
\end{enumerate}

We note that we have chosen the $\sR_{k_b}$ units in a non-random fashion through Theorem~\ref{thm:remainder_regularity} whereas we have chosen the last $\frac{N}{a+1}$ $\R$ units randomly using Theorem~\ref{thm:one_layer_approximation}. Therefore, the expectation above is only with respect to the randomness of the $\R$ units.
It is clear that $g^{\rem,a-1}(x) - \hat{g}^{(a)}(x) = g(x) -\left(\sum_{b=0}^{a} \hat{g}^{(b)}(x) \right)$ whenever $x \in B_q^2(r)$ and $\sum_{b=0}^{a} \hat{g}^{(b)}(x)$ is the output of a two-layer network with $N$ non-linear units containing $\R$ and $\sR_k$ units for $k \in \{k_1,\dots,k_a\}$. We conclude items 1 and 2 in the statement of the lemma. The sum of the absolute values of the  coefficients is 
$\sum_{b=0}^{a}\sum_{j=1}^{N/(a+1)} |\kappa_j^b| \leq C_1(C_g^{(0)} + C_g^{(2k_a+2)})$ as is clear from the discussion above. \hfill\qed

\subsection{Proof of Theorem~\ref{thm:fast_rates_part_2}}

We first note that whenever $x\in B^2_d(r)$, $\langle x,B_i\rangle \in B^2_q(r)$.  
We assume that ${N}/{(a+1)m}$ is an integer. In Theorem~\ref{thm:fast_rates_part_1}, we take $g = f_i$ and replace $N$ with ${N}/{m}$. We pick the weights $\omega_i$ inside the $\sR_k$ and $\R$ units to be in $\mathrm{span}(B_i)$ instead of $\mathbb{R}^q$ and the replace the distribution $\zeta(dx)$ by $\zeta(\langle dx,B_i\rangle)$, which is the measure induced by $\zeta$ over $\mathrm{span}(B_i)$. We conclude that there exists a random neural network $\mathrm{NN}_i$ with $1$ nonlinear layer whose output is $\hat{f}_i(x)$ such that:
\begin{enumerate}

\item For every $x \in B_q^2(r)$,
$$\mathbb{E}\hat{f}_i(x) = f_i(\langle x,B_i\rangle)$$
\item $$\mathbb{E}\int\big(f_i(\langle x,B_i\rangle) -\hat{f}_i(x)\big)^2\zeta(dx) \leq C_0\frac{Mm^{a+1}}{N^{a+1}}\,.$$
\end{enumerate}
We construct the random neural networks $\mathrm{NN}_i$ independently for $ i\in [m]$. We juxtapose these $m$ neural networks and average their outputs to obtain the estimator $\hat{f}(x) :=\frac{1}{m}\sum_{i=1}^{m}\hat{f}_i(x) $. 
Now
\begin{align}
&\mathbb{E}\int \Big(\frac{1}{m}\sum_{i=1}^{m}f_i(\langle x,B_i\rangle)- \hat{f}_i(x)\Big)^2\zeta(dx) \nonumber\\
&= \frac{1}{m^2} \sum_{i,j \in [m]} \mathbb{E}\int \left(f_i(\langle x,B_i\rangle)- \hat{f}_i(x)\right)\left(f_j(\langle x,B_i\rangle)- \hat{f}_j(x)\right)\zeta(dx) \nonumber\\
&= \frac{1}{m^2} \sum_{i,j \in [m]} \int \mathbb{E}\left(f_i(\langle x,B_i\rangle)- \hat{f}_i(x)\right)\left(f_j(\langle x,B_i\rangle)- \hat{f}_j(x)\right)\zeta(dx) \nonumber \\
&= \frac{1}{m^2} \sum_{i \in [m]}\int \mathbb{E}\left(f_i(\langle x,B_i\rangle)- \hat{f}_i( x)\right)^2\zeta(dx) \nonumber\\
&\leq C_0\frac{m^a M}{N^{a+1}}\,.\label{eq:final_layer_part_2}
\end{align}
In the fourth step we have used the fact that $\hat{f}_j(x)$ and $\hat{f}_i(x)$ are independent when $ i\neq j$. Because the above bound holds in expectation, it must hold for some configuration. \hfill\qed

\subsection{Proof of Theorem~\ref{thm:polynomial_approximation}}
Consider the low dimensional polynomial defined in Equation~\eqref{eq:low_degree_poly_def}. Define the following orthonormal set associated with each $V$ in the summation:
 \begin{enumerate}
 \item $B_V = \{e_j:  V(j)\neq 0\}$ where $e_j$ are the standard basis vectors in $\mathbb{R}^d$, if $|\{e_j:  V(j)\neq 0\}|=q$.
\item Otherwise, let $w = q - |\{e_j:  V(j)\neq 0\}|$. Otherwise, draw distinct $e_{j_1},\dots,e_{j_w} \notin \{e_j:  V(j)\neq 0\}$ from some arbitrary fixed procedure and define $B_V = \{e_j:  V(j)\neq 0\} \cup \{e_{j_1},\dots,e_{j_w} \} $. This ensures that $|B_V| = q$.
\end{enumerate}

Clearly, $p_V$ can be seen as a function over $\mathrm{span}(B_V)$ which is isomorphic to $\mathbb{R}^q$. Since we are only interested in $x \in [0,1]^d$, it follows that $\langle x,B_V\rangle \in [0,1]^q \subseteq B_q^2(\sqrt{q})$. We can also modify $p_V(x)$ to $p_V(x)\gamma\left(\|\langle B_V,x\rangle\|^2/q\right)$ where $\gamma \in \sch$ is the bump function defined in Section~\ref{sec:unbiased_estimators} such that $\gamma(t) = 1$ for $t \in [-1,1]$, $\gamma \geq 0$ and $\gamma(t) = 0$ for $|t| \geq 2$. Therefore, $p_V(x)\gamma\left(\|\langle B_V,x\rangle\|^2/q\right)$, when seen as a function over $\mathrm{span}(B_V)$, is itself a Schwartz function and it is equal to $p_V(x)$ whenever $\langle x,B_V\rangle \in B_q^2(\sqrt{q})$. Without any loss, we replace $p_V(x)$ with $p_V(x)\gamma\left(\|\langle B_V,x\rangle\|^2/q\right)$ in Equation~\eqref{eq:low_degree_poly_def}. We note that the low degree polynomials defined above are an instance of the low dimensional function defined in Equation~\eqref{eq:low_dim_function}, but without the factor of $m$. In Theorem~\ref{thm:fast_rates_sup_case_2}, we will just multiply throughout by a factor $m$ - for both $f$ and the estimator $\hat{f}$.  The only change which occurs in  the guarantees is that the error is multiplied by $m^2$ and the co-efficients $\kappa_j$ in the statement of the theorem are multiplied by $m$. In this case, we take $m = {{q+d}\choose{q}}$. Fix an $a \in \mathbb{N}\cup \{0\}$ and take $N \geq (a+1)m$ such that $N/(a+1)m \in \mathbb{N}$. Consider the Fourier norm of $p_V$ when seen as a function over $\mathrm{span}(B_V)$. Clearly $p_V$ is a Schwartz function and the Fourier norm defined in Equation~\eqref{eq:sup_fourier_norm} exists and is finite for every $l = 2k_a^{S}+2$ ($l$ is as used in Equation~\eqref{eq:sup_fourier_norm}). Therefore, we set $H:=  \sup_V (S^{0}_{p_V} + S^{(2k_a^{S}+2)}_{p_V})^2 < \infty$. It is clear that $H$ depends only on $q$ and $a$. Now, the corresponding squared Fourier norms for $J_Vp_V$, denoted by $M_V$ satisfies $M_V \leq HJ_V^2$ (where $M_V$ is the analogue of $M_i$ as defined in Theorem~\ref{thm:fast_rates_sup_case_2}).  Consider the sampling procedure given in Theorem~\ref{thm:fast_rates_sup_case_2}: since the bases $B_V$ (the analogues of $B_i$ in the statement of the theorem) are known explicitly, this sampling can be done without the knowledge of the polynomial. Now, by a direct application of Theorem~\ref{thm:fast_rates_sup_case_2}, we conclude the statement of Theorem~\ref{thm:polynomial_approximation}. \qed
\end{document}